\documentclass[10pt,journal,compsoc]{IEEEtran}
\usepackage[pagebackref=true,breaklinks=true,colorlinks,bookmarks=false]{hyperref}
\newlength\savewidth

\ifCLASSOPTIONcompsoc
  \usepackage[nocompress]{cite}
\else
  \usepackage{cite}
\fi
\usepackage{bbding}
\usepackage{threeparttable}
\usepackage[pdftex]{graphicx}
\ifCLASSINFOpdf
\else
\fi
\usepackage{amsmath}
\usepackage{algorithm}
\usepackage{algpseudocode}

\usepackage{array}

\usepackage{url}
\usepackage{amsfonts}
\usepackage[caption=false,font=normalsize,labelfont=sf,textfont=sf]{subfig}
\usepackage{textcomp}
\usepackage{stfloats}
\usepackage{verbatim}
\usepackage{fix-cm}
\usepackage{tabularx} 
\usepackage{booktabs}       
\usepackage{nicefrac}       
\usepackage{microtype}      
\usepackage{xcolor}         
\usepackage{amssymb, amsthm, amsmath}
\usepackage{tcolorbox}
\usepackage{hyperref, cleveref}
\usepackage{colortbl}
\usepackage{siunitx}  
\usepackage{threeparttable} 
\usepackage{makecell} 
\usepackage{ltablex} 
\usepackage{multirow}
\usepackage{stackengine}
\usepackage{caption}
\usepackage{rotating}
\usepackage{tikz}
\usepackage{lipsum}
\usepackage{fontawesome}
\usepackage{transparent}

\algrenewcommand\algorithmicindent{1.0em}

\makeatletter
\newcommand{\algcolor}[2]{%
  \hskip-\ALG@thistlm\colorbox{#1}{%
    \parbox{\dimexpr\linewidth-1.0em}{#2}}%
}
\newcommand{\algcoloredcomment}[2]{%
  \hskip-\ALG@thistlm\hspace*{1.0em}\colorbox{#1}{%
    \parbox{\dimexpr\linewidth-2.0em}{#2}}%
}
\makeatother

\newcommand{\StateColored}[2]{\State \algcolor{#1}{#2}}

\newtheorem{theorem}{Theorem}
\newtheorem{lemma}[theorem]{Lemma}

\usepackage{xcolor}

\begin{document}
%
\title{Parameter-Efficient Fine-Tuning for Continual Learning: A Neural Tangent Kernel Perspective}
%
%
%
%

\author{Jingren~Liu, Zhong~Ji, \textit{Senior Member, IEEE}, YunLong~Yu, Jiale Cao, Yanwei~Pang, \textit{Senior Member, IEEE}, Jungong Han, \textit{Senior Member, IEEE}, Xuelong Li, \textit{Fellow, IEEE}
\thanks{This work was supported by the National Key Research and Development Program of China (Grant No. 2022ZD0160403), and the National Natural Science Foundation of China (NSFC) under Grants 62441235 and 62176178 (Corresponding author: Zhong~Ji).}
\thanks{Jingren Liu, Zhong Ji, Jiale Cao, and Yanwei Pang are with the School of Electrical and Information Engineering, Tianjin Key Laboratory of Brain-Inspired Intelligence Technology, Tianjin University, Tianjin 300072, China, and also with the Shanghai Artificial Intelligence Laboratory, Shanghai 200232, China (e-mail: \{jrl0219, jizhong, connor, pyw\}@tju.edu.cn).}
\thanks{YunLong Yu is with the College of Information Science and Electronic Engineering, Zhejiang University, Hangzhou, 310027, China (e-mail: yuyunlong@zju.edu.cn).}
\thanks{Jungong Han is with the Department of Computer Science, the University of Sheffield, UK (e-mail: jungonghan77@gmail.com).}
\thanks{Xuelong Li is with the Institute of Artificial Intelligence (TeleAI) of China Telecom (e-mail: xuelong\_li@ieee.org).}
}


%
%

\markboth{Journal of \LaTeX\ Class Files,~Vol.~14, No.~8, August~2015}%
{Shell \MakeLowercase{\textit{et al.}}: Bare Demo of IEEEtran.cls for Computer Society Journals}
%



\IEEEtitleabstractindextext{%
\begin{abstract}
Parameter-efficient fine-tuning for continual learning (PEFT-CL) has shown promise in adapting pre-trained models to sequential tasks while mitigating catastrophic forgetting problem. However, understanding the mechanisms that dictate continual performance in this paradigm remains elusive. To unravel this mystery, we undertake a rigorous analysis of PEFT-CL dynamics to derive relevant metrics for continual scenarios using Neural Tangent Kernel (NTK) theory. With the aid of NTK as a mathematical analysis tool, we recast the challenge of test-time forgetting into the quantifiable generalization gaps during training, identifying three key factors that influence these gaps and the performance of PEFT-CL: training sample size, task-level feature orthogonality, and regularization. To address these challenges, we introduce NTK-CL, a novel framework that eliminates task-specific parameter storage while adaptively generating task-relevant features. Aligning with theoretical guidance, NTK-CL triples the feature representation of each sample, theoretically and empirically reducing the magnitude of both task-interplay and task-specific generalization gaps. Grounded in NTK analysis, our framework imposes an adaptive exponential moving average mechanism and constraints on task-level feature orthogonality, maintaining intra-task NTK forms while attenuating inter-task NTK forms. Ultimately, by fine-tuning optimizable parameters with appropriate regularization, NTK-CL achieves state-of-the-art performance on established PEFT-CL benchmarks. This work provides a theoretical foundation for understanding and improving PEFT-CL models, offering insights into the interplay between feature representation, task orthogonality, and generalization, contributing to the development of more efficient continual learning systems.
\end{abstract}

}

\maketitle

\IEEEdisplaynontitleabstractindextext

%
\IEEEpeerreviewmaketitle

\IEEEraisesectionheading{\section{Introduction}\label{secIntroduction}}

%
%
%
%
\IEEEPARstart{I}{n} practical applications, the relentless evolution of environments underscores the urgency for learning systems that can progressively accumulate knowledge. This has led to the prominence of Continual Learning (CL) \cite{xu2021adaptive,de2021continual,pham2023continual,li2023crnet,wang2024comprehensive,masana2022class}, a cornerstone task that equips the learning models with the ability to seamlessly assimilate fresh information over time, while mitigating catastrophic forgetting, i.e., a phenomenon that erodes previously acquired knowledge. In recent years, with the proliferation of pre-trained models possessing strong generalization capabilities \cite{brown2020language,radford2021learning}, researchers have discovered that they can empower early exploratory methods \cite{buzzega2020dark,wang2023distributionally,gao2023ddgr,magistri2024elastic,bhat2024imexreg,rudner2022continual,zhai2024fine,2024divide,yang2023scrollnet,yang2022dynamic,zhang2023continual,li2024contrastive,li2023variational,yan2024orchestrate}, enabling CL systems to integrate new knowledge more efficiently. However, full fine-tuning of pre-trained models is computationally intensive and may compromise their original generalization capabilities \cite{hu2022lora,yang2024unveiling}. Thus, as a promising paradigm, Parameter-Efficient Fine-Tuning for Continual Learning (PEFT-CL) emerges as an alternative, updating only a minimal set of additional parameters while keeping the pre-trained model intact. Specifically, PEFT-CL not only offers a more philosophically sound framework akin to Socratic dialogue but also provides a lightweight training process that avoids generalization deterioration associated with full-scale fine-tuning \cite{jia2022visual,wang2024parameter}. In addition, this seamless integration of new and old knowledge aligns with the wisdom expressed by Bernard of Chartres, demonstrating how PEFT-CL builds upon pre-existing knowledge to achieve a more adaptive learner with robust memory capabilities.

Despite initial successes in mitigating catastrophic forgetting \cite{wang2022learning, wang2022dualprompt, smith2023coda, gao2024consistent, zhou2024expandable}, PEFT-CL largely relies on subjective human insights and experiential doctrines for network design and enhancement, lacking a rigorous mathematical foundation. This reliance on non-theoretical intuition constrains the potential for a deeper understanding and advancement of the fundamental mechanisms within these learning systems. While Hide-Prompt \cite{wang2024hierarchical} acknowledges the importance of addressing this issue and offers a loss-based perspective, it falls short of modeling optimization dynamics and pinpointing key factors. Therefore, to address this gap, we adopt the Neural Tangent Kernel (NTK) theory \cite{jacot2018neural,canatar2021spectral,bordelon2020spectrum} as a robust mathematical tool to delve deeply into the intricacies of PEFT-CL optimization. Through this rigorous analysis, we derive several fundamental theorems and lemmas, including \cref{Task_Interplay_Generalization}, \cref{Task_Specific_Generalization}, \cref{Generalization_Impact_Factor}, and \cref{MSA_Generalization}. While initially considered from a CL perspective, these have been generalized to the PEFT-CL scenario, providing profound insights into the key factors essential for effectively combating catastrophic forgetting in PEFT-CL optimization. Guided by these theories and key factors, we develop an NTK-CL framework, effectively reducing the quantified catastrophic forgetting discussed later. 

In addition to theoretical advantages, we also detail the differences in structure and optimization between our NTK-CL framework and current mainstream methodologies in Fig.~\ref{Figure0}. Unlike the Additional Subnetworks paradigm (Fig.~\ref{Figure0}a), which constructs task-specific subnetwork parameter spaces and concatenates features from all network parameter spaces at inference time \cite{zhou2024expandable,gao2024consistent,liang2024inflora}, or the Prompts Optimization paradigm (Fig.~\ref{Figure0}b), which builds task-specific prompt pools for input interaction and employs cosine similarity for prompt selection \cite{wang2022learning,jung2023generating,wang2022dualprompt,smith2023coda,qiao2023prompt}, our NTK-CL (Fig.~\ref{Figure0}c) framework eliminates the need for task-specific parameter storage or prompt pools. Instead, it leverages a shared network parameter space across all tasks to adaptively generate task-relevant features based on input characteristics. Specifically, its design and optimization are entirely derived from NTK-based generalization gaps, which not only triple the sample representations but also consider knowledge retention, task-feature dissimilarity, and regularization term.

\begin{figure*}[t]
\centering
\includegraphics[width=1.0\textwidth]{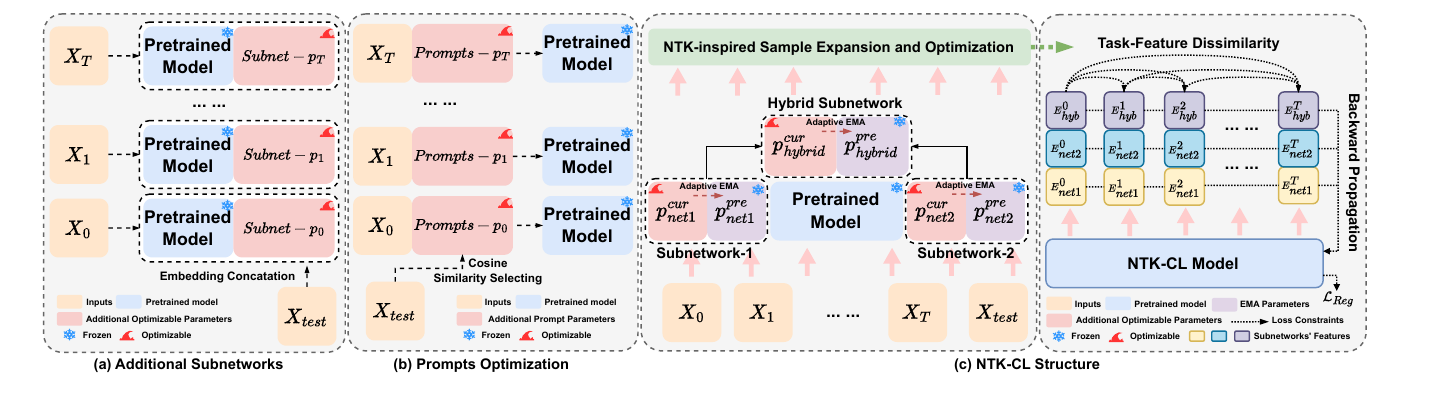}
\caption{Comparison chart between the mainstream frameworks in PEFT-CL and our NTK-CL framework.}
\label{Figure0}
\end{figure*}

Overall, our contributions are delineated across three primary areas:

(1) \textbf{Theoretical Exploration of PEFT-CL}: We pioneer the analysis of PEFT-CL through NTK lens and foundational mathematics. Through a series of derived theorems and lemmas, we identify critical factors that optimize PEFT-CL learners, including the number of samples in data subsets, the total sample volume across the dataset, knowledge retention strategies, task-feature dissimilarity constraints, and adjustments to regularization terms.

(2) \textbf{Innovative Solutions Based on Key Factors}: Closely aligned with the key factors derived from our theoretical analysis, we propose an NTK-CL framework specifically designed for the PEFT-CL scenario. First, guided by \cref{Task_Interplay_Generalization} and \cref{Task_Specific_Generalization}, to increase the sample size available for optimizing the PEFT-CL model without incurring excessive training costs, we incorporate multiple interventions to expand the representational breadth, ensuring that each sample is mapped to different spaces, effectively tripling the representational scope. Second, unlike most previous PEFT-CL methods that do not consider knowledge retention, we design an adaptive Exponential Moving Average (EMA) mechanism that preserves intra-task NTK forms in \cref{Task_Interplay_Generalization}, thereby enhancing knowledge retention. Additionally, we no longer focus on class-level orthogonality as in previous studies, but instead introduce task-feature orthogonality constraints that attenuate inter-task NTK forms in \cref{Task_Interplay_Generalization}, increasing knowledge separability. This dual approach not only effectively avoids the storage overhead associated with parameter preservation but also achieves superior continual performance. Finally, to ensure that network training aligns with the process of finding the saddle point solution in Eq.~\ref{NTK_Dynamics_1}, we implement tailored regularization adjustments. These strategies optimally minimize the generalization gaps and population losses in both task-interplay and task-specific settings within the PEFT-CL scenario, mitigating the catastrophic forgetting problem both theoretically and practically.

(3) \textbf{Empirical Validation on Diverse Datasets}: We conduct extensive experiments across various datasets to validate the effectiveness of our key factors and methodologies. Additionally, we perform fair comparisons against numerous state-of-the-art methods, ensuring consistent task segmentations to mitigate performance discrepancies. This comprehensive validation substantiates the efficacy of our theoretical innovations in practical applications.

These contributions significantly advance PEFT-CL field, bridging the gap between theoretical foundations and practical efficacy in enhancing model performance and generalization across diverse learning environments.

\section{Related Works}
\textbf{Parameter-Efficient Fine-Tuning} has emerged as a pivotal paradigm for optimizing model performance while mitigating computational and memory burdens associated with large-scale model adaptation. Seminal works introduce diverse methodologies, including Adapter modules \cite{houlsby2019parameter}, Low-Rank Adaptation (LORA) \cite{hu2022lora}, Prefix Tuning \cite{li2021prefix}, Prompt Tuning \cite{brown2020language}, and BitFit \cite{zaken2021bitfit}. These approaches demonstrate the efficacy of selectively fine-tuning components or introducing compact, trainable subnetworks within pre-trained architectures. Subsequent advancements further expand PEFT's scope and capabilities. Jia \textit{et al.} \cite{jia2021scaling} pioneer efficient prompt tuning techniques for vision transformers, extending PEFT's applicability to the visual domain. Zhou \textit{et al.} \cite{zhou2022learning} introduce contextual prompt fine-tuning, enhancing model adaptability while preserving generalization. Recent comprehensive studies \cite{xin2024parameter,xing2024survey} reinforce PEFT's critical role in enhancing model generalization and efficiency. These investigations rigorously analyze the theoretical underpinnings and empirical efficacy of various PEFT methodologies, solidifying its status as a transformative paradigm in adaptive learning.

\noindent\textbf{Continual Learning} is a critical field in artificial intelligence aimed at developing models that can learn new tasks while preserving knowledge from previous tasks. In general, this field can be categorized into task-specific and generalization-based approaches. Task-specific strategies include four main methodologies: replay, regularization, dynamic architectures, and knowledge distillation. Replay methods \cite{buzzega2020dark,wang2023distributionally,gao2023ddgr} combat catastrophic forgetting by storing or generating representative samples. Regularization techniques \cite{magistri2024elastic,bhat2024imexreg,rudner2022continual,zhai2024fine} constrain changes to critical parameters, ensuring stability across tasks. Dynamic architectures \cite{2024divide,yang2023scrollnet,yang2022dynamic,zhang2023continual} adapt network structures to incorporate new information, often through expansion or task-relevant modifications. Knowledge distillation \cite{li2024contrastive,li2023variational,yan2024orchestrate} transfers learned knowledge, maintaining information continuity. Generalization-based methods emphasize intrinsic model capabilities for knowledge transfer and retention. Lin \textit{et al.} \cite{lin2023theory} investigate the balance between retention and generalization. Raghavan \textit{et al.} \cite{raghavan2021formalizing} analyze the interaction between learning new information and preserving old knowledge. Ramkumar \textit{et al.} \cite{ramkumar2024effectiveness} study controlled forgetting to enhance model robustness, while Alabdulmohsin \textit{et al.} \cite{alabdulmohsin2021impact} examine the effects of network reinitialization on learning and generalization. Additional foundational research \cite{xiang2022tkil, Karakida2021LearningCF, bennani2020generalisation, doan2021theoretical} explores CL through the lenses of NTK and generalization theory, though these studies primarily address traditional continual learning scenarios and do not fully integrate advancements from the era of pre-trained models.

\noindent\textbf{Parameter-Efficient Fine-Tuning for Continual Learning} has established itself as an effective strategy to counter catastrophic forgetting by training minimal additional parameters atop pre-trained models. Notable approaches such as L2P \cite{wang2022learning} and DualPrompt \cite{wang2022dualprompt} introduce task-specific and dual prompts, respectively, facilitating adaptive task-specific learning while preserving invariant knowledge. S-Prompt \cite{wang2022s} employs structural prompts to map discriminative domain relationships, while CODA-Prompt \cite{smith2023coda} applies Schmidt orthogonalization to refine these prompts. In parallel, DAP \cite{jung2023generating} proposes the construction of real-time, instance-level dynamic subnetworks, offering a flexible mechanism to accommodate the nuances of diverse domains. HiDe-Prompt \cite{wang2024hierarchical} integrates hierarchical task-level knowledge subnetworks with distributional statistics to sample past data, effectively curbing suboptimal learning trajectories. EASE \cite{zhou2024expandable} further contributes by optimizing task-specific, expandable adapters, thereby fortifying the model’s capacity for knowledge retention. Despite these significant strides, the reliance on particular configurations highlights the imperative for a more profound theoretical investigation to fundamentally tackle the challenges inherent in PEFT-CL. This necessitates a paradigm shift toward a NTK perspective, which promises to enrich our understanding in PEFT-CL.

\section{Preliminaries}
In the PEFT-CL context, we augment pre-trained models with adaptive subnetworks to manage sequential tasks. Let $f^*_0$ and $f^*_T$ denote the initial and target parameter spaces respectively, with $*$ indicating optimized parameters. Given a series of tasks $\mathcal{D} = \{\mathcal{D}_1, \ldots, \mathcal{D}_T\}$, where each $\mathcal{D}_\tau$ comprises samples $(x, y)$ from $(X_\tau, Y_\tau)$, we introduce task-specific optimizable subnetwork parameters $p_\tau$. The transformed model is represented as $f^*_\tau = (f^*_0 \circ p_\tau \circ X_\tau \circ Y_\tau)$, with $\circ$ denoting component integration. This configuration, inspired by L2P \cite{wang2022learning}, features distinct class boundaries without explicit task identification during training, aligning with practical scenarios.

\noindent\textbf{Empirical NTK:} The NTK elucidates infinite-width neural network training dynamics, mapping the learning trajectory in high-dimensional parameter space \cite{jacot2018neural}. Leveraging NTK's spectral properties enables precise predictions about network generalization, linking architectural choices to extrapolation performance \cite{bordelon2020spectrum}. However, practical NTK calculation faces challenges due to extensive gradient computations across entire datasets. The empirical NTK \cite{jacot2018neural} addresses this, providing a more tractable analytical tool:
\begin{equation} \label{eq:jac_ntk}
    \Phi_{p_\tau}(x_1, x_2) = \left[J_{p_\tau}(f_{\tau}(x_1))\right] \left[J_{p_\tau}(f_{\tau}(x_2))\right]^\top,
\end{equation}
where $J_{p_\tau}(f_{\tau}(x))$ denotes the Jacobian matrix of network $f_\tau$ with parameters optimized for task $\tau$, evaluated at input $x$. This function maps $D$-dimensional inputs to $O$-dimensional features, with $J_{p_\tau}(f_{\tau}(x)) \in \mathbb{R}^{O \times P}$ and $\Phi_{p_\tau}(x_1, x_2) \in \mathbb{R}^{O \times O}$.

\noindent\textbf{Neural Tangent Kernel Regime:} As layer widths approach infinity, the NTK characterizes the asymptotic behavior of neural networks, yielding a time-invariant NTK throughout training \cite{jacot2018neural,lee2019wide}. This induces a linear dynamical system in function space, governed by the following evolution equation for the output $f(x,\theta)$ at input $x$:
\begin{equation}
\frac{\partial f(x,\theta(t))}{\partial t} = -\Phi(x,X) \nabla_f \mathcal{L}(f(X,\theta(t)),Y),
\end{equation}
where $\Phi(x,X)$ denotes the NTK matrix, $X$ represents the entire training dataset, $Y$ corresponds to labels, and $\mathcal{L}$ signifies the loss function.

This formulation elucidates the network's trajectory towards the global minimum, exhibiting exponential convergence under a positive definite NTK \cite{jacot2018neural,lee2019wide,canatar2021spectral,bordelon2020spectrum,yang2020tensor}. Furthermore, in PEFT-CL, to better adapt it for sequence learning scenarios, we have transformed it in Appendix~\ref{NTK_Dynamics} as follows:
\begin{equation}
    \begin{aligned}
        f_{T}(x) = & f_0^*(x) + \sum\limits_{i=1}^{T} \Phi_i(x, X) \\
        &\times (\Phi_i(X, X)+\lambda I)^{-1}(Y_i-f_{i-1}^*(X)),
    \end{aligned}
    \label{f_NTK}
\end{equation}
where \(\Phi_i\) denotes the locally converged NTK matrix for the \(i\)-th task, and \(\lambda\) is the hyper-parameter that controls the L2 regularization of the trainable parameters in Eq.~\ref{NTK_Dynamics_1}. This hyper-parameter is crucial for finding the dynamic saddle point solution of the model in PEFT-CL scenario.

\textbf{\textit{Remark:}} The NTK paradigm is effective across various neural architectures, including ResNets and Transformers \cite{yang2020tensor,yang2021tensor}, with primary variations evident in the configuration of the NTK matrix. Ideally, all \(\Phi_i\) matrices would evolve towards a consistent \(\Phi\) as the model trains \cite{jacot2018neural,canatar2021spectral}.

\section{Theoretical Insights}
The prevalent belief in PEFT-CL methods is that mitigating catastrophic forgetting should be evaluated based on accuracy, specifically by calculating the difference between the optimal accuracy on a previous task during its optimization and the accuracy on that task at the final stage. However, using abstract accuracy metrics is not conducive to precise mathematical quantification, and the accuracy gap during testing cannot effectively intervene in training. To better align with the role of NTK in studying model generalization, we propose shifting the focus from the accuracy gap to the generalization gap. This shift allows for rigorous mathematical analysis related to training conditions and aligns with established principles of generalizability \cite{zhou2018non,luo2022generalization}.

Harnessing the interpretative power of the NTK to decode network training dynamics, we assess the model's resilience against forgetting through the generalization gaps and population losses. Initially, we derive the general formulation of cross-task generalization gap and population loss for the PEFT-CL scenario, addressing data from the \(\tau\)-th task post the final training session. We further extend our analysis, which assesses the population loss for individual tasks using NTK spectral theory. By examining the commonalities in these losses, we identify key elements that influence the optimization process of the PEFT-CL model and propose further theoretical insights. These concepts will be elaborated upon in a step-by-step manner. \footnote{The derivation process is thoroughly detailed in Appendix~\ref{NTK_Dynamics}, Appendix~\ref{Interplay_Generalization_in_PEFT_CL}, and Appendix~\ref{Intrinsic_Generalization_in_PEFT_CL}. We extend our appreciation to the contributions from \cite{doan2021theoretical, chai2009generalization, canatar2021spectral, bennani2020generalisation} for their invaluable assistance in theoretical derivations, some of which we reference in our work.}.

\begin{theorem}[Task-Interplay Generalization in PEFT-CL] \label{Task_Interplay_Generalization}
Consider a sequence of kernel functions \(\{\Phi_\tau : \mathcal{X} \times \mathcal{X} \rightarrow \mathbb{R}\}_{\tau=1}^T\) and corresponding feature maps \(\varphi_\tau : \mathcal{X} \rightarrow \mathcal{H}\), where \(\mathcal{H}\) represents a Hilbert space. For any function \(f\) within \(\mathcal{F}_T\), it is established with at least \(1 - \delta\) confidence that the discrepancy between the population loss \(L_D(f(X_\tau))\) and the empirical loss \(L_S(f(X_\tau))\) for the \(\tau\)-th task's data is bounded by:
\begin{equation}
\footnotesize
    \sup_{f \in \mathcal{F}_T} \{ L_D(f(X_\tau)) - L_S(f(X_\tau)) \} \leq 2 \rho \hat{\mathcal{R}}(\mathcal{F}_T) + 3 c \sqrt{\frac{\log(2/\delta)}{2 N}},
\end{equation}
where \(\rho\) denotes the Lipschitz constant, \(c\) a constant, and \(N\) the total sample count.

Moreover, if \(f^*_T\) is the optimally selected function from \(\mathcal{F}_T\), the upper bound for the population loss \(L_D(f^*_T)\) in relation to the empirical loss \(L_S(f^*_T)\) can be expressed as:
\begin{equation}
\small
    L_D(f^*_T(X_\tau)) \leq L_S(f^*_T(X_\tau)) + 2 \rho \hat{\mathcal{R}}(\mathcal{F}_T) + 3 c \sqrt{\frac{\log(2/\delta)}{2 N}},
\end{equation}
\begin{equation}
\footnotesize
\begin{split}
    L_S (f_T^*(X_\tau)) \leq & \frac{1}{n_\tau} \bigg[ \lambda^2 \tilde{Y}_\tau^\top(\Phi_\tau(X_\tau,X_\tau) + \lambda I)^{-1}\tilde{Y}_\tau + \sum_{k=\tau+1}^{T} \tilde{Y}_k^\top \\
    & \times (\Phi_k(X_k, X_k)+\lambda I)^{-1} \Phi_k(X_\tau, X_k) \\
    & \times \Phi_k(X_\tau, X_k)^\top (\Phi_k(X_k, X_k)+\lambda I)^{-1}\tilde{Y}_k \bigg]_{\mathcal{D}_\tau},
\end{split}
\end{equation}
\begin{equation}
\small
    \hat{\mathcal{R}}(\mathcal{F}_T) \leq \left[\sum\limits_{\tau=1}^T \mathcal{O}(\sqrt{\frac{[\tilde{Y}_\tau^\top (\Phi_\tau(X, X)+\lambda I)^{-1} \tilde{Y}_\tau]}{n_\tau}}) \right]_{\mathcal{D}_\tau}.
\end{equation}
\end{theorem}

\begin{theorem}[Task-Specific Generalization in PEFT-CL] \label{Task_Specific_Generalization}
In the realm of PEFT-CL, consider a sequence of learning tasks, each uniquely identified by an index \(\tau\). For each task \(\tau\), define \(f^*_\tau(x)\) as the task-specific optimal function, whose performance is critically influenced by the spectral properties of the NTK. The population loss, \(L_D(f^*_\tau)\), for task \(\tau\) is influenced by these spectral properties, and can be quantified as follows:
\begin{equation}
\small
    L_D(f^*_\tau) = \sum\limits_{\rho, i}\frac{w_\rho^{*2}}{\lambda_\rho} \left(\frac{1}{\lambda_\rho} + \frac{s_i}{\lambda + tu_i}\right)^{-2} (1 - \frac{m_i s_i}{(\lambda + tu_i)^2})^{-1},
\end{equation}
Here, \(\rho\) indexes the eigenvalues, \(\lambda_{\rho}\) and \(w_{\rho}^*\) are the eigenvalues and the optimal weights associated with the orthogonal basis functions of the kernel, respectively. The variable \(s_i\) indicates the sample size for \(i = 1, 2, \ldots, n_\tau\). The parameters \(m_i\) and \(tu_i\) are derived from the established relationships:
\begin{equation}
\small
    m_i = \sum_{\rho, i} (\frac{1}{\lambda_\rho} + \frac{s_i}{\lambda + m_i})^{-1}, \quad tu_i = \sum_{\rho, i} (\frac{1}{\lambda_\rho} + \frac{s_i}{\lambda + m_i})^{-2}.
\end{equation}
\end{theorem}

To clarify the exposition, we detail the derivation processes for \cref{Task_Interplay_Generalization} and \cref{Task_Specific_Generalization} in Appendix~\ref{Interplay_Generalization_in_PEFT_CL} and Appendix~\ref{Intrinsic_Generalization_in_PEFT_CL}, respectively. Building on these foundations, we further analyze and derive \cref{Generalization_Impact_Factor}, establishing the basis for the details of subsequent NTK-CL implementations.

\begin{lemma}[Enhanced Generalization in PEFT-CL] \label{Generalization_Impact_Factor}
Within the PEFT-CL scenario, targeted optimizations are essential for augmenting generalization across tasks and bolstering knowledge transfer. Based on the insights from \cref{Task_Interplay_Generalization} and \cref{Task_Specific_Generalization}, the following pivotal strategies are identified to enhance generalization:

\begin{enumerate}
    \item \textbf{Sample Size Expansion:} Increasing both \(n_\tau\) and \(N\) effectively reduces the empirical loss and Rademacher complexity, which in turn lowers the generalization gap and the population loss \(L_D(f^*_\tau)\).

    \item \textbf{Task-Level Feature Constraints:} Preserving the original past knowledge and intensifying inter-task feature dissimilarity, i.e., by maintaining \(\Phi_\tau(X_\tau,X_\tau)\) and \(\Phi_k(X_k, X_k)\), while minimizing \(\Phi_k(X_\tau, X_k)\), adheres to the theoretical underpinnings posited in \cite{doan2021theoretical}.

    \item \textbf{Regularization Adjustment:} Fine-tuning the regularization parameter \(\lambda\) helps optimize the model complexity and the empirical loss, mitigating catastrophic forgetting problem. In addition, adjusting \(\lambda\) influences the eigenvalue distribution within the NTK framework, directly affecting the kernel's conditioning and the generalization bounds as established for \(f^*_\tau(x)\).
\end{enumerate}

\textbf{Proof Outline:} The lemma unfolds through an analysis of the interrelations among Rademacher complexity \footnote{Rademacher complexity measures the complexity and capacity of a function class, estimating a model's generalization ability by assessing its performance on random data. Essentially, it reflects how well a function class can fit under random noise. A higher complexity implies that the function class \(\mathcal{F}\) is more complex and more prone to overfitting the training data.}, empirical loss, and NTK spectral characteristics, as discussed in \cref{Task_Interplay_Generalization} and \cref{Task_Specific_Generalization}. It underscores the significance of sample size expansion, the delineation of task-level features as instrumental, and meticulous regularization to advancing generalization and fostering knowledge retention within PEFT-CL environments.
\end{lemma}

From \cref{Generalization_Impact_Factor}, we identify the key factors that require attention during the optimization process of the PEFT-CL model and propose the NTK-CL framework. While these key factors may also play a beneficial role in other paradigms or traditional CL approaches, our NTK-CL framework introduces specific improvements and innovations tailored for the PEFT-CL scenario. Each component is meticulously designed to align with the constraints and requirements derived from our theoretical analysis, thereby addressing the limitations of existing PEFT-CL methods.

\section{NTK-CL}
\subsection{Extend Sample Size Through PEFT}

\begin{figure*}[t]
\centering
\includegraphics[width=1.0\textwidth]{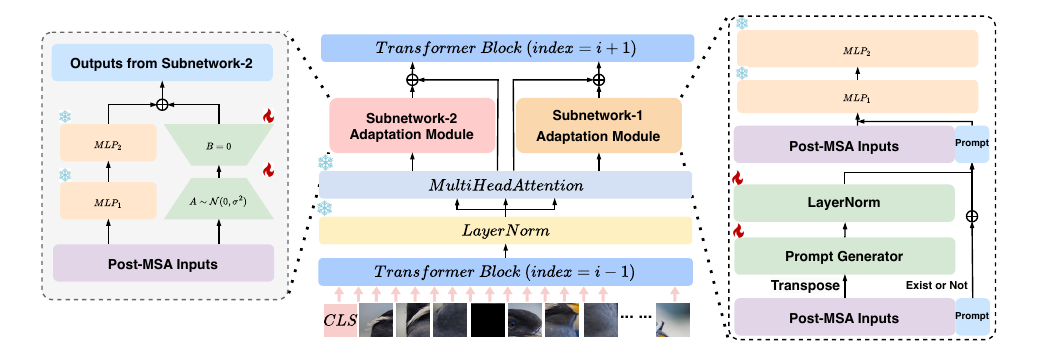}
\caption{Comprehensive visualization of the generation and integration processes of the subnetwork-1 and subnetwork-2 adaptation modules within the transformer architecture.}
\label{Figure1}
\end{figure*}

Drawing upon the theoretical underpinnings elucidated in \cref{Generalization_Impact_Factor} and \cite{alabdulmohsin2022revisiting}, it becomes evident that the augmentation of task-specific sample size exerts a significant influence on mitigating generalization discrepancies. In light of this insight, we introduce a novel strategy meticulously tailored for the PEFT paradigm, predicated on the existence of an optimal function \( f_0^*(x) \), as rigorously defined in Eq.~\ref{f_NTK}. This approach operates across three specialized subnetworks, each responsible for feature generation within unique representational space, thereby engendering a composite feature set. This process not only amplifies the effective sample (feature) size pertinent to each subtask but also fosters a more nuanced and comprehensive representation of the underlying data manifold. Through the judicious adjustment of subnetwork parameters \( p_i \), facilitated by the integration of these multi-dimensional feature representations, our proposed framework achieves a tripling of the representational scope for individual samples. More importantly, we can replace different types of subnetworks to enable the model to adaptively learn the same image in different representational spaces, thereby avoiding the need for human-provided prior processing \cite{zhang2018mixup,kim2020puzzle,cubuk2018autoaugment,cubuk2020randaugment} at the image level and reducing additional optimization overhead. This enhancement is systematically illustrated through the intricate adaptive interactions depicted in Fig.~\ref{Figure1}.

Utilizing the pre-trained ViT architecture, our framework divides \( B \) input images, denoted as \( x \), into patch tokens of dimensionality \( D \) and count \( N \), further augmented with a class token \( E_{CLS} \) to establish the initial sequence \( I_0 = [E_{CLS}; E_1^0, E_2^0, \ldots, E_N^0] \). After transformation through the \( i \)-th transformer block, the sequence changes to:
\begin{equation}
    \begin{aligned}
        I_i = [E_{CLS}; E_1^i, E_2^i, \ldots, E_N^i] \in \mathbb{R}^{B \times (N+1) \times D}.
    \end{aligned}
    \label{A_eq0}
\end{equation}

PEFT-CL methodologies typically employ a prompt pool or introduce auxiliary parameters while preserving pre-trained weights, modifying $E_1^i, E_2^i, \ldots, E_N^i$ within each transformer block to influence the class token $E_{CLS}$. This generates a novel feature space that adapts to subtasks and mitigates catastrophic forgetting. In these methods, the predetermined task prompt pool is traditionally used to derive task-specific embeddings, selecting prompts through cosine similarity \cite{wang2022learning, wang2022dualprompt, jung2023generating}. While effective, this paradigm incurs substantial computational overhead when intervening in the self-attention mechanism and constrains the network's capacity for generating diverse, instance-specific adaptive interventions dynamically. To address these limitations, our proposed NTK-CL framework implements a more efficient paradigm utilizing additional trainable parameters to autonomously generates instance-specific interventions. These interventions then interact with our proposed feature space post-multi-head self-attention (MSA) module to yield task-specific embeddings. This approach not only maximizes the utilization of pre-trained knowledge but also effectively reduces the computational burden brought by intervening MSA calculations.

The input to the adaptation modules post-MSA module is structured as follows:
\begin{equation}
    \begin{aligned}
        u_i = MSA(I_i) \in \mathbb{R}^{B \times (N+1) \times D}.
    \end{aligned}
    \label{A_eq1}
\end{equation}

Next, we elucidate the generation processes for subnetwork-1 adaptation features, subnetwork-2 adaptation features, and hybrid adaptation features, which effectively triple the sample size in the feature space and reduce the generalization gaps in PEFT-CL training based on \cref{Generalization_Impact_Factor}.

\noindent\textbf{Creating Subnetwork-1 Adaptation Features:} 
To pinpoint the optimal interventions for enhancing the patch (\( N+1 \)) dimensionality within transformer blocks, we deploy a specialized subnetwork-1 adaptation module \( G_{S1} \). Tailored to the post-MSA inputs \( u_i \), \( G_{S1} \) adaptively transforms them into the most suitable prompts \( q_i \) for this task, as illustrated in Fig. \ref{Figure1} (right).
\begin{equation}
    \begin{aligned}
        q_i = G_{S1}(u_i; q_{i-1}) \in \mathbb{R}^{B \times (N+Q+1) \times D},
    \end{aligned}
    \label{A_eq2}
\end{equation}
where \( Q \) denotes the dimensionality of the prompts.

Delving into the details, within each transformer block, the prompt generator in \( G_{S1} \) (as a fully connected layer) condenses the dimensional knowledge and adds it residually to the prompts generated in the previous transformer block, ensuring the integrity of the optimized information. The generated prompts \( q_i \) are then concatenated with the input \( u_i \) and subsequently passed into the pre-trained fully connected layers of the transformer block for continued optimization.
\begin{equation}
\small
    \begin{aligned}
        SAE_i^1 = MLP_2(MLP_1([E_{CLS}; q_i; E_1^i, E_2^i, \ldots, E_N^i])),
    \end{aligned}
    \label{A_eq3}
\end{equation}
where \( SAE_i^1 \) represents the subnetwork-1 adaptation embeddings generated by the \( i \)-th transformer block.

After passing through all transformer blocks, we extract the final optimized \( SAE^{1}_* \) to obtain the subnetwork-1 adaptation features \( E_{CLS}^{S1} \), thereby constructing a feature space suited to patch-level knowledge for this task.

\noindent\textbf{Creating Subnetwork-2 Adaptation Features:} 
To enrich the embedding landscape and foster knowledge acquisition, we integrate the LORA architecture \cite{hu2022lora} as the subnetwork-2 adaptation module $G_{S2}$. Designed for efficient fine-tuning of pre-trained models by minimizing parameter adjustments, LORA enables the mastering of extensive knowledge in compact, low-rank representations while preserving efficacy during high-dimensional reconstructions. Our implementation bifurcates into \(G_{S2}^{low}\) for low-rank space mapping and \(G_{S2}^{high}\) for reconversion to the high-dimensional space.

Employing the input \(u_i\), \(G_{S2}\) follows a procedure akin to the prompt generator in \(G_{S1}\), generating the channel interventions \(c_i\). However, unlike in \(G_{S1}\), the generated \(c_i\) does not pass through the pre-trained fully connected layers.
\begin{equation}
    \begin{aligned}
        c_i = G_{S2}^{high}(G_{S2}^{low}(u_i)) \in \mathbb{R}^{B \times (N+1) \times D}.
    \end{aligned}
    \label{A_eq4}
\end{equation}

Considering that \(c_i\) and processed \(u_i\) by the pre-trained fully connected layers share identical dimensionalities, we opt for a summation rather than concatenation. This approach forms the subnetwork-2 adaptation embeddings $SAE^2_i$, streamlining the process and reducing computational overhead:
\begin{equation}
    \begin{aligned}
        SAE^2_i = c_i \oplus  MLP_2(MLP_1(u_i)) \in \mathbb{R}^{B \times (N+1) \times D}.
    \end{aligned}
    \label{A_eq5}
\end{equation}

Similarly, after passing through all transformer blocks, we also obtain the final optimized \(SAE^2_*\), from which we extract the subnetwork-2 adaptation features that are most suitable for this task's channel information, \(E_{CLS}^{S2}\), constructing the corresponding feature space.

\noindent\textbf{Synthesizing Hybrid Adaptation Features:} 
The primary objective of PEFT adaptations across both subnetwork-1 and subnetwork-2 is to increase the sample size within each task subset, thereby reducing the generalization gaps. However, this approach presents a dilemma: which feature space should be used to construct the prototype classifier? Our solution is to leverage all available spaces and creates an intermediate space that integrates the strengths of both, thereby expanding the sample size further. We integrate these spaces by merging the best of both worlds, ensuring a comprehensive and robust feature representation.

\begin{figure}[t]
\centering
\includegraphics[width=0.5\textwidth]{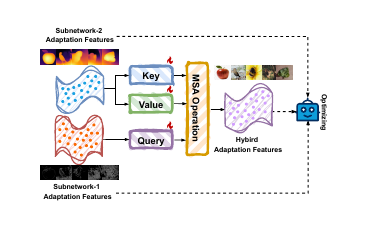}
\caption{The illustration depicts the fusion of multi-level features to generate three distinct features per sample, thereby increasing the sample size available for model optimization.}
\label{Figure2}
\vspace{-2ex}
\end{figure}

\begin{theorem}[Generalization in MSA] \label{MSA_Generalization}
    Given an iteration horizon \( K \geq 1 \), consider any parameter vector \( \boldsymbol{\theta} \in \mathbb{R}^{H(dT+d^2)} \) and a number of attention heads \( H \) satisfying:
    \begin{equation}
        \begin{aligned}
            \sqrt{H} \geq d T^{1/2} R^5 \|\boldsymbol{\theta}\|_{2, \infty} \|\boldsymbol{\theta} - \boldsymbol{\theta}_0\|^3.
        \end{aligned}
        \label{theorem3_0}
    \end{equation}
    Here, \(d\) specifies the dimensionality of the input features, while \(T\) indicates the sequence length. \(R\) is a constant inherent to the network's architecture, and \( \|\cdot\|_{2, \infty} \) represents the maximum \( \ell_2 \)-norm across the various parameter matrices. Additionally, the step-size \( \eta \) is required to comply with the following constraints:
    \begin{equation}
    \begin{aligned}
        \eta \leq \min\left\{1, \frac{1}{\rho(\boldsymbol{\theta})}, \frac{\|\boldsymbol{\theta} - \boldsymbol{\theta}_0\|^2}{K \hat{L}(\boldsymbol{\theta})}, \frac{\|\boldsymbol{\theta} - \boldsymbol{\theta}_0\|^2}{\hat{L}(\boldsymbol{\theta}_0)}\right\},
    \end{aligned}
    \label{theorem3_1}
    \end{equation}
    where \( \rho(\boldsymbol{\theta}) \) denotes the spectral radius, approximated by:
    \begin{equation}
    \begin{aligned}
        \rho(\boldsymbol{\theta}) \approx d^{3/2} T^{3/2} R^{13} \|\boldsymbol{\theta}\|_{2, \infty}^2 \|\boldsymbol{\theta} - \boldsymbol{\theta}_0\|^2.
    \end{aligned}
    \label{theorem3_2}
    \end{equation}
    Then, at iteration \( K \), the training loss \( \hat{L} \) and the norm of the weight differences are bounded as follows:
    \begin{equation}
    \begin{aligned}
        \hat{L}(\boldsymbol{\theta}_K) \leq \frac{1}{K} \sum_{k=1}^K \hat{L}(\boldsymbol{\theta}_k) + 2 \hat{L}(\boldsymbol{\theta}) + \frac{5\|\boldsymbol{\theta} - \boldsymbol{\theta}_0\|^2}{4\eta K},
    \end{aligned}
    \label{theorem3_3}
    \end{equation}
    \begin{equation}
    \begin{aligned}
        \|\boldsymbol{\theta}_K - \boldsymbol{\theta}_0\| \leq 4 \|\boldsymbol{\theta} - \boldsymbol{\theta}_0\|.
    \end{aligned}
    \label{theorem3_4}
    \end{equation}
    Furthermore, the expected generalization gap at iteration \( K \) is constrained by:
    \begin{equation}
    \small
    \begin{aligned}
        \mathbb{E}\left[L(\boldsymbol{\theta}_K) - \hat{L}(\boldsymbol{\theta}_K)\right] \leq \frac{4}{n} \mathbb{E} \left[2K \hat{L}(\boldsymbol{\theta}) + \frac{9 \|\boldsymbol{\theta} - \boldsymbol{\theta}_0\|^2}{4\eta}\right],
    \end{aligned}
    \label{theorem3_5}
    \end{equation}
    where expectations are computed over the randomness of the training set, \( n \) denotes the size of the dataset, and \( L \) and \( \hat{L} \) represent the empirical and population losses, respectively.
\end{theorem}

Drawing on insights from \cite{deora2023optimization}, we have refined elements of this work to develop \cref{MSA_Generalization}. This development definitively shows that the MSA module, under specified initialization conditions, offers robust generalization guarantees. Furthermore, the composition of the generalization gap and population loss aligns with our predefined standards: it is inversely proportional to the sample size, necessitates L2 regularization for bounded parameters, and mandates that patterns between samples be orthogonal with equal-energy means and exhibit NTK separability. This coherence reinforces the validity of our methodology and underpins our further innovations.

In our fusion architecture, the MSA module remains crucial for theoretical convergence and generalization optimization. Drawing inspiration from \cite{chen2021crossvit}, we implement an advanced fusion strategy by using \(E_{CLS}^{S2}\) as both the key and value, while \(E_{CLS}^{S1}\) serves as the query within the MSA mechanism. This configuration facilitates dynamic knowledge interchange between components, yielding a hybrid adaptation feature \(E_{CLS}^{HAE}\). This synergistic consolidation effectively doubles the MSA module's input dimensionality, theoretically reducing the generalization gaps and allowing the empirical loss to closely approximate the population loss, thereby approaching optimal parameter estimates. Figure \ref{Figure2} illustrates this integration.
\begin{equation}
\small
\begin{aligned}
    E_{CLS}^{HAE} = \text{Softmax}\left(\frac{Q(E_{CLS}^{S1}) \cdot K(E_{CLS}^{S2})^T}{\sqrt{\text{head\_dim}}}\right) \cdot V(E_{CLS}^{S2}),
\end{aligned}
\label{A_eq6}
\end{equation}
where \( Q \), \( K \), and \( V \) represent the query, key, and value operations in the self-attention mechanism, respectively.

At this point, for each sample, we obtain three features in different feature spaces: subnetwork-1 adaptation feature (\(E_{CLS}^{S1}\)), subnetwork-2 adaptation feature (\(E_{CLS}^{S2}\)), and hybrid adaptation feature (\(E_{CLS}^{HAE}\)). Among them, \(E_{CLS}^{HAE}\) is our preferred choice for constructing the prototype classifier.

Ultimately, by using these three features and their corresponding labels to construct a cross-entropy loss, we achieve a threefold expansion of the sample size within each finite task subset, effectively reducing generalization gaps:
\begin{equation}
\small
\begin{aligned}
    \mathcal{L}_{cls} = CE(E_{CLS}^{S1}, y) + CE(E_{CLS}^{S2}, y) + CE(E_{CLS}^{HAE}, y),
\end{aligned}
\label{A_eq7}
\end{equation}
where \( CE \) denotes the cross-entropy loss function, and \( y \) indicates the corresponding labels.

\subsection{Task-Level Feature Constraints} \label{Task-feature_Constraints}
Informed by insights from \cref{Task_Interplay_Generalization}, our approach underscores that effectively reducing generalization gap involves the diligent preservation of historical knowledge \(\Phi_\tau(X_\tau,X_\tau)\) and \(\Phi_k(X_k, X_k)\) from the perspective of the task \(T\), coupled with a concerted effort to diminish cross-task interactions \(\Phi_k(X_\tau, X_k)\), for \(k > \tau\). Given \(\Phi_k(X_\tau, X_k) = \frac{\partial f_k^*(X_\tau)}{\partial p_k} \frac{\partial f_k^*(X_k)}{\partial p_k}\), if the difference between \(f_k^*(X_\tau)\) and \(f_k^*(X_k)\) is maximized, then \(\Phi_k(X_\tau, X_k)\) will be minimized. Since \(p_k\) in the optimization process of PEFT-CL will only be influenced by \(f_k^*(X_k)\), ensuring orthogonality between \(f_k^*(X_\tau)\) and \(f_k^*(X_k)\) will make \(\frac{\partial f_k^*(X_\tau)}{\partial p_k}\) extremely small \cite{doan2021theoretical}. However, in the practical setting of PEFT-CL, cross-task access to data is strictly prohibited, presenting a substantial challenge in maintaining task-level distinctiveness.

Therefore, we propose a compromise approach. Within the context of NTK theory, the optimization of infinitely wide neural networks mirrors a Gaussian process \cite{chai2009generalization, lee2017deep}, yielding a locally constant NTK matrix \cite{jacot2018neural, lee2019wide, chizat2019lazy}. Given this, it is reasonable to assume that \(\Phi^*(X_\tau, X_k) = \Phi_0(X_\tau, X_k) = \Phi_1(X_\tau, X_k) = \cdots = \Phi_\infty(X_\tau, X_k)\). Moreover, networks pre-trained on extensive datasets emulate the properties of infinitely wide networks \cite{lee2020finite, wei2022toy, vyas2022limitations}, aligning with our pre-trained model. Therefore, we relax the original constraint, assuming that the pre-trained model is at this local optimum.

Under this framework, \(\Phi_k(X_\tau, X_k) \approx \Phi^*(X_\tau, X_k) = \frac{\partial f^*(X_\tau)}{\partial p} \cdot \frac{\partial f^*(X_k)}{\partial p}\), suggesting that ensuring orthogonality between \(f^*(X_\tau)\) and \(f^*(X_k)\) is feasible to some extent. To practically achieve this, integrating a prototype classifier and imposing orthogonality constraints ensure that embeddings from different tasks remain distinct, thus not violating the constraints under the PEFT-CL scenarios and aligning with the objective to minimize generalization gap.

\noindent\textbf{Knowledge Retention:} Achieving the retention of past knowledge is a critical component in traditional CL methods \cite{buzzega2020dark, li2024contrastive, li2023variational}. However, in PEFT-CL methods, this fundamental aspect has been notably underemphasized. Contemporary PEFT-CL methods predominantly involve repositioning prompts \cite{wang2022learning, gao2024consistent} or jointly utilizing all task-specific subnetwork parameters \cite{liang2024inflora, zhou2024expandable}, thereby shifting the focus away from the retention of past knowledge to the training performance of each task. However, such strategies necessitate the maintenance of optimal parameter configurations for each encountered task, which not only incurs substantial storage demands but also potentially limits the system's adaptability, particularly in scenarios characterized by a high density of tasks. To mitigate these challenges, we propose a paradigm shift that emphasizes the reevaluation of knowledge retention mechanism, eschewing the necessity for per-task parameter storage. Central to our method is the introduction of an adaptive Exponential Moving Average (EMA) mechanism. This mechanism, as depicted in Fig.~\ref{Figure3}, facilitates a more streamlined and scalable solution to the catastrophic forgetting problem, enhancing the overall efficiency and efficacy of PEFT-CL systems.

\begin{figure}[t]
\centering
\includegraphics[width=0.5\textwidth]{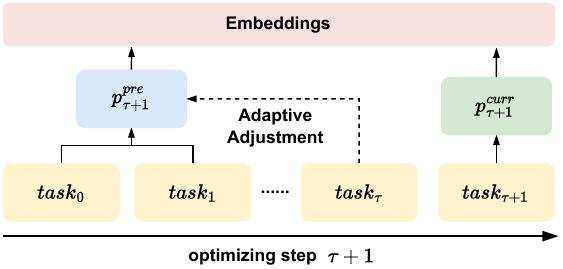}
\caption{Leveraging the adaptive EMA mechanism, we meticulously maintain a repository of visual summaries from the adaptation modules' parameters of prior tasks. The resulting network embedding is bifurcated into two distinct components: the pre-embedding, which retains historical knowledge, and the curr-embedding, which captures current insights. These segments are concatenated to create a composite embedding, ensuring a comprehensive representation that integrates past and present knowledge seamlessly.}
\label{Figure3}
\end{figure}

Traditional EMA applications often maintain a static base model, incrementally integrating optimized weights to preserve historical data. However, this approach proves suboptimal in PEFT-CL settings due to the substantial disparities in weights across tasks. Directly preserving a large proportion of past weights can detrimentally affect the performance on current task, while retaining an entire model's weights is excessively redundant. Therefore, we propose two improvements. First, we categorize the adaptation parameters responsible for generating embedding into two segments: \(p^{pre}\) for historical knowledge and \(p^{curr}\) for current insights. Secondly, we apply the EMA mechanism exclusively to the adaptation modules' parameters, leaving other optimizable parameters untouched to ensure the optimization remains streamlined. Throughout the optimization of task \(\tau+1\), only \(p^{curr}_{\tau+1}\) is modified, while \(p^{pre}_{\tau+1}\) is adaptively adjusted post-task-$\tau$ completion, employing an adaptive EMA scheme:
\begin{equation}
    \begin{aligned}
        \relax [k_1(n), k_2(n)] = 
        \begin{cases} 
        [0, 1] & \text{if } n = 0 \\
        [\frac{1}{{n+1}} \cdot \frac{1}{k_2(n-1)}, \frac{1}{{n+1}}] & \text{otherwise}
        \end{cases}
    \end{aligned}
    \label{A_eq8}
\end{equation}

\begin{equation}
    \begin{aligned}
        p^{pre}_{\tau+1} =  k_1(\tau) p^{pre}_{\tau} + k_2(\tau) p^{curr}_{\tau}.
    \end{aligned}
    \label{A_eq9}
\end{equation}

Under this mechanism, each past task equitably contributes to constructing embeddings related to historical knowledge without compromising the current task's insights, while avoiding the excessive memory overhead of storing parameters for each task, as seen in \cite{zhou2024expandable}. Consequently, \(E_{CLS}^{S1}\), \(E_{CLS}^{S2}\), and \(E_{CLS}^{HAE}\) all consist of two components: \(concat[f(x, p_{pre}), f(x, p_{curr})]\).

\noindent\textbf{Task-Feature Dissimilarity:} \footnote{Regarding why task-feature orthogonality does not impair the propagation and retention of knowledge among similar classes across different tasks, we provide further explanations in Appendix~\ref{PR_Explaination}.} 
Based on the findings in \cite{smith2023coda, qiao2023prompt}, it is evident that achieving class-level orthogonal insulation can effectively enhance the performance of PEFT-CL models. However, our theoretical analysis in Section~\ref{Task-feature_Constraints} and insights from \cite{doan2021theoretical, bennani2020generalisation} indicate that achieving task-level orthogonal insulation between \(f^*(X_\tau)\) and \(f^*(X_k)\) is sufficient to reduce the generalization gap and obtain good continual performance. This task-level orthogonal insulation not only simplifies the model requirements but also ensures robust and efficient learning across tasks.
Therefore, relying on the prototype classifier, we propose an optimization loss. In line with \cite{zhou2024expandable,wang2024hierarchical}, we update the prototype classifier \(\zeta\) upon completion of each task's optimization and strictly prohibit accessing previous samples in subsequent optimizations to comply with PEFT-CL constraints. During the optimization of task \(\tau\), we randomly sample \(\zeta_\tau\) from \(\zeta\) to represent \(f^*(X_\tau)\) \footnote{Sampling from the parameter space of the prototype classifier \(\zeta\), unlike approaches such as Hide-Prompt \cite{wang2024hierarchical} and APG \cite{tang2023prompt}, avoids compressing past embedding distributions and adding extra training overhead. This method also eliminates the need for a replay buffer, effectively bypassing the typical constraints associated with PEFT-CL.}. To initially distinguish \(f^*(X_\tau)\) from \(f^*(X_k)\), we use the InfoNCE \cite{oord2018representation} as a metric, employing \(\zeta_\tau\) as the negative sample, while using samples \(x_\tau\) (represented by \(E_{CLS}^{HAE}\), as this is the feature used for final classification) from task \(\tau\) as positive samples.
\begin{equation}
    \mathcal{L}_{dis} = -\frac{1}{|x_\tau|} \sum_{i \in |x_\tau|} \log \frac{\exp(\text{sim}(z_i, c_i))}{\sum_{j \in |\zeta_\tau|} \exp(\text{sim}(z_i, c_j))},
    \label{A_eq10}
\end{equation}
where \(|x_\tau|\) represents the number of positive samples, \(|\zeta_\tau|\) denotes the number of negative samples, \(z_i\) and \(c_i\) are the same-class positive samples used for optimization, and \(z_j\) is the negative samples sampled from the prototype classifier.

To further ensure orthogonality between \(f^*(X_\tau)\) and \(f^*(X_k)\), we apply the truncated SVD method \cite{hansen1987truncated} to constrain the optimization of \(f^*(X_k)\). Specifically, we decompose \(\zeta\) to obtain the orthogonal basis \(\mathbf{U}\) that defines the classification (preceding feature) space. We then map \(x_\tau\) into this space and remove the unmappable part from the original \(x_\tau\). When the retained mappable portion is sufficiently small, the orthogonality between \(x_\tau\) and \(\zeta\) is ensured.
\begin{equation}
    \mathcal{L}_{\text{orth}} = \sum_{i \in |x_\tau|} \| z_i - \tilde{proj}(z_i, U) \|_2^2,
    \label{A_eq11}
\end{equation}
where \(\tilde{proj}(a, b)\) represents the unmappable portion of \(a\) within the space spanned by the orthogonal basis functions decomposed from \(b\).

\subsection{Regularization Adjustment}
In accordance with the theoretical constraints delineated in Appendix~\ref{NTK_Dynamics}, which advocate for the incorporation of ridge regression to ensure a well-conditioned solution, we deploy an L2 regularization. As specified in Eq.~\ref{NTK_Dynamics_1}, the regularization term is structured as \(\left\|p_\tau - p_{\tau-1}^*\right\|_2^2\), targeting the parameter shifts from task \(\tau-1\) to task \(\tau\). Consequently, we meticulously design our regularization term to mirror this structure and temporarily retain the trainable parameters \(p^{pre}\) from the preceding task. This targeted regularization is then precisely applied to the parameters of the various modules within our NTK-CL, formulated as follows:
\begin{equation}
\small
    \mathcal{L}_{\text{reg}} = \| p_{G_{S1}}^{curr} - p_{G_{S1}}^{pre} \|_2^2 + \| p_{G_{S2}}^{curr} - p_{G_{S2}}^{pre} \|_2^2 + \| p_{G_H}^{curr} - p_{G_H}^{pre} \|_2^2,
    \label{A_eq12}
\end{equation}
where \(G_{S1}\), \(G_{S2}\), and \(G_H\) represent the trainable parameters of the subnetwork-1 adaptation module, the subnetwork-2 adaptation module, and the hybrid adaptation module.

\noindent\textbf{Training Optimization:} The composite objective for optimizing the training of each task subset within our NTK-CL is rigorously defined as follows:
\begin{equation}
    \begin{aligned}
        \mathcal{L}_{total} = \mathcal{L}_{cls} +  \eta \mathcal{L}_{dis} + \upsilon \mathcal{L}_{orth} + \lambda \mathcal{L}_{reg},
    \end{aligned}
    \label{overall_loss}
\end{equation}
where \(\eta\) and \(\upsilon\) are hyper-parameters, meticulously calibrated to maximize task-feature dissimilarity and to promote orthogonality in task-feature representations, respectively. The parameter \(\lambda\) controls the intensity of the regularization, ensuring the model's robustness and generalizability.

\noindent\textbf{Prototype Classifier:} Upon the completion of each task's training, we conduct an averaging operation on the features generated by all classes involved in that task to update the classifier \(\zeta\) with the most representative features of each class. It is important to note that the features used at this stage are designated as hybrid adaptation features \(E_{CLS}^{HAE}\).
\begin{equation}
    \zeta_i = \frac{1}{N_i} \sum_{j=1}^{N_i} E_{CLS, ij}^{HAE},
    \label{A_eq13}
\end{equation}
where \(N_i\) denotes the number of feature vectors for class \(i\) within the task, and \(E_{CLS, ij}^{HAE}\) represents the hybrid adaptation feature vector of the \(j\)-th sample in class \(i\).

Upon updating all class features within the task in the prototype classifier \(\zeta\), the system transitions to training the subsequent task. During this new training phase, there is a strict prohibition on accessing data from previous tasks, reinforcing the integrity of the continual learning process.

\noindent\textbf{Testing Evaluation:} Upon concluding the training regimen, the evaluation phase commences with simultaneous testing across all tasks. This phase distinctly prioritizes the synthesized hybrid adaptation features \(E_{CLS}^{HAE}\) for final analysis. Through the final prototype classifier \(\zeta\), these features are transformed into logits, which are aligned with the corresponding labels to deduce the test accuracy.

\begin{algorithm}[t]
\footnotesize 
\caption{NTK-CL Framework for PEFT-CL}
\label{alg:ntk_cl}
\begin{algorithmic}[1]
\Require Pre-trained model \( f^*_0 \); task set \( \mathcal{D} = \{\mathcal{D}_1, \ldots, \mathcal{D}_T\} \); initial PEFT parameters \( p_1 = p_1^{\text{pre}} \oplus p_1^{\text{curr}} \), where \( p_1^{\text{pre}} = p_{1, S1}^{\text{pre}} \oplus p_{1, S2}^{\text{pre}} \oplus p_{1, HAE}^{\text{pre}} \) and \( p_1^{\text{curr}} = p_{1, S1}^{\text{curr}} \oplus p_{1, S2}^{\text{curr}} \oplus p_{1, HAE}^{\text{curr}} \); hyper-parameters $\eta$, $\upsilon$, $\lambda$; learning rate $\xi$
\Ensure Trained PEFT-CL model \( f^*_T = f^*_0 \circ p^*_T \) with minimized generalization error

\StateColored{blue!14}{Initialize frozen \( p_1^{\text{pre}} \) and trainable \( p_1^{\text{curr}} \)}
\For{each task \( \mathcal{D}_\tau \) in \( \mathcal{D} \)}
    \StateColored{blue!7}{Retrieve task-specific data \( (X_{\tau}, Y_{\tau}) \)}
    \StateColored{blue!14}{Compute features:}
    \Statex \quad \( E_{CLS}^{S1} = (f^*_0 \circ p_{\tau, S1}^{\text{pre}})(X_{\tau}) \oplus (f^*_0 \circ p_{\tau, S1}^{\text{curr}})(X_{\tau}) \)
    \Statex \quad \( E_{CLS}^{S2} = (f^*_0 \circ p_{\tau, S2}^{\text{pre}})(X_{\tau}) \oplus (f^*_0 \circ p_{\tau, S2}^{\text{curr}})(X_{\tau}) \)
    \Statex \quad \( E_{CLS}^{HAE} = p_{\tau, HAE}^{\text{pre}}(E_{CLS}^{S1}, E_{CLS}^{S2}) \oplus p_{\tau, HAE}^{\text{curr}}(E_{CLS}^{S1}, E_{CLS}^{S2}) \)
    \StateColored{blue!7}{Compute the classification loss \( \mathcal{L}_{\text{cls}} \) as in Eq.~\ref{A_eq7}:}
    \Statex \quad \( \mathcal{L}_{\text{cls}} = CE(E_{CLS}^{S1}, Y_{\tau}) + CE(E_{CLS}^{S2}, Y_{\tau}) + CE(E_{CLS}^{HAE}, Y_{\tau}) \)
    \StateColored{blue!14}{Enforce task-level orthogonality constraints \( \mathcal{L}_{\text{dis}} \) and \( \mathcal{L}_{\text{orth}} \) for \( E_{CLS}^{HAE} \) according to Eq.~\ref{A_eq10} and Eq.~\ref{A_eq11}}
    \StateColored{blue!7}{Apply parameter regularization \( \mathcal{L}_{\text{reg}} \) as per Eq.~\ref{A_eq12}:}
    \Statex \quad \( \mathcal{L}_{\text{reg}} = \sum_{i \in \{S1, S2, HAE\}} \| p_{\tau, i}^{\text{curr}} - p_{\tau, i}^{\text{pre}} \|_2^2 \)
    \StateColored{blue!14}{Compute the overall loss using Eq.~\ref{overall_loss}:}
    \Statex \quad \( \mathcal{L}_{\text{total}} = \mathcal{L}_{\text{cls}} + \eta \mathcal{L}_{\text{dis}} + \upsilon \mathcal{L}_{\text{orth}} + \lambda \mathcal{L}_{\text{reg}} \)
    \StateColored{blue!7}{Update \( p_{\tau}^{\text{curr}} \) using backpropagation optimization:}
    \Statex \quad \( p_{\tau}^{\text{curr}} \leftarrow p_{\tau}^{\text{curr}} - \xi \nabla_{p_{\tau}^{\text{curr}}} \mathcal{L}_{\text{total}} \)
    \StateColored{blue!14}{Update the prototype classifier with \( E_{\text{CLS}}^{\text{HAE}} \)}
    \If{\( \tau \) is not the last task}
        \StateColored{blue!7}{Perform Adaptive EMA updates on \( p^{\text{pre}}_\tau \) as per Eq.~\ref{A_eq8} and Eq.~\ref{A_eq9}}
    \EndIf
\EndFor
\StateColored{blue!14}{\Return Final model \( f^*_T \)}
\end{algorithmic}
\end{algorithm}

In summary, the operational sequence of our NTK-CL framework is encapsulated in Algorithm~\ref{alg:ntk_cl}.

\begin{table}[t]
\centering
\caption{Summary of datasets for the PEFT-CL settings, detailing task counts, class counts, image totals, and domains.}
\label{table:dataset_summary}
\begin{tabular}{cccccc}
\toprule
\textbf{Dataset} & \textbf{Task} & \textbf{Class} & \textbf{Image} & \textbf{Domain} \\
\midrule
CIFAR-100 & 10 & 100 & 60000 & Object Recognition \\
ImageNet-R & 10 & 200 & 30000 & Object Recognition \\
ImageNet-A & 10 & 200 & 7500 & Object Recognition \\
DomainNet & 15 & 345 & 423506 & Domain Adaptation \\
Oxford Pets & 7 & 37 & 7393 & Animal Recognition \\
EuroSAT & 5 & 10 & 27000 & Earth Observation \\
PlantVillage & 5 & 15 & 20638 & Agricultural Studies \\
VTAB & 5 & 50 & 10415 & Task Adaptation \\
Kvasir & 4 & 8 & 4000 & Healthcare Diagnosis \\
\bottomrule
\end{tabular}
\end{table}

\section{Experiments}
In this study, we utilize a carefully curated suite of benchmark datasets designed to support a rigorous and comprehensive evaluation of model generalization within the PEFT-CL paradigm. These datasets encompass a wide spectrum of domains, including general object recognition, domain adaptation, fine-grained animal classification, earth observation, agricultural analytics, task adaptation, and healthcare diagnostics. This diverse selection ensures a robust evaluation framework that captures the complexities and challenges inherent in real-world applications. Detailed descriptions of each dataset, along with the corresponding training protocols, evaluation metrics, and implementation specifics, are provided in Table~\ref{table:dataset_summary} and Appendix~\ref{appendix:datasets}, thereby promoting reproducibility and facilitating a clearer understanding.

\begin{table*}[htbp]
\centering
\fontsize{8pt}{10pt}\selectfont
\caption{Comparative performance analysis in PEFT-CL using ViT-Base16, pre-trained on ImageNet-21K, as the foundational model. Bold segments indicate optimal results, while underlined segments denote suboptimal results.}
\label{tab:accuracy_1}
\begin{tabular}{>{\centering\arraybackslash}p{3cm} >{\centering\arraybackslash}p{2cm} cccccc}
\toprule
\multirow{2}{*}{\textbf{Method}} & \multirow{2}{*}{\textbf{Publisher}} & \multicolumn{2}{c}{\textbf{CIFAR-100}} & \multicolumn{2}{c}{\textbf{ImageNet-R}} & \multicolumn{2}{c}{\textbf{ImageNet-A}} \\
\cmidrule(lr){3-4} \cmidrule(lr){5-6} \cmidrule(lr){7-8}
& & $\bar{A}$ (\%) & $A_T$ (\%) & $\bar{A}$ (\%) & $A_T$ (\%) & $\bar{A}$ (\%) & $A_T$ (\%) \\
\midrule
L2P \cite{wang2022learning} & CVPR 2022 & 89.30 $\pm$ 0.34 & 84.16 $\pm$ 0.72 & 72.38 $\pm$ 0.89 & 65.57 $\pm$ 0.67 & 47.86 $\pm$ 1.26 & 38.08 $\pm$ 0.79 \\
DualPrompt \cite{wang2022dualprompt} & ECCV 2022 & 90.68 $\pm$ 0.21 & 85.76 $\pm$ 0.45 & 72.45 $\pm$ 0.94 & 66.31 $\pm$ 0.55 & 52.16 $\pm$ 0.83 & 40.07 $\pm$ 1.69 \\
CODA-Prompt \cite{smith2023coda} & CVPR 2023 & 91.36 $\pm$ 0.18 & 86.70 $\pm$ 0.28 & 77.16 $\pm$ 0.65 & 71.59 $\pm$ 0.53 & 56.13 $\pm$ 2.51 & 45.34 $\pm$ 0.92 \\
EvoPrompt \cite{kurniawan2024evolving} & AAAI 2024 & 92.06 $\pm$ 0.37 & 87.78 $\pm$ 0.63 & 78.84 $\pm$ 1.13 & 73.60 $\pm$ 0.39 & 54.88 $\pm$ 1.21 & 44.31 $\pm$ 0.88 \\
OVOR \cite{huang2024ovor} & ICLR 2024 & 91.11 $\pm$ 0.38 & 86.36 $\pm$ 0.38 & 75.63 $\pm$ 1.08 & 70.48 $\pm$ 0.19 & 53.33 $\pm$ 1.11 & 42.88 $\pm$ 0.67 \\
L2P-PGP \cite{qiao2023prompt} & ICLR 2024 & 89.61 $\pm$ 0.64 & 84.23 $\pm$ 0.87 & 74.91 $\pm$ 1.50 & 68.06 $\pm$ 0.46 & 50.57 $\pm$ 0.15 & 39.75 $\pm$ 1.02 \\
CPrompt \cite{gao2024consistent} & CVPR 2024 & 91.58 $\pm$ 0.52 & 87.17 $\pm$ 0.32 & 81.02 $\pm$ 0.33 & 75.30 $\pm$ 0.57 & 60.10 $\pm$ 1.34 & 49.78 $\pm$ 0.87 \\
EASE \cite{zhou2024expandable} & CVPR 2024 & 92.58 $\pm$ 0.48 & 88.11 $\pm$ 0.67 & \underline{81.92 $\pm$ 0.48} & \underline{76.04 $\pm$ 0.19} & \underline{64.35 $\pm$ 1.41} & \underline{54.64 $\pm$ 0.70} \\
InfLoRA \cite{liang2024inflora} & CVPR 2024 & 91.96 $\pm$ 0.24 & 86.93 $\pm$ 0.90 & 81.63 $\pm$ 0.82 & 75.53 $\pm$ 0.53 & 55.50 $\pm$ 0.85 & 44.21 $\pm$ 1.77 \\
C-ADA \cite{gao2024beyond} & ECCV 2024 & 92.16 $\pm$ 0.41 & 87.54 $\pm$ 0.14 & 79.41 $\pm$ 0.98 & 73.77 $\pm$ 0.55 & 56.96 $\pm$ 1.72 & 46.00 $\pm$ 0.91 \\
VPT-NSP \cite{lu2024visual} & NeurIPS 2024 & \underline{92.93 $\pm$ 0.32} & \underline{88.79 $\pm$ 0.45} & 81.80 $\pm$ 0.70 & 76.03 $\pm$ 0.27 & 60.86 $\pm$ 0.93 & 50.03 $\pm$ 0.75 \\
\rowcolor{blue!14}
NTK-CL (Ours) & - & \textbf{93.76 $\pm$ 0.35} & \textbf{90.27 $\pm$ 0.20} & \textbf{82.77 $\pm$ 0.66} & \textbf{77.17 $\pm$ 0.19} & \textbf{66.56 $\pm$ 1.53} & \textbf{58.54 $\pm$ 0.91} \\
\bottomrule
\end{tabular}
\end{table*}

\begin{table*}[htbp]
\centering
\fontsize{8pt}{10pt}\selectfont
\caption{Comparative performance analysis in PEFT-CL using ViT-Base16, fine-tuned on ImageNet-1K, as the foundational model. Bold segments indicate optimal results, while underlined segments denote suboptimal results.}
\label{tab:accuracy_2}
\begin{tabular}{>{\centering\arraybackslash}p{3cm} >{\centering\arraybackslash}p{2cm} cccccc}
\toprule
\multirow{2}{*}{\textbf{Method}} & \multirow{2}{*}{\textbf{Publisher}} & \multicolumn{2}{c}{\textbf{CIFAR-100}} & \multicolumn{2}{c}{\textbf{ImageNet-R}} & \multicolumn{2}{c}{\textbf{ImageNet-A}} \\
\cmidrule(lr){3-4} \cmidrule(lr){5-6} \cmidrule(lr){7-8}
& & $\bar{A}$ (\%) & $A_T$ (\%) & $\bar{A}$ (\%) & $A_T$ (\%) & $\bar{A}$ (\%) & $A_T$ (\%) \\
\midrule
L2P \cite{wang2022learning} & CVPR 2022 & 87.86 $\pm$ 0.23 & 81.62 $\pm$ 0.75 & 72.38 $\pm$ 0.89 & 65.57 $\pm$ 0.67 & 53.42 $\pm$ 0.95 & 44.98 $\pm$ 1.31 \\
DualPrompt \cite{wang2022dualprompt} & ECCV 2022 & 88.96 $\pm$ 0.36 & 83.50 $\pm$ 0.67 & 72.45 $\pm$ 0.94 & 66.31 $\pm$ 0.55 & 57.56 $\pm$ 1.02 & 47.85 $\pm$ 0.47 \\
CODA-Prompt \cite{smith2023coda} & CVPR 2023 & 91.22 $\pm$ 0.48 & 86.43 $\pm$ 0.23 & 77.67 $\pm$ 1.36 & 72.00 $\pm$ 1.33 & 61.28 $\pm$ 0.90 & 51.80 $\pm$ 0.79 \\
EvoPrompt \cite{kurniawan2024evolving} & AAAI 2024 & 91.89 $\pm$ 0.45 & \underline{87.56 $\pm$ 0.23} & 81.43 $\pm$ 1.07 & 75.86 $\pm$ 0.33 & 58.46 $\pm$ 1.10 & 48.13 $\pm$ 0.38 \\
OVOR \cite{huang2024ovor} & ICLR 2024 & 89.50 $\pm$ 0.60 & 84.26 $\pm$ 0.60 & 78.61 $\pm$ 1.00 & 73.18 $\pm$ 0.49 & 59.50 $\pm$ 1.00 & 50.10 $\pm$ 1.00 \\
L2P-PGP \cite{qiao2023prompt} & ICLR 2024 & 89.49 $\pm$ 0.48 & 84.63 $\pm$ 0.40 & 72.05 $\pm$ 0.85 & 66.42 $\pm$ 0.57 & 47.28 $\pm$ 1.23 & 39.21 $\pm$ 0.93 \\
CPrompt \cite{gao2024consistent} & CVPR 2024 & 91.74 $\pm$ 0.43 & 87.51 $\pm$ 0.38 & 82.20 $\pm$ 0.89 & 76.77 $\pm$ 0.64 & 55.07 $\pm$ 20.79 & 46.99 $\pm$ 17.03 \\
EASE \cite{zhou2024expandable} & CVPR 2024 & 91.88 $\pm$ 0.48 & 87.45 $\pm$ 0.34 & \underline{82.59 $\pm$ 0.70} & 77.12 $\pm$ 0.23 & \underline{67.36 $\pm$ 0.94} & \underline{58.28 $\pm$ 0.82} \\
InfLoRA \cite{liang2024inflora} & CVPR 2024 & 91.47 $\pm$ 0.65 & 86.44 $\pm$ 0.47 & 82.50 $\pm$ 1.00 & 76.68 $\pm$ 0.60 & 58.65 $\pm$ 1.39 & 47.31 $\pm$ 0.99 \\
C-ADA \cite{gao2024beyond} & ECCV 2024 & \underline{92.40 $\pm$ 0.32} & 87.46 $\pm$ 0.56 & 80.68 $\pm$ 1.19 & 74.90 $\pm$ 0.46 & 61.90 $\pm$ 1.26 & 50.90 $\pm$ 0.43 \\
VPT-NSP \cite{lu2024visual} & NeurIPS 2024 & 91.11 $\pm$ 0.58 & 85.80 $\pm$ 1.32 & 82.58 $\pm$ 0.74 & \underline{77.41 $\pm$ 0.61} & 62.13 $\pm$ 1.07 & 51.30 $\pm$ 0.88 \\
\rowcolor{blue!14}
NTK-CL (Ours) & - & \textbf{93.16 $\pm$ 0.46} & \textbf{89.43 $\pm$ 0.34} & \textbf{83.18 $\pm$ 0.40} & \textbf{77.76 $\pm$ 0.25} & \textbf{68.76 $\pm$ 0.71} & \textbf{60.58 $\pm$ 0.56} \\
\bottomrule
\end{tabular}
\end{table*}

\begin{table*}[htbp]
\centering
\fontsize{8pt}{10pt}\selectfont
\caption{Performance analysis in the PEFT-CL context utilizes ViT-Base16, pre-trained on ImageNet-21K, across various datasets. The bold segments denote optimal results, and the underlined segments indicate suboptimal outcomes.}
\label{tab:accuracy_extended_21k}
\begin{tabular}{>{\centering\arraybackslash}p{3cm} >{\centering\arraybackslash}p{2cm} cccccc}
\toprule
\multirow{2}{*}{\textbf{Method}} & \multirow{2}{*}{\textbf{Publisher}} & \multicolumn{2}{c}{\textbf{DomainNet}} & \multicolumn{2}{c}{\textbf{Oxford Pets}} & \multicolumn{2}{c}{\textbf{EuroSAT}} \\
\cmidrule(lr){3-4} \cmidrule(lr){5-6} \cmidrule(lr){7-8}
& & $\bar{A}$ (\%) & $A_T$ (\%) & $\bar{A}$ (\%) & $A_T$ (\%) & $\bar{A}$ (\%) & $A_T$ (\%) \\
\midrule
OVOR \cite{huang2024ovor} & ICLR 2024 & \textbf{75.80 $\pm$ 0.40} & \textbf{68.77 $\pm$ 0.12} & 91.56 $\pm$ 1.87 & 84.08 $\pm$ 1.37 & 78.97 $\pm$ 3.71 & 63.24 $\pm$ 5.22 \\
EASE \cite{zhou2024expandable} & CVPR 2024 & 71.82 $\pm$ 0.49 & 65.42 $\pm$ 0.14 & \underline{94.71 $\pm$ 1.79} & \underline{89.97 $\pm$ 1.56} & \underline{87.61 $\pm$ 2.36} & \underline{77.73 $\pm$ 2.41} \\
InfLoRA \cite{liang2024inflora} & CVPR 2024 & 69.74 $\pm$ 0.31 & 57.35 $\pm$ 0.32 & 59.46 $\pm$ 0.80 & 29.15 $\pm$ 1.04 & 81.82 $\pm$ 3.27 & 70.04 $\pm$ 5.78 \\
\rowcolor{blue!14}
NTK-CL (Ours) & - & \underline{73.70 $\pm$ 0.47} & \underline{67.44 $\pm$ 0.36} & \textbf{96.69 $\pm$ 0.99} & \textbf{94.11 $\pm$ 0.09} & \textbf{87.63 $\pm$ 2.32} & \textbf{79.84 $\pm$ 0.32} \\
\midrule
\multicolumn{2}{c}{\textbf{Additional Datasets}} & \multicolumn{2}{c}{\textbf{PlantVillage}} & \multicolumn{2}{c}{\textbf{VTAB}} & \multicolumn{2}{c}{\textbf{Kvasir}} \\
\cmidrule(lr){3-4} \cmidrule(lr){5-6} \cmidrule(lr){7-8}
OVOR \cite{huang2024ovor} & ICLR 2024 & 81.08 $\pm$ 2.73 & 65.96 $\pm$ 3.93 & 85.14 $\pm$ 3.14 & 77.55 $\pm$ 3.36 & 77.27 $\pm$ 2.28 & 58.3 $\pm$ 2.69 \\
EASE \cite{zhou2024expandable} & CVPR 2024 & \textbf{88.79 $\pm$ 4.43} & \underline{80.92 $\pm$ 6.18} & \textbf{89.81 $\pm$ 1.68} & \underline{84.76 $\pm$ 1.10} & \underline{84.35 $\pm$ 2.22} & \underline{69.32 $\pm$ 5.35} \\
InfLoRA \cite{liang2024inflora} & CVPR 2024 & \underline{88.61 $\pm$ 4.23} & 80.34 $\pm$ 5.72 & 85.04 $\pm$ 2.21 & 78.17 $\pm$ 3.69 & 80.62 $\pm$ 1.70 & 59.0 $\pm$ 4.13 \\
\rowcolor{blue!14}
NTK-CL (Ours) & - & 88.00 $\pm$ 2.37 & \textbf{81.88 $\pm$ 0.25} & \underline{89.67 $\pm$ 1.88} & \textbf{85.53 $\pm$ 0.81} & \textbf{90.03 $\pm$ 0.73} & \textbf{82.7 $\pm$ 0.55} \\
\bottomrule
\end{tabular}
\end{table*}

\begin{table*}[htbp]
\centering
\fontsize{8pt}{10pt}\selectfont
\caption{Performance analysis in the PEFT-CL context utilizes ViT-Base16, fine-tuned on ImageNet-1K, across various datasets. The bold segments denote optimal results, and the underlined segments indicate suboptimal outcomes.}
\label{tab:accuracy_extended_1k}
\begin{tabular}{>{\centering\arraybackslash}p{3cm} >{\centering\arraybackslash}p{2cm} cccccc}
\toprule
\multirow{2}{*}{\textbf{Method}} & \multirow{2}{*}{\textbf{Publisher}} & \multicolumn{2}{c}{\textbf{DomainNet}} & \multicolumn{2}{c}{\textbf{Oxford Pets}} & \multicolumn{2}{c}{\textbf{EuroSAT}} \\
\cmidrule(lr){3-4} \cmidrule(lr){5-6} \cmidrule(lr){7-8}
& & $\bar{A}$ (\%) & $A_T$ (\%) & $\bar{A}$ (\%) & $A_T$ (\%) & $\bar{A}$ (\%) & $A_T$ (\%) \\
\midrule
OVOR \cite{huang2024ovor} & ICLR 2024 & \underline{71.92 $\pm$ 0.41} & 63.87 $\pm$ 0.81 & 90.47 $\pm$ 1.51 & 82.85 $\pm$ 1.72 & 78.67 $\pm$ 2.74 & 62.88 $\pm$ 7.41 \\
EASE \cite{zhou2024expandable} & CVPR 2024 & 71.30 $\pm$ 0.50 & \underline{64.84 $\pm$ 0.13} & \underline{94.94 $\pm$ 1.60} & \underline{90.16 $\pm$ 1.48} & \textbf{91.06 $\pm$ 0.76} & \textbf{82.77 $\pm$ 2.15} \\
InfLoRA \cite{liang2024inflora} & CVPR 2024 & 69.19 $\pm$ 0.31 & 56.66 $\pm$ 0.33 & 58.71 $\pm$ 1.55 & 28.36 $\pm$ 1.89 & 82.12 $\pm$ 2.00 & 71.94 $\pm$ 6.09 \\
\rowcolor{blue!14}
NTK-CL (Ours) & - & \textbf{72.94 $\pm$ 0.27} & \textbf{66.73 $\pm$ 0.21} & \textbf{96.62 $\pm$ 0.75} & \textbf{94.28 $\pm$ 0.09} & \underline{88.50 $\pm$ 2.61} & \underline{80.94 $\pm$ 0.30} \\
\midrule
\multicolumn{2}{c}{\textbf{Additional Datasets}} & \multicolumn{2}{c}{\textbf{PlantVillage}} & \multicolumn{2}{c}{\textbf{VTAB}} & \multicolumn{2}{c}{\textbf{Kvasir}} \\
\cmidrule(lr){3-4} \cmidrule(lr){5-6} \cmidrule(lr){7-8}
OVOR \cite{huang2024ovor} & ICLR 2024 & 80.74 $\pm$ 2.70 & 65.14 $\pm$ 4.21 & 84.87 $\pm$ 3.57 & 76.02 $\pm$ 2.96 & 76.84 $\pm$ 3.91 & 56.48 $\pm$ 2.97 \\
EASE \cite{zhou2024expandable} & CVPR 2024 & \textbf{88.50 $\pm$ 4.55} & 80.75 $\pm$ 5.68 & \underline{88.45 $\pm$ 1.69} & \underline{82.55 $\pm$ 2.02} & \underline{84.30 $\pm$ 1.39} & \underline{69.65 $\pm$ 3.51} \\
InfLoRA \cite{liang2024inflora} & CVPR 2024 & \underline{87.98 $\pm$ 4.54} & \textbf{81.54 $\pm$ 4.69} & 86.99 $\pm$ 2.74 & 78.19 $\pm$ 3.01 & 77.50 $\pm$ 5.19 & 58.72 $\pm$ 9.12 \\
\rowcolor{blue!14}
NTK-CL (Ours) & - & 87.26 $\pm$ 2.16 & \underline{80.81 $\pm$ 0.22} & \textbf{88.48 $\pm$ 2.25} & \textbf{83.47 $\pm$ 1.90} & \textbf{91.88 $\pm$ 1.15} & \textbf{84.72 $\pm$ 0.40} \\
\bottomrule
\end{tabular}
\end{table*}

\begin{table*}[ht]
    \centering
    \caption{Ablation study on the \textit{ViT-B/16-IN21K} model, evaluating its performance across the CIFAR100 and ImageNet-R datasets. The study features a detailed breakdown of model components, denoted in each column by the inclusion (\checkmark) of specific modules and strategies: Subnetwork-1 Adaptation Module (S1), Subnetwork-2 Adaptation Module (S2), Hybrid Adaptation Module (Hybrid), Knowledge Retention (KR), Task-Feature Dissimilarity Loss (Dis), Orthogonality Loss (Orth), and Regularization Loss (Reg). Incremental accuracies ($\bar{A}$) are reported to highlight their respective impacts on performance.}
    \scriptsize
    \begin{tabular}{ccccccc|c|ccccccc|c}
        \toprule
        \multicolumn{7}{c|}{\textbf{Frozen \textit{ViT-B/16-IN21K} on CIFAR100}} & & \multicolumn{7}{c|}{\textbf{Frozen \textit{ViT-B/16-IN21K} on ImageNet-R}} & \\
        \cmidrule(lr){1-7} \cmidrule(lr){9-15}
        \multicolumn{3}{c|}{\textbf{Adaptation Modules}} & \multicolumn{3}{c|}{\textbf{Task Constraints}} & \multicolumn{1}{c|}{\textbf{Regularization}} & \multirow{2}{*}{\footnotesize \textbf{$\bar{A}$} (\%) } & \multicolumn{3}{c|}{\textbf{Adaptation Modules}} & \multicolumn{3}{c|}{\textbf{Task Constraints}} & \multicolumn{1}{c|}{\textbf{Regularization}} & \multirow{2}{*}{\footnotesize \textbf{$\bar{A}$} (\%) } \\
        \cmidrule(lr){1-7} \cmidrule(lr){9-15} 
        \textbf{S1} & \textbf{S2} & \textbf{Hybrid} & \textbf{KR} & \textbf{Dis} & \textbf{Orth} & \textbf{Reg} & & \textbf{S1} & \textbf{S2} & \textbf{Hybrid} & \textbf{KR} & \textbf{Dis} & \textbf{Orth} & \textbf{Reg} & \\
        \midrule
        \checkmark & & & & & & & 83.22 $\pm$ 0.41 & \checkmark & & & & & & & 67.78 $\pm$ 2.99 \\
        \rowcolor{blue!14}
        \checkmark & & & \checkmark & & & & 84.61 $\pm$ 1.69 & \checkmark & & & \checkmark & & & & 69.89 $\pm$ 2.35 \\
        & \checkmark & & & & & & 85.14 $\pm$ 0.35 & & \checkmark & & & & & & 68.84 $\pm$ 0.33 \\
        \rowcolor{blue!14}
        & \checkmark & & \checkmark & & & & 91.29 $\pm$ 0.70 & & \checkmark & & \checkmark & & & & 77.63 $\pm$ 0.51 \\
        & & \checkmark & & & & & 86.75 $\pm$ 0.32 & & & \checkmark & & & & & 71.90 $\pm$ 0.29 \\
        \rowcolor{blue!14}
        & & \checkmark & \checkmark & & & & 88.92 $\pm$ 0.60 & & & \checkmark & \checkmark & & & & 77.39 $\pm$ 0.58 \\
        \checkmark & \checkmark & \checkmark & & & & & 89.18 $\pm$ 0.35 & \checkmark & \checkmark & \checkmark & & & & & 75.94 $\pm$ 0.27 \\
        \rowcolor{blue!14}
        \checkmark & \checkmark & \checkmark & \checkmark & & & & 91.89 $\pm$ 0.36 & \checkmark & \checkmark & \checkmark & \checkmark & & & & 80.91 $\pm$ 0.31 \\
        \rowcolor{blue!14}
        \checkmark & \checkmark & \checkmark & \checkmark & \checkmark & & & 92.83 $\pm$ 0.56 & \checkmark & \checkmark & \checkmark & \checkmark & \checkmark & & & 82.12 $\pm$ 0.59 \\
        \rowcolor{blue!14}
        \checkmark & \checkmark & \checkmark & \checkmark & & \checkmark & & 92.39 $\pm$ 0.36 & \checkmark & \checkmark & \checkmark & \checkmark & & \checkmark & & 81.26 $\pm$ 0.41 \\
        \rowcolor{blue!14}
        \checkmark & \checkmark & \checkmark & \checkmark & & & \checkmark & 92.16 $\pm$ 0.36 & \checkmark & \checkmark & \checkmark & \checkmark & & & \checkmark & 81.41 $\pm$ 0.41 \\
        \rowcolor{blue!14}
        \checkmark & \checkmark & \checkmark & \checkmark & \checkmark & \checkmark & & 93.06 $\pm$ 0.55 & \checkmark & \checkmark & \checkmark & \checkmark & \checkmark & \checkmark & & 82.50 $\pm$ 0.61 \\
        \rowcolor{blue!14}
        \checkmark & \checkmark & \checkmark & \checkmark & \checkmark & \checkmark & \checkmark & 93.76 $\pm$ 0.35 & \checkmark & \checkmark & \checkmark & \checkmark & \checkmark & \checkmark & \checkmark & 82.77 $\pm$ 0.66 \\
        \bottomrule
    \end{tabular}
    \label{ablation_study}
    \vspace{-2ex}
\end{table*}

\captionsetup[subfloat]{labelfont={bf,small}, textfont={small}}
\begin{figure*}[ht]
    \centering
    \subfloat[CIFAR100-$\eta$ Tuning]{\includegraphics[width=0.24\linewidth]{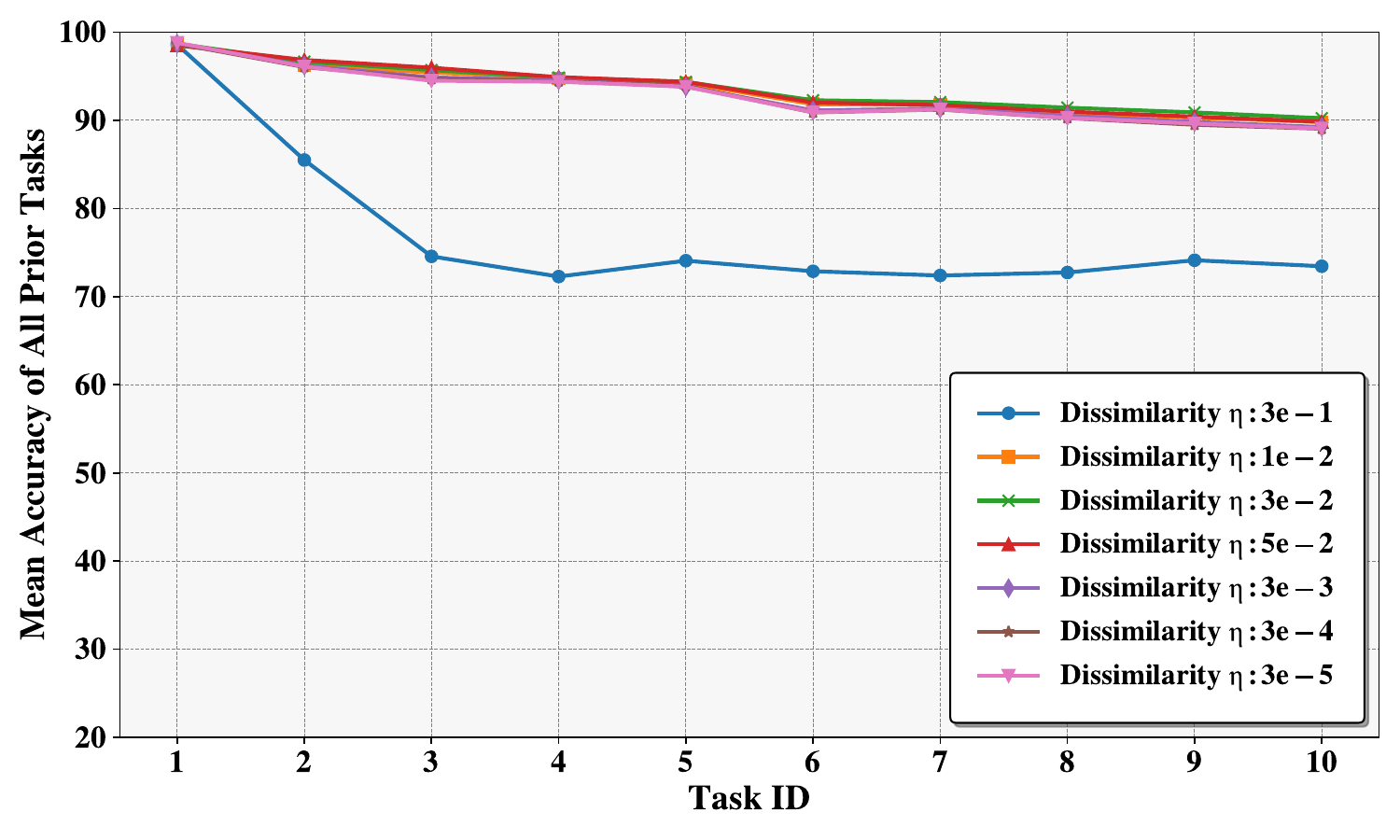}}
    \hfil
    \subfloat[CIFAR100-$\upsilon$ Tuning]{\includegraphics[width=0.24\linewidth]{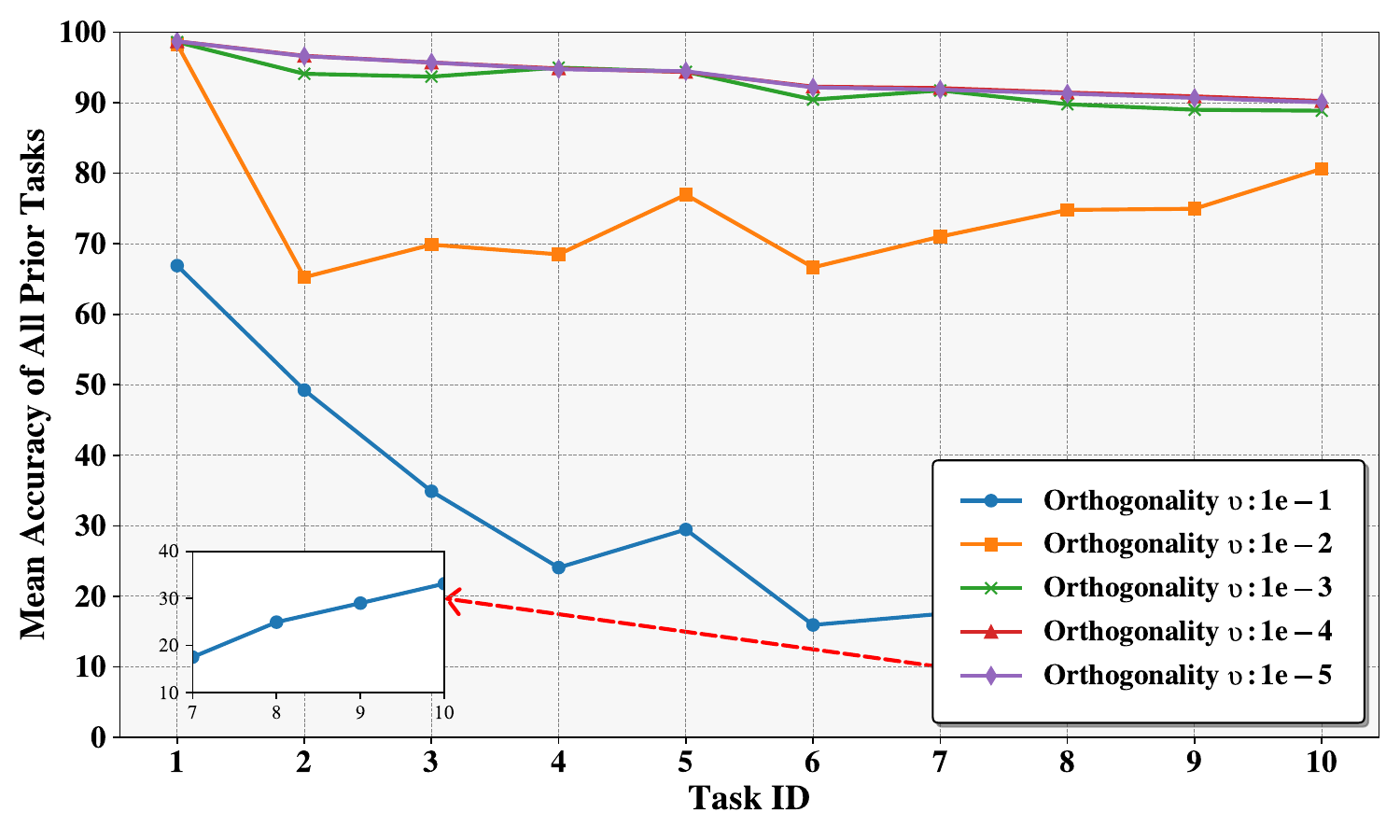}}
    \hfil
    \subfloat[CIFAR100-$\lambda$ Tuning]{\includegraphics[width=0.24\linewidth]{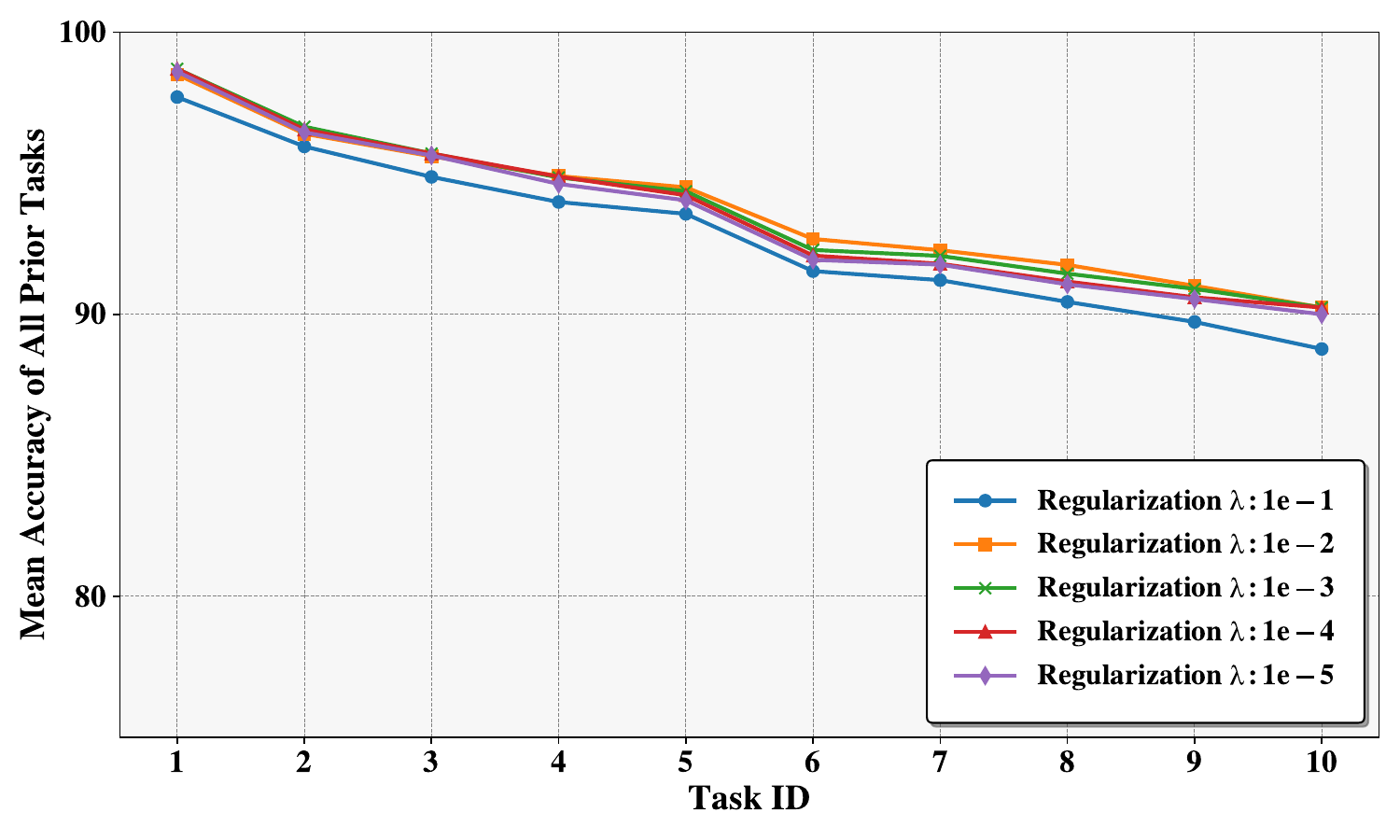}}
    \hfil
    \subfloat[CIFAR100 AHPS]{\includegraphics[width=0.235\linewidth]{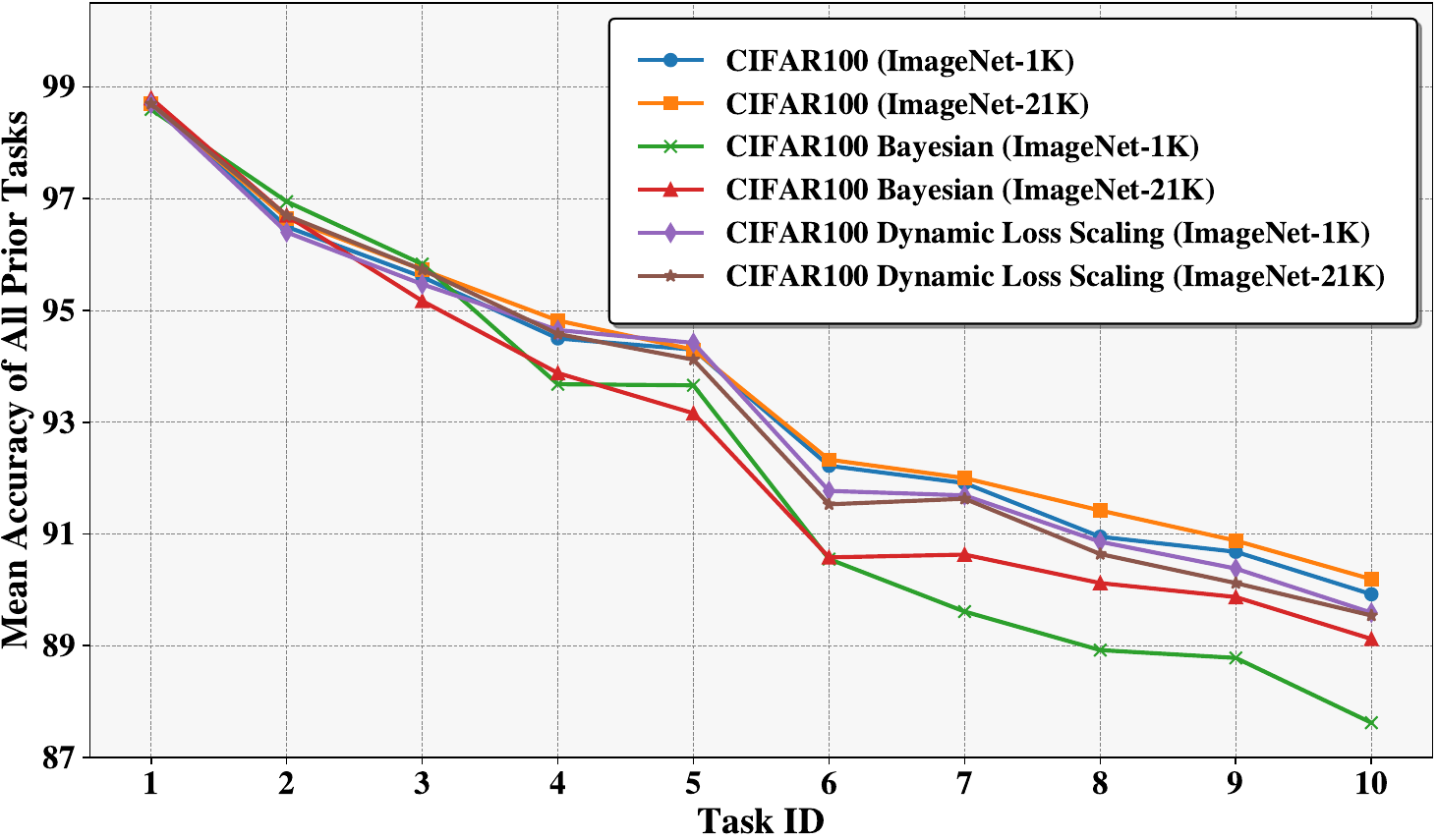}} \\
    \vfil
    \subfloat[ImageNet-R-$\eta$ Tuning]{\includegraphics[width=0.24\linewidth]{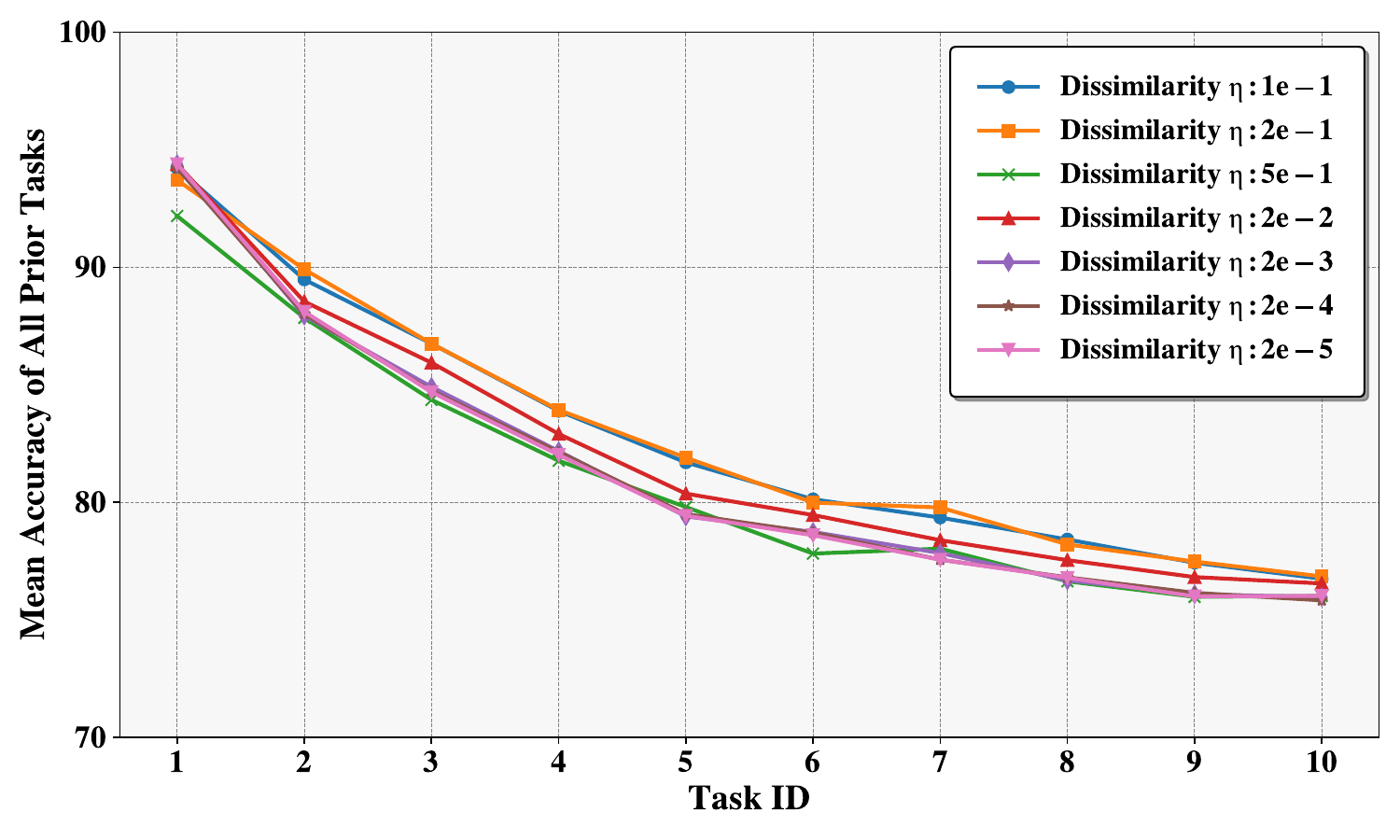}}
    \hfil
    \subfloat[ImageNet-R-$\upsilon$ Tuning]{\includegraphics[width=0.24\linewidth]{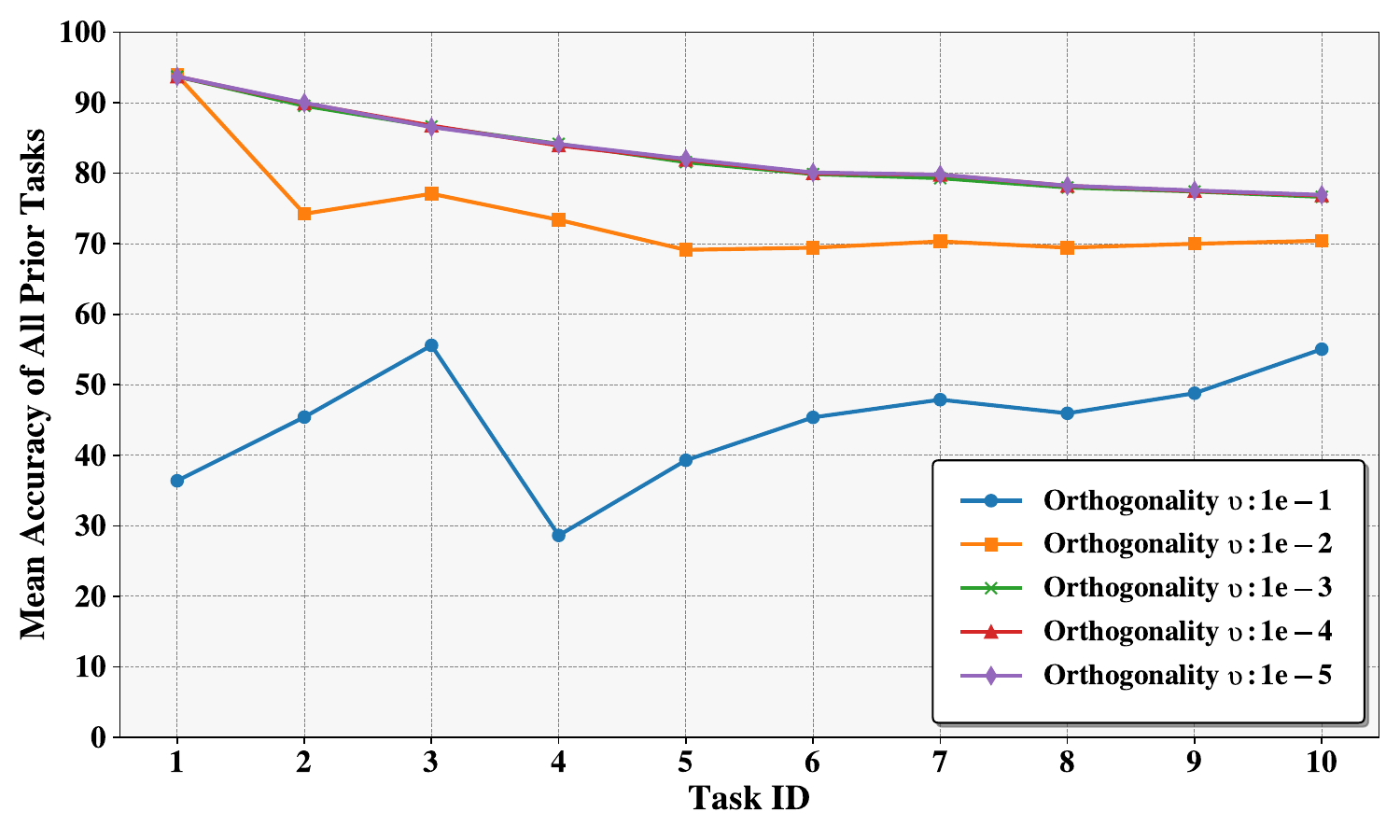}}
    \hfil
    \subfloat[ImageNet-R-$\lambda$ Tuning]{\includegraphics[width=0.24\linewidth]{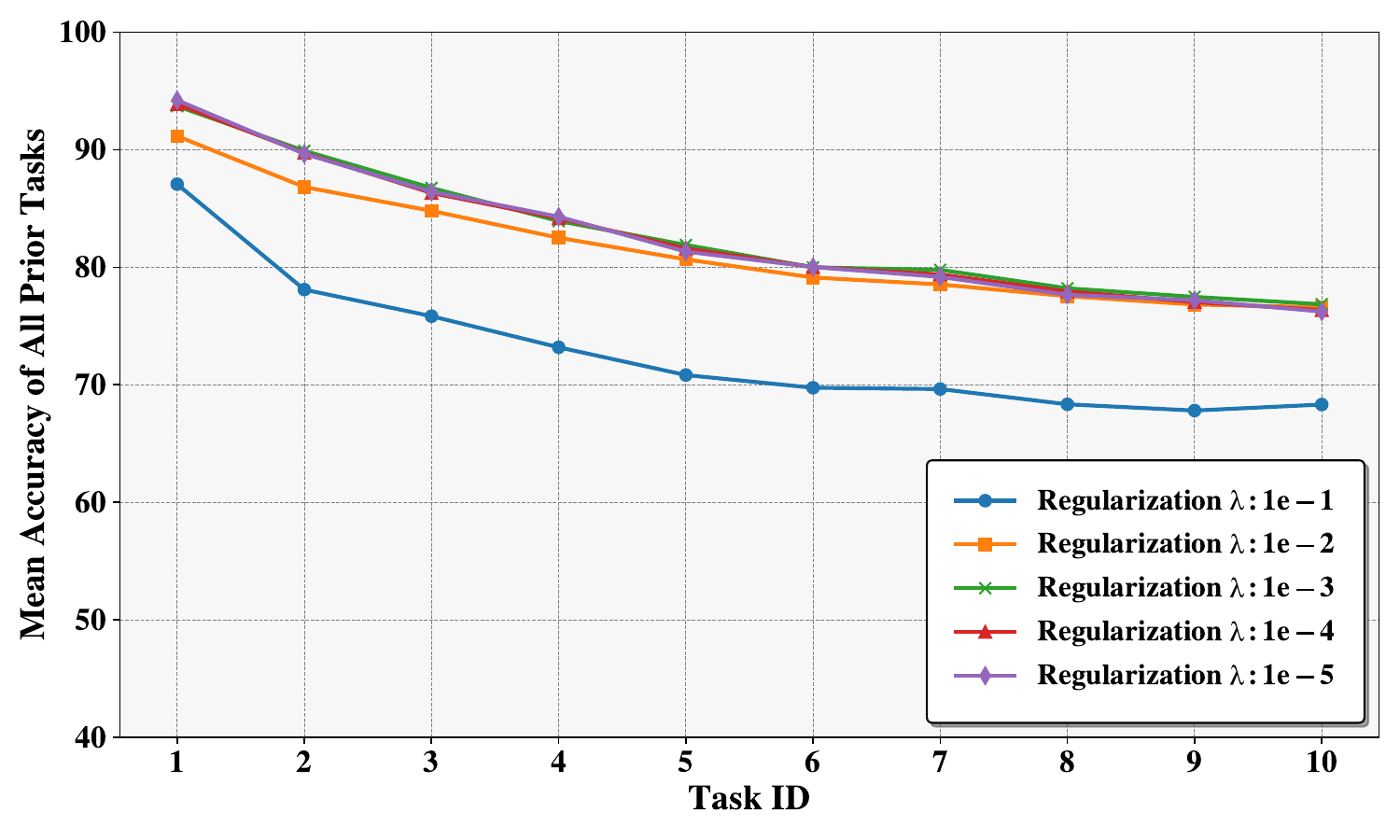}}
    \hfil
    \subfloat[ImageNet-R AHPS]{\includegraphics[width=0.238\linewidth]{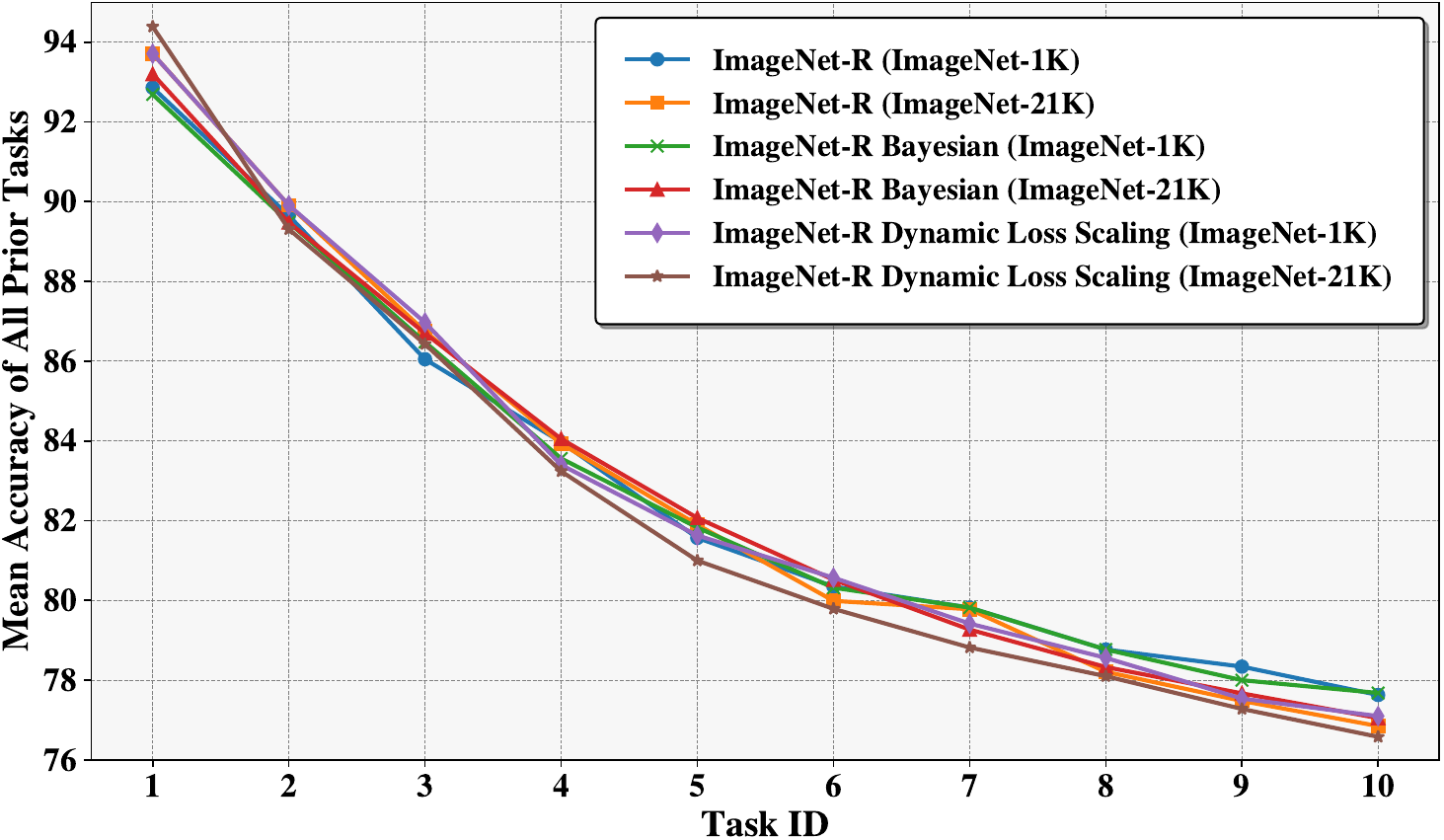}}
    \caption{Performance comparison of NTK-CL with different hyper-parameter settings and the proposed Automatic Hyper-Parameter Search (AHPS) strategies, based on Bayesian optimization or dynamic loss scaling, on CIFAR-100 and ImageNet-R.}
    \label{Hyperparameter_Tuning}
    \vspace{-2ex}
\end{figure*}

\subsection{Benchmark Comparison}
In this subsection, we evaluate the NTK-CL method against other leading methods. To ensure a fair performance comparison, we fix random seeds from 0 to 4, ensuring consistent task segmentation for each run \footnote{In Appendix~\ref{Task_Segmentation}, detailed procedures for modifying class order and the class orders for primary datasets are provided, enabling researchers to accurately replicate our task segmentation process and evaluate the impact of different class orders on model performance.}. We utilize uniformly sourced pre-trained weights and maintain the optimal hyper-parameters from the open-source code without modifications. Performance metrics for major datasets using ImageNet-21K pre-trained weights and ImageNet-1K fine-tuned weights are presented in Tables \ref{tab:accuracy_1} and \ref{tab:accuracy_2}.

For primary datasets such as CIFAR100, ImageNet-R, and ImageNet-A, we assess our method against most contemporary methods, excluding DAP \cite{jung2023generating} due to its flawed testing process, Hide-Prompt \cite{wang2024hierarchical} which compresses and samples past data, and Dual-PGP \cite{qiao2023prompt} which requires specific instance counts. By controlling for confounding factors, our method consistently achieves state-of-the-art performance. The NTK-CL method exhibits a clear advantage in both incremental accuracy (\(\bar{A}\)) and final accuracy (\(A_T\)), with improvements ranging from 1\% to 7\% compared to methods such as EASE \cite{zhou2024expandable} and EvoPrompt \cite{kurniawan2024evolving}. This advantage is particularly significant on ImageNet-A, a dataset known for challenging traditional models. Our NTK-CL framework substantially enhances model generalization and demonstrates robustness in complex visual recognition tasks.

Additionally, performance on auxiliary datasets including DomainNet, Oxford Pets, EuroSAT, PlantVillage, VTAB, and Kvasir, as detailed in Tables \ref{tab:accuracy_extended_21k} and \ref{tab:accuracy_extended_1k}, highlights the generalization and adaptability of NTK-CL across diverse domains. On these datasets, NTK-CL not only consistently delivers superior accuracy metrics but also exhibits reduced standard deviation in performance, emphasizing its stability. Notably, on Oxford Pets, NTK-CL achieves incremental accuracy improvements ranging from 1.8\% to 2.1\% and final accuracy enhancements of up to 4.6\% compared to EASE \cite{zhou2024expandable}. On the Kvasir dataset, NTK-CL outperforms competing methods, achieving the highest incremental accuracy improvements ranging from 6.7\% to 9.0\% and the highest final accuracy improvements ranging from 19.3\% to 21.1\%, showcasing its significant potential for medical applications. Across other datasets, NTK-CL consistently ranks as the best or the second-best, further affirming the method’s efficacy and versatility. 

To underscore the versatility of our NTK-CL framework, we provide a detailed examination of its performance in few-shot and imbalanced settings within Appendix \ref{other_settings}. Our results unequivocally illustrate that the framework sustains high performance levels under these conditions, validating the effectiveness of generalization principles.

\subsection{Ablation Study}
To systematically evaluate the contribution of each component in the proposed framework, we conduct extensive ablation studies on CIFAR100 and ImageNet-R using a frozen \textit{ViT-B/16-IN21K} backbone. For fair comparison, all experiments are conducted under identical settings, including fixed random seeds (0-4), consistent task partitions, and the same pre-trained backbone initialization. The ablation studies examine configurations involving the Subnetwork-1 Adaptation Module (S1), Subnetwork-2 Adaptation Module (S2), Hybrid Adaptation Module (Hybrid) \footnote{``Hybrid Only'' refers to the training objective rather than the network structure. In this setting, both Subnetwork-1 (S1) and Subnetwork-2 (S2) are retained in the model and remain fully trainable, producing two adapted feature streams. The Hybrid Adaptation Module performs cross-attention based fusion between these two streams to obtain $E_{CLS}^{HAE}$, rather than operating on a single backbone output with $Q=K=V$. The term ``Hybrid Only'' indicates that the model is supervised solely using $CE(E_{CLS}^{HAE}, y)$ in Eq.~\ref{A_eq7}, while the
branch-specific auxiliary losses on S1 and S2 are disabled.}, Knowledge Retention (KR) mechanism, Task-Feature Dissimilarity Loss (Dis), Orthogonality Loss (Orth), and Regularization Loss (Reg). The average accuracy (\(\bar{A}\)) across tasks is evaluated to quantitatively assess the contributions of each component.

The results in Table~\ref{ablation_study} indicate that different adaptation strategies consistently improve performance over the frozen backbone. Among them, the Hybrid adaptation module achieves higher average accuracy than the standalone Subnetwork-1 (S1) and Subnetwork-2 (S2) configurations on both datasets, suggesting that jointly leveraging multiple adaptation pathways is more effective for continual representation learning.
The introduction of the Knowledge Retention (KR) mechanism further improves performance across all adaptation configurations. In particular, for the S2-based setting, incorporating KR leads to a notable increase in average accuracy on both CIFAR100 and ImageNet-R, while for the S1 and Hybrid configurations, KR provides consistent performance gains with comparable variability across runs. These results indicate that KR effectively enhances knowledge preservation under different adaptation strategies.
We further analyze the effect of auxiliary constraints, including the task-feature dissimilarity loss, orthogonality loss, and regularization loss. When combined with the adaptation modules and KR, these components yield additional improvements in average accuracy on both datasets. The performance variations reported in Table~\ref{ablation_study} remain within a comparable range across different configurations, indicating that the observed gains are stable and reproducible.
Finally, integrating all adaptation modules, the KR mechanism, and auxiliary losses achieves the best overall performance on both CIFAR100 and ImageNet-R. These results demonstrate that each component contributes to the final performance, with the Hybrid adaptation and Knowledge Retention mechanisms playing a central role in the proposed parameter-efficient continual learning framework.

\begin{table*}[h]
    \centering
    \caption{Comparison of feature size expansion methods on CIFAR100 and their impact on the evolution of incremental top-1 accuracy (\%). The bold segments denote optimal results, and the underlined segments indicate suboptimal outcomes.}
    \setlength{\tabcolsep}{3pt}
    \renewcommand{\arraystretch}{1.05}
    \scriptsize
    \resizebox{\textwidth}{!}{%
    \begin{tabular}{@{} c *{10}{c} @{}}
        \toprule
        \multirow{2}{*}{\textbf{Combinations Type}} & \multicolumn{10}{c}{\textbf{Incremental Top-1 Accuracy (\%)}} \\
        \cmidrule(l){2-11}
        & \textbf{Task 1} & \textbf{Task 2} & \textbf{Task 3} & \textbf{Task 4} & \textbf{Task 5} & \textbf{Task 6} & \textbf{Task 7} & \textbf{Task 8} & \textbf{Task 9} & \textbf{Task 10} \\
        \midrule
        Original + Mixed-up \cite{zhang2018mixup} & 98.50 $\pm$ 0.98 & 94.70 $\pm$ 1.64 & 93.21 $\pm$ 1.15 & 91.99 $\pm$ 0.95 & 90.19 $\pm$ 2.06 & 90.07 $\pm$ 1.30 & 88.52 $\pm$ 0.97 & 87.81 $\pm$ 0.33 & 86.92 $\pm$ 0.74 & 86.91 $\pm$ 0.62 \\
        Original + PuzzleMix \cite{kim2020puzzle} & 98.44 $\pm$ 0.90 & 94.98 $\pm$ 1.42 & 93.53 $\pm$ 0.77 & 93.26 $\pm$ 0.47 & 90.26 $\pm$ 1.74 & 90.04 $\pm$ 1.11 & 88.65 $\pm$ 0.92 & 88.41 $\pm$ 0.65 & 86.72 $\pm$ 0.79 & 86.60 $\pm$ 0.77 \\
        Original + AutoAug \cite{cubuk2020randaugment} & 98.50 $\pm$ 0.85 & 95.14 $\pm$ 1.47 & 93.54 $\pm$ 1.27 & 92.39 $\pm$ 0.96 & 91.94 $\pm$ 1.96 & 91.07 $\pm$ 1.27 & 89.39 $\pm$ 1.14 & 88.82 $\pm$ 0.85 & 87.98 $\pm$ 0.96 & 87.06 $\pm$ 0.52 \\
        Original + RandAug \cite{cubuk2020randaugment} & \textbf{98.54 $\pm$ 0.91} & 95.65 $\pm$ 1.87 & 93.38 $\pm$ 1.17 & 92.81 $\pm$ 0.93 & 90.17 $\pm$ 2.19 & 90.05 $\pm$ 1.28 & 89.54 $\pm$ 0.88 & 88.67 $\pm$ 0.53 & 87.94 $\pm$ 0.67 & 87.89 $\pm$ 0.63 \\
        \midrule
        IA$^3$ \cite{liu2022few} + S1 & 96.72 $\pm$ 1.31 & 65.77 $\pm$ 3.20 & 64.20 $\pm$ 4.63 & 60.83 $\pm$ 2.95 & 59.70 $\pm$ 3.88 & 59.95 $\pm$ 5.11 & 58.74 $\pm$ 4.51 & 58.09 $\pm$ 2.85 & 57.20 $\pm$ 2.64 & 55.38 $\pm$ 3.15 \\
        IA$^3$ \cite{liu2022few} + S2 & 98.48 $\pm$ 0.85 & 97.05 $\pm$ 0.79 & 95.26 $\pm$ 0.77 & 94.22 $\pm$ 0.40 & 93.02 $\pm$ 1.13 & 91.61 $\pm$ 0.98 & 91.21 $\pm$ 0.48 & 90.46 $\pm$ 0.80 & 89.41 $\pm$ 0.50 & 88.96 $\pm$ 0.56 \\
        Compacter \cite{Mahabadi2021ParameterefficientMF} + S1 & 98.24 $\pm$ 1.16 & 96.55 $\pm$ 0.99 & 94.69 $\pm$ 0.93 & 93.80 $\pm$ 0.59 & 92.79 $\pm$ 0.86 & 92.32 $\pm$ 0.94 & 91.06 $\pm$ 0.65 & 90.36 $\pm$ 0.54 & 89.65 $\pm$ 0.58 & 89.18 $\pm$ 0.48 \\
        Compacter \cite{Mahabadi2021ParameterefficientMF} + S2 & \textbf{98.54 $\pm$ 0.89} & \underline{97.35 $\pm$ 0.65} & 95.65 $\pm$ 0.70 & \underline{94.57 $\pm$ 0.37} & \underline{93.61 $\pm$ 0.87} & 92.17 $\pm$ 0.99 & \underline{91.92 $\pm$ 0.44} & \underline{91.29 $\pm$ 0.38} & \underline{90.50 $\pm$ 0.22} & \underline{90.10 $\pm$ 0.18} \\
        Side-Tuning \cite{zhang2020side} + S1 & 96.06 $\pm$ 1.30 & 93.45 $\pm$ 1.41 & 90.47 $\pm$ 1.90 & 87.93 $\pm$ 1.62 & 86.49 $\pm$ 2.22 & 85.95 $\pm$ 1.01 & 84.52 $\pm$ 1.80 & 83.89 $\pm$ 1.76 & 82.98 $\pm$ 1.69 & 82.58 $\pm$ 0.81 \\
        Side-Tuning \cite{zhang2020side} + S2 & 98.48 $\pm$ 0.99 & \textbf{97.42 $\pm$ 0.90} & \textbf{95.71 $\pm$ 0.67} & 94.20 $\pm$ 0.31 & 93.52 $\pm$ 0.75 & 92.23 $\pm$ 0.93 & 91.03 $\pm$ 0.28 & 90.42 $\pm$ 0.34 & 89.58 $\pm$ 0.39 & 89.22 $\pm$ 0.19 \\
        Side-Tuning \cite{zhang2020side} + Compacter \cite{Mahabadi2021ParameterefficientMF} & 98.38 $\pm$ 1.07 & 96.84 $\pm$ 0.81 & 94.87 $\pm$ 0.87 & 93.87 $\pm$ 0.98 & 92.79 $\pm$ 1.42 & \underline{92.41 $\pm$ 0.72} & 91.31 $\pm$ 0.59 & 90.66 $\pm$ 0.33 & 89.80 $\pm$ 0.47 & 89.41 $\pm$ 0.31 \\
        \rowcolor{blue!14}
        S1 + S2 (Ours) & \underline{98.52 $\pm$ 0.98} & 97.33 $\pm$ 0.71 & \underline{95.69 $\pm$ 0.76} & \textbf{94.72 $\pm$ 0.28} & \textbf{93.65 $\pm$ 0.88} & \textbf{93.25 $\pm$ 0.97} & \textbf{91.99 $\pm$ 0.65} & \textbf{91.44 $\pm$ 0.40} & \textbf{90.71 $\pm$ 0.24} & \textbf{90.27 $\pm$ 0.20} \\
        \bottomrule
    \end{tabular}}
    \label{aug_method_comparison}
\end{table*}

\begin{table*}[h]
    \centering
    \caption{Comparison of feature fusion methods on CIFAR100 and their impact on incremental top-1 accuracy (\%). \(Q\), \(K\), and \(V\) are the query, key, and value in our feature fusion method. The structures of the various fusion methods are illustrated in Appendix~\ref{architectures}. The bold segments denote optimal results, and the underlined segments indicate suboptimal outcomes.}
        \setlength{\tabcolsep}{3pt}
    \renewcommand{\arraystretch}{1.05}
    \scriptsize
    \resizebox{\textwidth}{!}{%
    \begin{tabular}{@{} c *{10}{c} @{}}
        \toprule
        \multirow{2}{*}{\textbf{Fusion Type}} & \multicolumn{10}{c}{\textbf{Incremental Top-1 Accuracy (\%)}} \\
        \cmidrule(l){2-11}
        & \textbf{Task 1} & \textbf{Task 2} & \textbf{Task 3} & \textbf{Task 4} & \textbf{Task 5} & \textbf{Task 6} & \textbf{Task 7} & \textbf{Task 8} & \textbf{Task 9} & \textbf{Task 10} \\
        \midrule
        MLP & \underline{98.32 $\pm$ 1.20} & 96.83 $\pm$ 1.05 & 95.13 $\pm$ 0.70 & 93.63 $\pm$ 0.73 & \underline{92.61 $\pm$ 1.49} & 92.03 $\pm$ 0.82 & 90.31 $\pm$ 0.97 & 89.46 $\pm$ 0.40 & 88.46 $\pm$ 0.62 & 88.25 $\pm$ 0.46 \\
        VAE & 98.16 $\pm$ 1.40 & 96.67 $\pm$ 0.84 & 94.43 $\pm$ 0.87 & 93.49 $\pm$ 1.74 & 92.18 $\pm$ 1.80 & 91.83 $\pm$ 0.88 & 90.38 $\pm$ 1.10 & 89.77 $\pm$ 0.79 & 88.92 $\pm$ 1.73 & 88.57 $\pm$ 0.46 \\
        RNN & 98.10 $\pm$ 1.38 & 96.54 $\pm$ 0.70 & 94.83 $\pm$ 0.42 & \underline{94.66 $\pm$ 0.69} & 92.33 $\pm$ 1.70 & \underline{92.05 $\pm$ 0.92} & \underline{90.75 $\pm$ 0.78} & \underline{90.08 $\pm$ 0.38} & 89.31 $\pm$ 0.74 & 88.96 $\pm$ 0.42 \\
        Mamba & 98.00 $\pm$ 1.14 & \underline{96.86 $\pm$ 0.90} & \underline{95.30 $\pm$ 0.99} & 94.30 $\pm$ 0.57 & 91.80 $\pm$ 2.48 & 91.16 $\pm$ 1.87 & 89.83 $\pm$ 1.64 & 89.04 $\pm$ 1.39 & 89.17 $\pm$ 1.04 & 87.81 $\pm$ 1.39 \\
        \rowcolor{blue!14}
        S1($K/V$) + S2 ($Q$) & 97.72 $\pm$ 1.43 & 96.14 $\pm$ 1.26 & 94.31 $\pm$ 1.33 & 93.25 $\pm$ 1.90 & 92.51 $\pm$ 2.97 & 90.57 $\pm$ 1.78 & 90.46 $\pm$ 1.40 & 90.02 $\pm$ 1.00 & \underline{89.51 $\pm$ 1.10} & \underline{89.08 $\pm$ 1.19} \\
        \rowcolor{blue!14}
        S1($Q$) + S2 ($K/V$) & \textbf{98.52 $\pm$ 0.98} & \textbf{97.33 $\pm$ 0.71} & \textbf{95.69 $\pm$ 0.76} & \textbf{94.72 $\pm$ 0.28} & \textbf{93.65 $\pm$ 0.88} & \textbf{93.25 $\pm$ 0.97} & \textbf{91.99 $\pm$ 0.65} & \textbf{91.44 $\pm$ 0.40} & \textbf{90.71 $\pm$ 0.24} & \textbf{90.27 $\pm$ 0.20} \\
        \bottomrule
    \end{tabular}}
    \label{fusion_method_comparison}
\end{table*}

\begin{table*}[h]
    \centering
    \caption{Comparison of knowledge inheritance methods on CIFAR100 and their impact on incremental top-1 accuracy (\%). The bold segments denote optimal results, and the underlined segments indicate suboptimal outcomes.}
    \setlength{\tabcolsep}{3pt}
    \renewcommand{\arraystretch}{1.05}
    \scriptsize
    \resizebox{\textwidth}{!}{%
    \begin{tabular}{@{} c *{10}{c} @{}}
        \toprule
        \multirow{2}{*}{\textbf{Knowledge Inheritance Type}} & \multicolumn{10}{c}{\textbf{Incremental Top-1 Accuracy (\%)}} \\
        \cmidrule(l){2-11}
        & \textbf{Task 1} & \textbf{Task 2} & \textbf{Task 3} & \textbf{Task 4} & \textbf{Task 5} & \textbf{Task 6} & \textbf{Task 7} & \textbf{Task 8} & \textbf{Task 9} & \textbf{Task 10} \\
        \midrule
        MoCo-v3 \cite{chen2021empirical} & 98.08 $\pm$ 1.52 & 96.12 $\pm$ 1.34 & 92.93 $\pm$ 1.69 & 91.11 $\pm$ 1.96 & 89.62 $\pm$ 3.05 & 88.83 $\pm$ 1.62 & 88.15 $\pm$ 1.89 & 87.33 $\pm$ 1.49 & 86.82 $\pm$ 1.17 & 86.70 $\pm$ 1.38 \\
        LAE \cite{gao2023unified} & 98.00 $\pm$ 1.49 & 94.91 $\pm$ 1.35 & 92.48 $\pm$ 1.64 & 89.93 $\pm$ 2.52 & 87.29 $\pm$ 3.62 & 86.48 $\pm$ 3.51 & 85.21 $\pm$ 3.07 & 84.70 $\pm$ 2.57 & 83.94 $\pm$ 2.30 & 83.56 $\pm$ 2.61 \\
        EASE \cite{zhou2024expandable} & \underline{98.24 $\pm$ 1.18} & \underline{96.53 $\pm$ 1.08} & \underline{94.24 $\pm$ 2.80} & \underline{92.19 $\pm$ 0.75} & \underline{90.95 $\pm$ 1.89} & \underline{90.29 $\pm$ 1.04} & \underline{89.69 $\pm$ 1.57} & \underline{87.79 $\pm$ 2.87} & \underline{87.19 $\pm$ 2.30} & \underline{87.14 $\pm$ 1.74} \\
        \rowcolor{blue!14}
        Adaptive EMA (Ours) & \textbf{98.52 $\pm$ 0.98} & \textbf{97.33 $\pm$ 0.71} & \textbf{95.69 $\pm$ 0.76} & \textbf{94.72 $\pm$ 0.28} & \textbf{93.65 $\pm$ 0.88} & \textbf{93.25 $\pm$ 0.97} & \textbf{91.99 $\pm$ 0.65} & \textbf{91.44 $\pm$ 0.40} & \textbf{90.71 $\pm$ 0.24} & \textbf{90.27 $\pm$ 0.20} \\
        \bottomrule
    \end{tabular}}
    \label{knowledge_inheritance_method_comparison}
\end{table*}

\begin{table*}[h]
    \centering
    \caption{Comparison of orthogonality losses on CIFAR100 and their impact on incremental top-1 accuracy (\%). The bold segments denote optimal results, and the underlined segments indicate suboptimal outcomes.}
    \setlength{\tabcolsep}{3pt}
    \renewcommand{\arraystretch}{1.05}
    \scriptsize
    \resizebox{\textwidth}{!}{%
    \begin{tabular}{@{} c *{11}{c} @{}}
        \toprule
        \multicolumn{2}{c}{\textbf{Orthogonality Loss Type}} & \multicolumn{10}{c}{\textbf{Incremental Top-1 Accuracy (\%)}} \\
        \cmidrule(r){1-2} \cmidrule(l){3-12}
        \textbf{Name} & \textbf{Level} & \textbf{Task 1} & \textbf{Task 2} & \textbf{Task 3} & \textbf{Task 4} & \textbf{Task 5} & \textbf{Task 6} & \textbf{Task 7} & \textbf{Task 8} & \textbf{Task 9} & \textbf{Task 10} \\
        \midrule
        $\mathcal{L}_{\text{Standard}}$ & Class-level & \underline{98.46 $\pm$ 0.99} & \textbf{97.33 $\pm$ 0.73} & \textbf{95.70 $\pm$ 0.78} & \underline{94.53 $\pm$ 0.47} & \underline{93.47 $\pm$ 0.97} & 92.10 $\pm$ 0.75 & \underline{91.71 $\pm$ 0.54} & \underline{91.16 $\pm$ 0.33} & \underline{90.30 $\pm$ 0.21} & \underline{89.90 $\pm$ 0.21} \\
        $\mathcal{L}_{\text{Schmidt}}$ & Class-level & \underline{98.46 $\pm$ 0.91} & \underline{97.30 $\pm$ 0.88} & 95.56 $\pm$ 0.77 & 94.20 $\pm$ 0.62 & 92.89 $\pm$ 1.29 & \underline{92.61 $\pm$ 0.99} & 91.02 $\pm$ 0.81 & 90.22 $\pm$ 0.60 & 89.26 $\pm$ 0.52 & 88.76 $\pm$ 0.51 \\
        \rowcolor{blue!14}
        $\mathcal{L}_{\text{orth}}$ (Ours) & Task-level & \textbf{98.52 $\pm$ 0.98} & \textbf{97.33 $\pm$ 0.71} & \underline{95.69 $\pm$ 0.76} & \textbf{94.72 $\pm$ 0.28} & \textbf{93.65 $\pm$ 0.88} & \textbf{93.25 $\pm$ 0.97} & \textbf{91.99 $\pm$ 0.65} & \textbf{91.44 $\pm$ 0.40} & \textbf{90.71 $\pm$ 0.24} & \textbf{90.27 $\pm$ 0.20} \\
        \bottomrule
    \end{tabular}}
    \label{orthogonalization_loss_comparison}
    \vspace{-2ex}
\end{table*}

\begin{table*}[h]
    \centering
    \caption{Comparison of regularization methods on CIFAR100 and their impact on incremental top-1 accuracy (\%). The bold segments denote optimal results, and the underlined segments indicate suboptimal outcomes.}
    \setlength{\tabcolsep}{3pt}
    \renewcommand{\arraystretch}{1.05}
    \scriptsize
    \resizebox{\textwidth}{!}{%
    \begin{tabular}{@{} c *{10}{c} @{}}
        \toprule
        \multirow{2}{*}{\textbf{Regularization Type}} & \multicolumn{10}{c}{\textbf{Incremental Top-1 Accuracy (\%)}} \\
        \cmidrule(l){2-11}
        & \textbf{Task 1} & \textbf{Task 2} & \textbf{Task 3} & \textbf{Task 4} & \textbf{Task 5} & \textbf{Task 6} & \textbf{Task 7} & \textbf{Task 8} & \textbf{Task 9} & \textbf{Task 10} \\
        \midrule
        L1 & 97.34 $\pm$ 1.44 & 95.98 $\pm$ 1.10 & 94.41 $\pm$ 0.59 & 93.55 $\pm$ 0.97 & 92.40 $\pm$ 1.45 & 91.80 $\pm$ 0.60 & 90.60 $\pm$ 0.69 & 89.87 $\pm$ 0.55 & 89.16 $\pm$ 0.73 & 88.67 $\pm$ 0.25 \\
        Spectral & \underline{98.54} $\pm$ 0.89 & \underline{97.27 $\pm$ 0.75} & \textbf{95.75 $\pm$ 0.82} & \underline{94.67 $\pm$ 0.29} & \underline{93.56 $\pm$ 0.90} & \underline{93.16 $\pm$ 1.80} & \underline{91.90 $\pm$ 0.61} & \underline{91.27 $\pm$ 0.32} & \underline{90.53 $\pm$ 0.16} & 89.14 $\pm$ 0.15 \\
        HiDe-Prompt \cite{wang2024hierarchical} & \textbf{98.56 $\pm$ 0.90} & 97.12 $\pm$ 0.63 & 95.62 $\pm$ 0.74 & 94.56 $\pm$ 0.38 & 93.44 $\pm$ 0.84 & 92.58 $\pm$ 0.76 & 91.82 $\pm$ 0.55 & 91.23 $\pm$ 0.45 & 90.50 $\pm$ 0.35 & \underline{90.04 $\pm$ 0.27} \\
        \rowcolor{blue!14}
        L2 (Ours) & 98.52 $\pm$ 0.98 & \textbf{97.33 $\pm$ 0.71} & \underline{95.69 $\pm$ 0.76} & \textbf{94.72 $\pm$ 0.28} & \textbf{93.65 $\pm$ 0.88} & \textbf{93.25 $\pm$ 0.97} & \textbf{91.99 $\pm$ 0.65} & \textbf{91.44 $\pm$ 0.40} & \textbf{90.71 $\pm$ 0.24} & \textbf{90.27 $\pm$ 0.20} \\
        \bottomrule
    \end{tabular}}
    \label{regularization_method_comparison}
\end{table*}

\subsection{Hyper-parameters Adjustment}
In our experimental setup, we systematically vary the hyper-parameters \(\eta\), \(\upsilon\), and \(\lambda\) to investigate their influence on the PEFT-CL performance, specifically the contributions of \(\mathcal{L}_{dis}\), \(\mathcal{L}_{orth}\), and \(\mathcal{L}_{reg}\). To ensure a fair comparison across different conditions, all experiments employ a fixed random seed of 0. Our ultimately adopted optimal hyper-parameters are dataset-specific, with ImageNet-R benefiting from \(\eta = 0.2\), \(\upsilon = 0.0001\), and \(\lambda = 0.001\), whereas CIFAR100 achieves best results with \(\eta = 0.03\), \(\upsilon = 0.0001\), and \(\lambda = 0.001\). In each experiment, we isolate the effect of a single hyper-parameter by holding the others constant. As depicted in Fig.~\ref{Hyperparameter_Tuning}, alterations in these hyper-parameters markedly affect the model's performance during continual learning. Maintaining orthogonality and regularization parameters near 0.0001 and 0.001, respectively, is essential for optimal performance. Deviating from these values can precipitate substantial declines in performance for both new and previously learned tasks, underscoring the delicate balance required between enforcing orthogonality among task features and applying parameter regularization to preserve classification accuracy.

Moreover, to address more complex real-world scenarios and mitigate the challenges associated with manual hyper-parameter specification, we propose two automatic hyper-parameter search (AHPS) methods: Bayesian Optimization strategy and Dynamic Loss Scaling strategy, as elaborated in Appendix~\ref{hyp_search}. The former offers a solid theoretical foundation but requires additional computational resources for validation-based search, whereas the latter operates without incurring extra computational overhead and eliminates the need for manual tuning of balancing coefficients. Empirical evaluations presented in Fig.~\ref{Hyperparameter_Tuning} demonstrate that both automated strategies achieve performance on par with labor-intensive manual tuning, thereby providing practical benefits and enhanced efficiency for real-world PEFT-CL applications.

\subsection{Alternative Experiments}
To rigorously evaluate and underscore the distinct advantages of the proposed components, a series of meticulously designed alternative experiments are conducted on the CIFAR100 dataset. Each experiment adheres to a stringent protocol, employing the fixed seeds of 0-4 to ensure reproducibility, while all experimental conditions and runtime environments are rigorously standardized to maintain consistency. A particular emphasis is placed on the incremental top-1 accuracy $A_\tau$ across tasks.

The alternative experiments of sample size expansion methods encompasses two primary paradigms. Firstly, image-level augmentation techniques are explored, with Mixed-up \cite{zhang2018mixup}, PuzzleMix \cite{kim2020puzzle}, AutoAug \cite{cubuk2018autoaugment}, and RandAug \cite{cubuk2020randaugment} serving as representative methods. In this context, only the S2 component is retained, given its demonstrated superiority over S1 module, leveraging it to facilitate PEFT and the feature fusion from paired images. Despite achieving an expansion of the sample size at the image level, the results presented in Table~\ref{aug_method_comparison} reveal that this paradigm’s continual performance remains inferior to the network-level feature size expansion achieved through PEFT combinations. 
The second paradigm involves the amalgamation of distinct PEFT techniques, specifically IA$^3$ \cite{liu2022few}, Compacter \cite{Mahabadi2021ParameterefficientMF}, and Side-Tuning \cite{zhang2020side}. As evidenced in Table~\ref{aug_method_comparison}, other PEFT combinations do not surpass the efficacy of ours. This disparity can be attributed to the fact that S1 module and S2 module extract information from the same image, albeit in the channel and spatial dimensions, respectively. This targeted extraction yields feature subspaces that are more discriminative and less redundant compared to those generated by other PEFT combinations that capture dataset-level biases, thereby contributing to superior performance.

In addition, we conduct systematic comparative studies on feature fusion strategies, knowledge inheritance mechanisms, orthogonality constraints, and regularization schemes under the proposed dual-stream architecture with a shared and frozen backbone. These experiments are designed to evaluate the effectiveness of individual components independently of the backbone execution flow, and all results are obtained under identical training and inference settings.
The comparative results are summarized in Table~\ref{fusion_method_comparison}, Table~\ref{knowledge_inheritance_method_comparison}, Table~\ref{orthogonalization_loss_comparison}, and Table~\ref{regularization_method_comparison}, respectively. Table~\ref{fusion_method_comparison} shows that the proposed MSA-based fusion strategy achieves superior continual performance compared to alternative fusion designs, which is consistent with the generalization analysis presented in \cref{MSA_Generalization}. 
Regarding knowledge inheritance, the results in Table~\ref{knowledge_inheritance_method_comparison} demonstrate the effectiveness of the proposed Adaptive EMA mechanism in preserving and integrating historical knowledge across tasks. Table~\ref{orthogonalization_loss_comparison} further indicates that, in the PEFT-CL setting, enforcing task-level orthogonality is sufficient, differing from conventional CL scenarios where class-level orthogonality is often required. Finally, Table~\ref{regularization_method_comparison} validates the adopted L2 regularization scheme in stabilizing the optimization dynamics of PEFT-CL, as characterized in Eq.~\ref{NTK_Dynamics_1}.
Collectively, these results confirm that the proposed design choices are well-suited to the dual-stream continual learning framework and contribute consistently to performance improvements.

\begin{figure*}[t]
    \centering
    \begin{tabular}{@{}c@{}c@{}c@{}c@{}c@{}c@{}c@{}c@{}}
        & \textbf{ViT-21K} & \textbf{\small S1 Task-0} & \textbf{\small S1 Task-9} & \textbf{\small S2 Task-0} & \textbf{\small S2 Task-9} & \textbf{\small Hybrid Task-0} & \textbf{\small Hybrid Task-9} \\
        \rotatebox{90}{\parbox{2.3cm}{\centering \textbf{ImageNet-R}}} & 
        \includegraphics[width=0.13\textwidth]{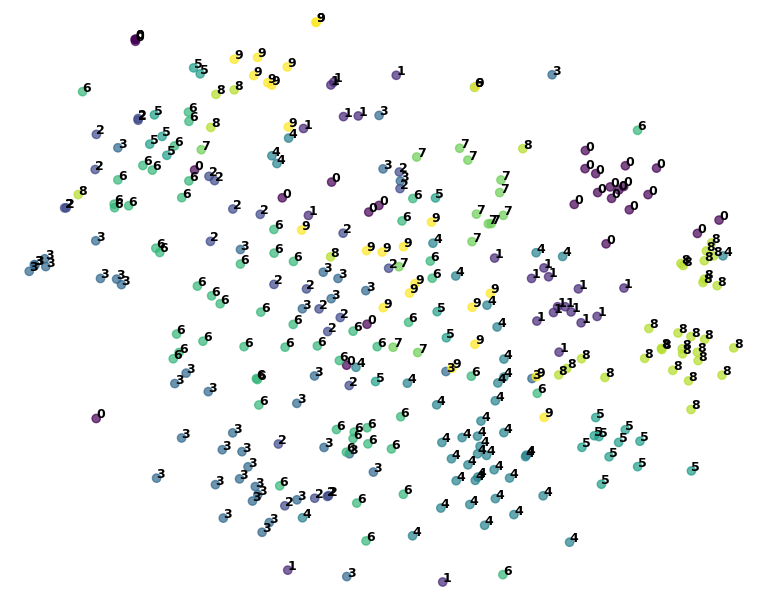} &
        \includegraphics[width=0.13\textwidth]{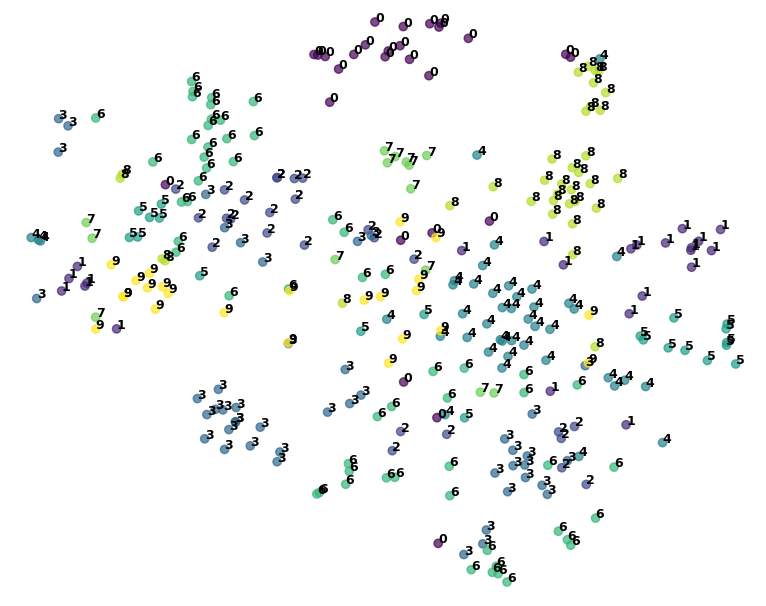} &
        \includegraphics[width=0.14\textwidth]{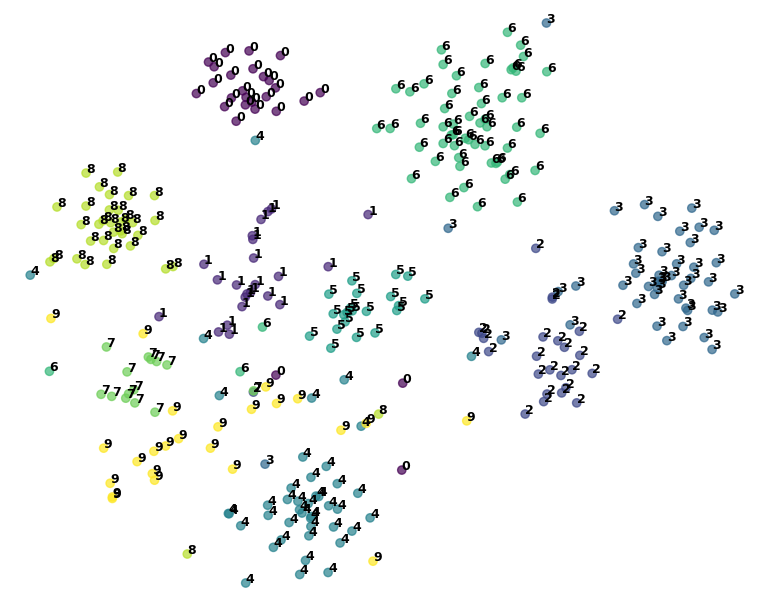} &
        \includegraphics[width=0.13\textwidth]{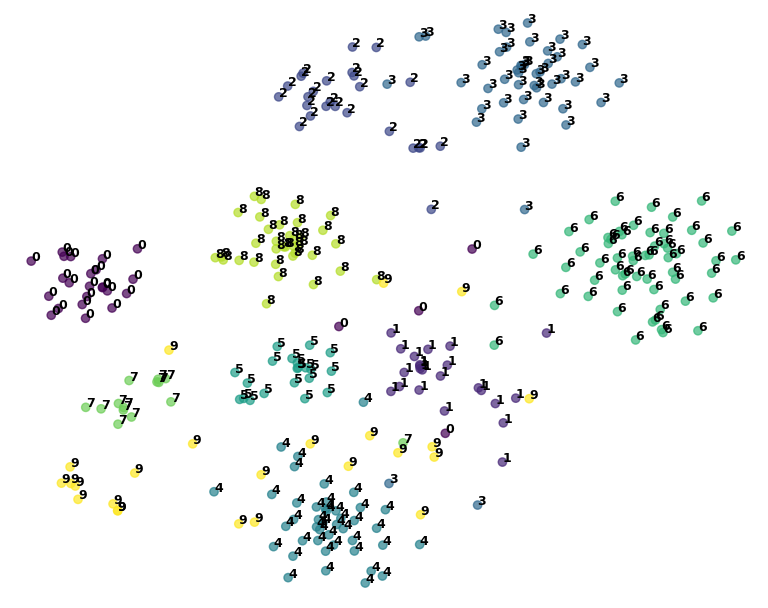} &
        \includegraphics[width=0.13\textwidth]{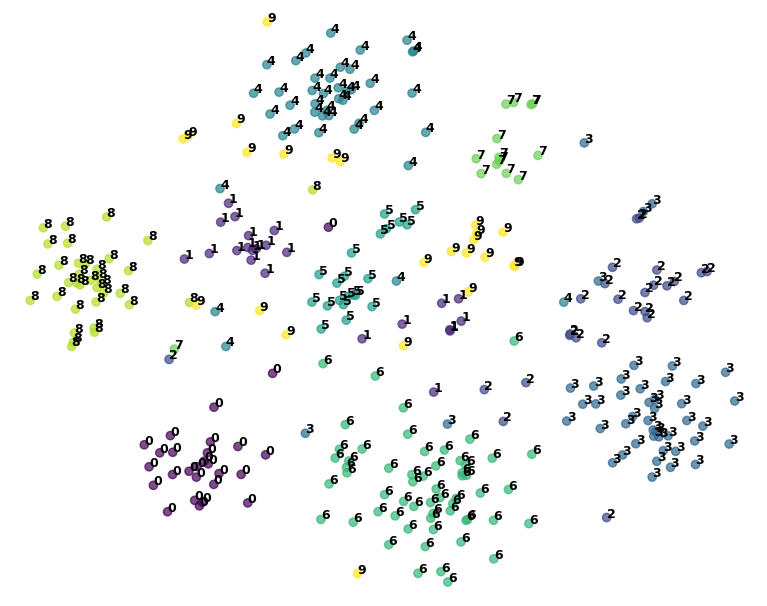} &
        \includegraphics[width=0.13\textwidth]{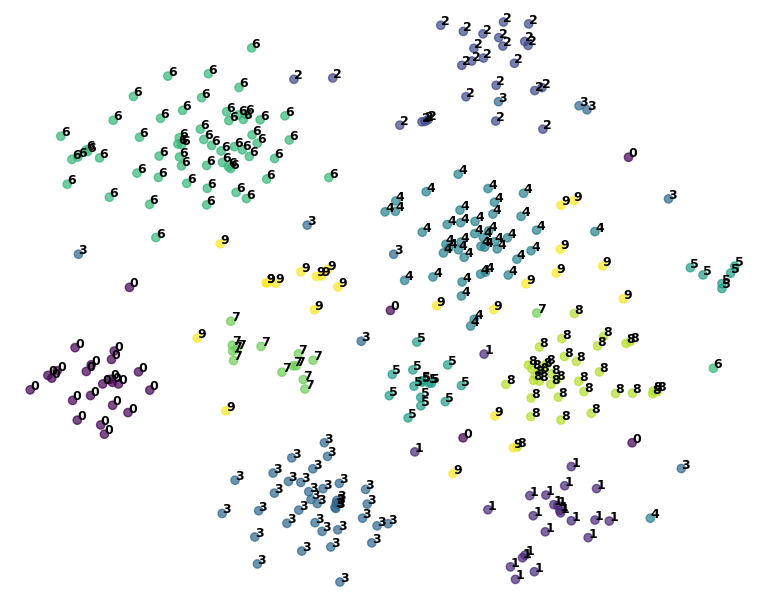} &
        \includegraphics[width=0.13\textwidth]{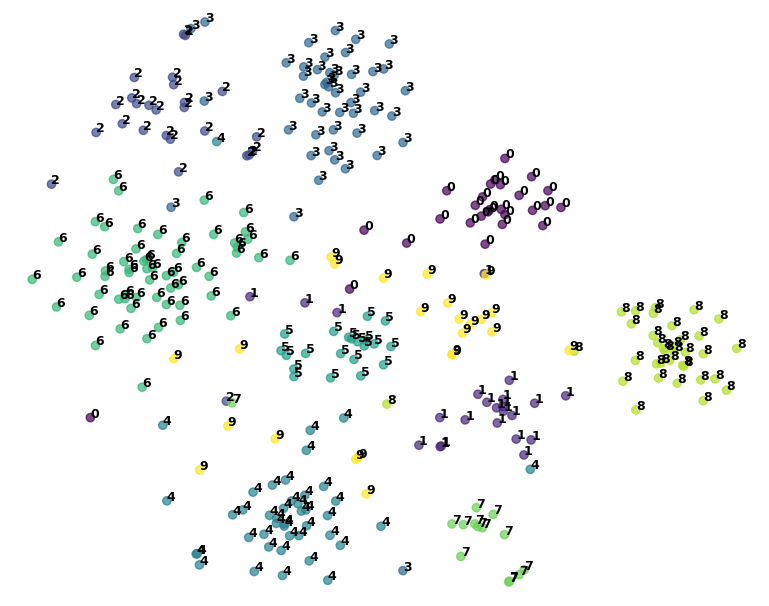} \\
        
        \rotatebox{90}{\parbox{2.3cm}{\centering \textbf{ImageNet-A}}} & 
        \includegraphics[width=0.13\textwidth]{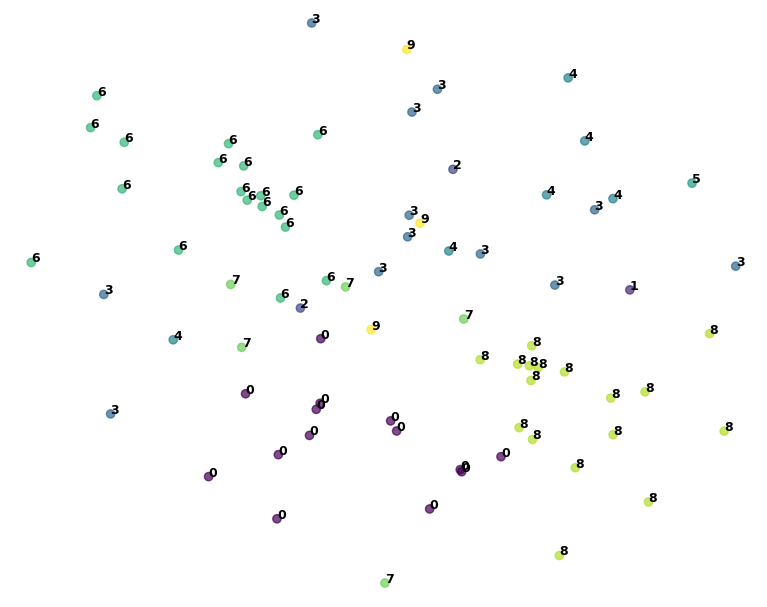} &
                \includegraphics[width=0.13\textwidth]{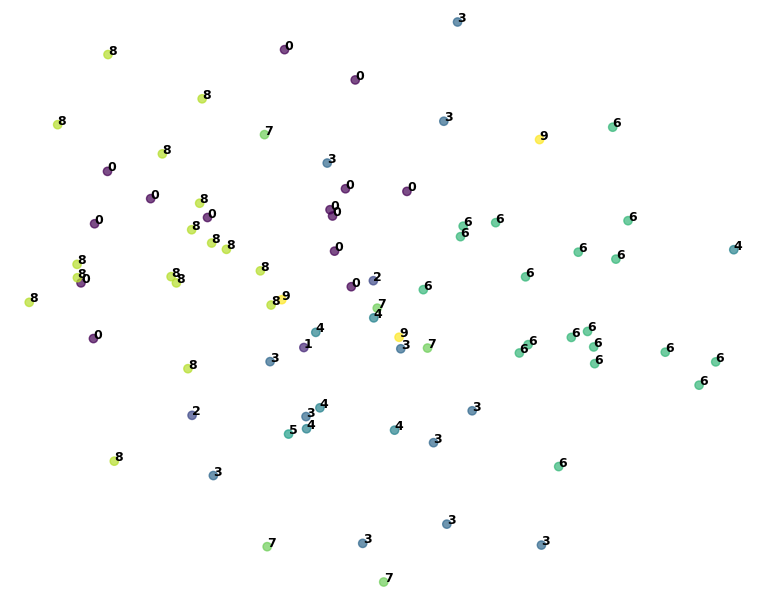} &
        \includegraphics[width=0.13\textwidth]{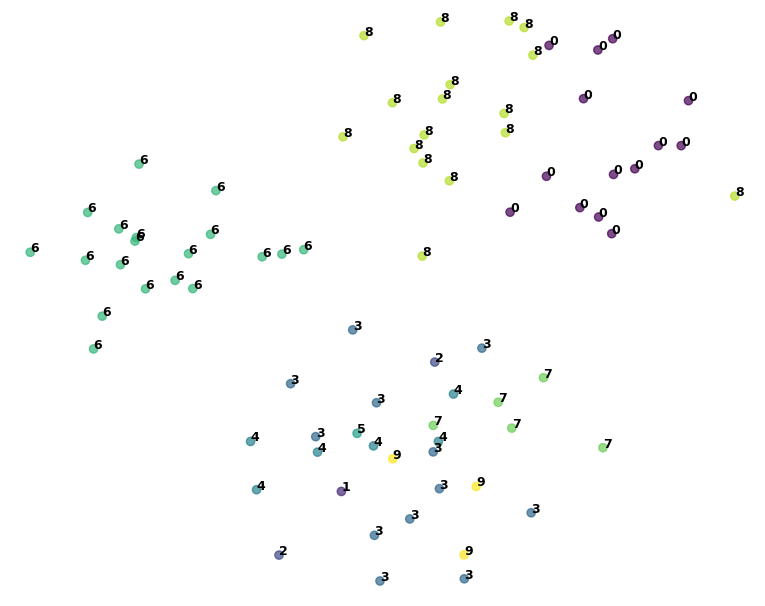} &
        \includegraphics[width=0.13\textwidth]{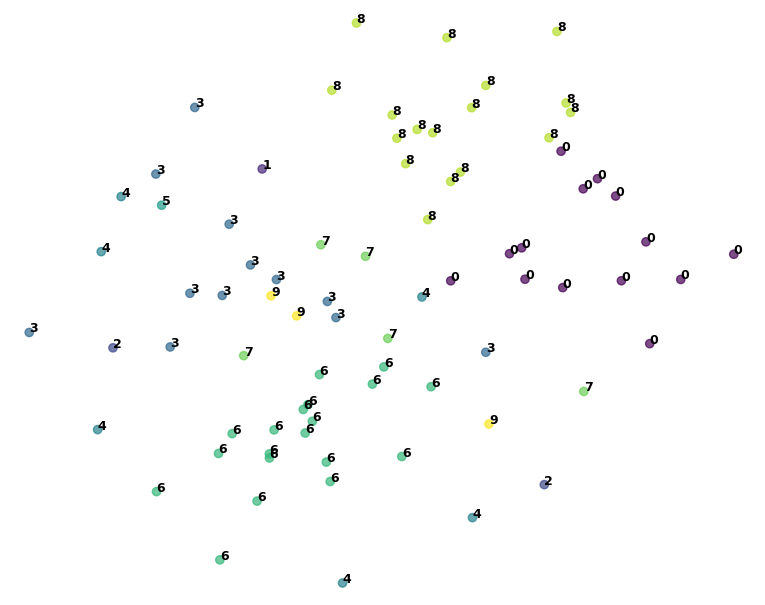} &
        \includegraphics[width=0.13\textwidth]{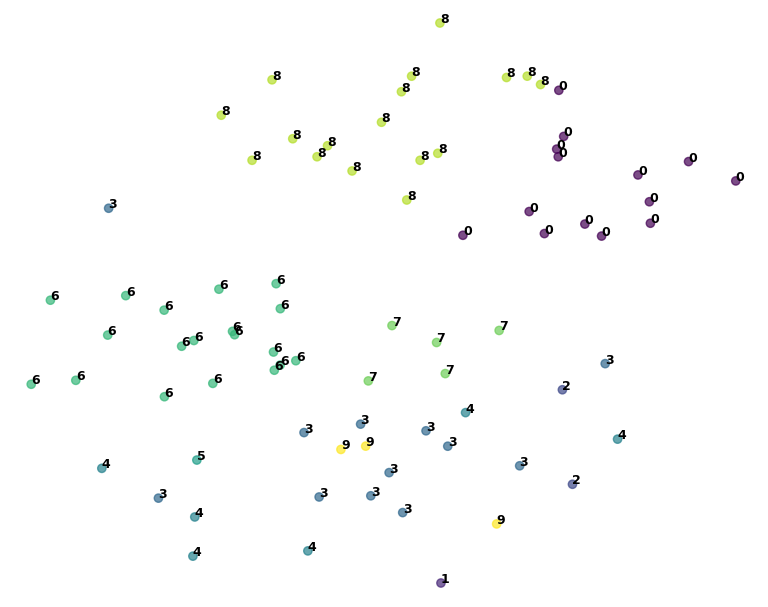} &
        \includegraphics[width=0.13\textwidth]{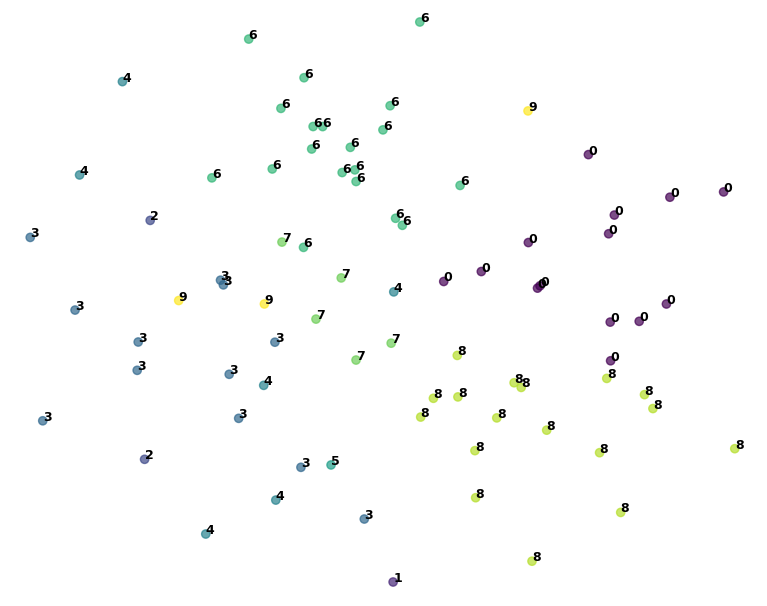} &
        \includegraphics[width=0.13\textwidth]{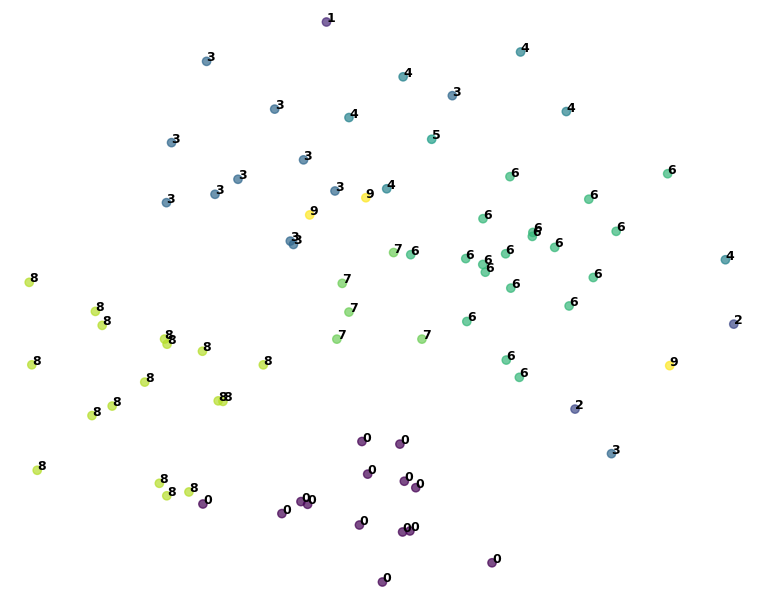} \\
    \end{tabular}
    \caption{The illustration of the t-SNE visualization for samples from Task 0 on the ImageNet-R and ImageNet-A datasets primarily focuses on the original ViT-21K pre-trained features. It also includes subnetwork-1 (S1) adaptation features, subnetwork-2 (S2) adaptation features, and hybrid adaptation features from weights at Task-0 and Task-9 stages, helping to elucidate the evolution and differentiation of feature representations across different stages.}
    \label{fig:tsne}
\end{figure*}

\subsection{Visualization}
To provide a more intuitive human visual assessment of the information captured by the pre-trained ViT and processed through the S1 and S2 modules, we demonstrate that this information resides in entirely distinct representational spaces. In Appendix~\ref{visualization}, we employ the Deep Image Prior (DIP) technique \cite{ulyanov2018deep} to reconstruct the image information at different task stages. In Fig.~\ref{fig:imagenetr_grid} and Fig.~\ref{fig:imageneta_grid}, we present a detailed visualization of the DIP results for images from Task-0 on ImageNet-R and ImageNet-A datasets, respectively. These visualizations reveal distinct differences in the information captured by the S1 and S2 modules. Specifically, the S2 module tends to focus more on the shapes and intrinsic features of the images, while the S1 module emphasizes color and fine details. This distinction underscores our design strategy of differentiating feature subspaces, thereby providing optimal input for the Hybrid Adaptation Module. Furthermore, the evolution from Task-0 to Task-9 within our NTK-CL framework demonstrates its capability to effectively retain knowledge from previous tasks. This confirms the efficacy of our Knowledge Retention innovation in maintaining consistent performance across tasks, even as new tasks are introduced. These visualizations not only clarify the model's behavior but also validate the effectiveness of our framework in the PEFT-CL scenario.

To further elucidate the evolution and advantages of S1 adaptation features, S2 adaptation features, and hybrid adaptation features during continual training, we conduct detailed t-SNE experiments. Utilizing the original ViT-21K pre-trained weights as a baseline, we compare the features of samples from Task 0 at both Task-0 and Task-9 stages. As illustrated in Fig. \ref{fig:tsne}, the S1 adaptation features, S2 adaptation features, and hybrid adaptation features all exhibit significantly enhanced discriminability compared to the features produced by the original ViT-21K pre-trained weights. Additionally, the discriminability of Task-0 samples is effectively maintained even at Task-9, which to a certain extent demonstrates the anti-forgetting capability of our framework. Notably, the hybrid adaptation features show superior discriminability relative to both S1 adaptation features and S2 adaptation features, affirming the effectiveness of the fusion component.

\begin{table*}[htbp]
\centering
\fontsize{8pt}{10pt}\selectfont
\caption{Performance analysis in NTK-CL for different pre-trained weights of \textit{ViT-B/16}. The bolded segments represent the optimal results, while the underlined segments represent suboptimal results.}
\label{tab:accuracy_weights}
\begin{tabular}{>{\centering\arraybackslash}p{4cm} cccccc}
\toprule
\multirow{2}{*}{\textbf{Method}} & \multicolumn{2}{c}{\textbf{CIFAR-100}} & \multicolumn{2}{c}{\textbf{ImageNet-R}} & \multicolumn{2}{c}{\textbf{ImageNet-A}} \\
\cmidrule(lr){2-3} \cmidrule(lr){4-5} \cmidrule(lr){6-7}
& $\bar{A}$ (\%) & $A_T$ (\%) & $\bar{A}$ (\%) & $A_T$ (\%) & $\bar{A}$ (\%) & $A_T$ (\%) \\
\midrule
\rowcolor{yellow!10}
\multicolumn{7}{c}{\textbf{Self-Supervised Methods}} \\
\rowcolor{yellow!10}
Dino ImageNet-1K \cite{caron2021emerging} & 84.85 $\pm$ 0.46 & 78.09 $\pm$ 0.44 & 74.08 $\pm$ 0.70 & 66.88 $\pm$ 0.35 & 45.03 $\pm$ 1.03 & 34.85 $\pm$ 0.82 \\
\rowcolor{yellow!10}
MAE ImageNet-1K \cite{he2022masked} & 48.29 $\pm$ 3.80 & 41.59 $\pm$ 2.05 & 40.49 $\pm$ 1.33 & 32.73 $\pm$ 1.40 & 8.63 $\pm$ 1.54 & 5.66 $\pm$ 1.83 \\
\rowcolor{yellow!10}
iBOT ImageNet-1K \cite{zhou2021ibot} & 87.36 $\pm$ 0.53 & 81.78 $\pm$ 0.24 & 76.54 $\pm$ 0.88 & 69.52 $\pm$ 0.48 & 52.34 $\pm$ 1.39 & 42.40 $\pm$ 0.97 \\
\rowcolor{yellow!10}
iBOT ImageNet-22K \cite{zhou2021ibot} & 89.91 $\pm$ 0.44 & 84.76 $\pm$ 0.40 & 73.93 $\pm$ 0.65 & 65.37 $\pm$ 0.72 & 55.31 $\pm$ 1.91 & 44.42 $\pm$ 0.91 \\
\midrule
\rowcolor{green!10}
\multicolumn{7}{c}{\textbf{Supervised Methods}} \\
\rowcolor{green!10}
CLIP-Vision WIT \cite{radford2021learning} & 82.71 $\pm$ 0.69 & 74.91 $\pm$ 0.52 & \textbf{84.17 $\pm$ 0.91} & \textbf{77.91 $\pm$ 0.56} & 61.42 $\pm$ 0.64 & 51.88 $\pm$ 1.08 \\
\rowcolor{green!10}
MiiL ImageNet-1K \cite{ridnik2021imagenet} & 88.84 $\pm$ 0.17 & 83.12 $\pm$ 1.02 & 78.83 $\pm$ 0.45 & 72.63 $\pm$ 0.60 & 62.12 $\pm$ 0.34 & 51.28 $\pm$ 0.57 \\
\rowcolor{green!10}
SAM ImageNet-1K \cite{foret2020sharpness} & 91.28 $\pm$ 0.47 & 86.50 $\pm$ 0.51 & 74.86 $\pm$ 0.68 & 68.29 $\pm$ 0.65 & 53.81 $\pm$ 0.57 & 44.69 $\pm$ 0.58 \\
\rowcolor{green!10}
MiiL ImageNet-21K \cite{ridnik2021imagenet} & 87.83 $\pm$ 0.39 & 82.37 $\pm$ 1.31 & 74.09 $\pm$ 0.48 & 66.29 $\pm$ 0.82 & 56.24 $\pm$ 0.62 & 44.85 $\pm$ 1.29 \\
\midrule
\rowcolor{blue!14}
Supervised ImageNet-1K & \underline{93.16 $\pm$ 0.46} & \underline{89.43 $\pm$ 0.34} & \underline{83.18 $\pm$ 0.40} & \underline{77.76 $\pm$ 0.25} & \textbf{68.76 $\pm$ 0.71} & \textbf{60.58 $\pm$ 0.56} \\
\rowcolor{blue!14}
Supervised ImageNet-21K & \textbf{93.76 $\pm$ 0.35} & \textbf{90.27 $\pm$ 0.20} & 82.77 $\pm$ 0.66 & 77.17 $\pm$ 0.19 & \underline{66.56 $\pm$ 1.53} & \underline{58.54 $\pm$ 0.91} \\
\bottomrule
\end{tabular}
\end{table*}

\subsection{Other Pre-trained Weights}
To more comprehensively explore the impact of \(f_0^*\) in Eq. \ref{f_NTK} on the final performance of NTK-CL, extensive experiments using other pre-trained weights for \textit{ViT-B/16} are conducted. To ensure absolute fairness, the hyper-parameters and training strategies involved during their training are kept completely consistent, with only the backbone parameters differing. The results, as shown in Table~\ref{tab:accuracy_weights}, reveal several key insights.

Firstly, self-supervised methods exhibit notable variability in performance across various continual tasks. Among them, iBOT ImageNet-22K \cite{zhou2021ibot} achieves the highest incremental accuracy on both CIFAR-100 and ImageNet-A, indicating a positive correlation between the scale of pre-training data, model generalization, and resistance to forgetting, as discussed in \cref{Generalization_Impact_Factor}. In contrast, the masked modeling generative method MAE demonstrates significant limitations, with inferior performance in both task-specific accuracy and knowledge retention. This deficiency is primarily attributed to its pixel-level masked reconstruction objective, which emphasizes low-level structural recovery at the expense of learning semantically discriminative features. As a result, MAE fails to maintain sufficient class separability, reducing its effectiveness in NTK-CL.
Secondly, supervised pre-trained weights consistently deliver superior performance across all evaluated datasets, significantly outperforming both self-supervised and customized supervised alternatives. This suggests that excessive specialization in pre-training objectives does not necessarily enhance generalization in PEFT-CL scenarios.
Finally, CLIP-Vision, despite relying solely on visual modality input, achieves state-of-the-art performance exclusively on the ImageNet-R dataset, while exhibiting slight limitations on other benchmarks. We attribute this phenomenon to its alignment with the semantic complexity inherent in ImageNet-R. To further investigate the causes of performance variation, particularly MAE and CLIP-Vision, we provide detailed visualization analyses in Appendix~\ref{visualization}.

These findings underscore the pivotal importance of choosing appropriate pre-training weights \(f_0^*\) for optimizing PEFT-CL performance. They direct future research toward enhancing model robustness and generalization capabilities, crucial for dynamic learning environments.

\section{Conclusion}
In this study, we adopt an NTK perspective to analyze PEFT-CL tasks, elucidating model behavior and generalization gaps in sequential task learning. Our analysis identifies crucial factors affecting PEFT-CL effectiveness, particularly through the dynamics of task-interplay and task-specific generalization gaps. We recommend strategies to mitigate these gaps, such as expanding sample sizes, enforcing task-level feature constraints, and refining regularization techniques. These strategies inform architectural and optimization adjustments, enhancing model generalization while advancing their theoretical and practical foundations.

\section{Future Work and Discussions}
With the emergence of pre-trained Large Language Models (LLMs), a fundamental challenge is extending the NTK-CL framework to encompass both LLMs and Multimodal Large Language Models (MLLMs/Omni-Models). Although several preliminary approaches have been proposed \cite{chen2023parameterizing, zhao2024sapt, qin2022lfpt, razdaibiedina2023progressive, wang2023orthogonal, yang2024parameter}, they predominantly focus on simplified architectures, such as T5, and have yet to demonstrate scalability or efficacy on more sophisticated LLMs and generalist Omni-Models. A detailed discussion of these limitations is provided in Appendix \ref{appendix:llms}. Additionally, although generative self-supervised pre-training schemes (e.g., MAE) achieve strong generalization across some other domains, their deployment in PEFT-CL settings exposes limitations, particularly the issue of semantic indistinguishability, which remains an open problem and warrants further investigation. Lastly, future work should prioritize developing theoretically grounded Bayesian hyper-parameter search algorithms that, like Dynamic Loss Scaling, introduce minimal overhead while ensuring rigorous mathematical guarantees.


%





\ifCLASSOPTIONcaptionsoff
  \newpage
\fi



%


\bibliographystyle{ieee_fullname}
\bibliography{refs}

\vspace{-6ex}
\begin{IEEEbiography}[{\includegraphics[width=1in,height=1.25in,clip,keepaspectratio]{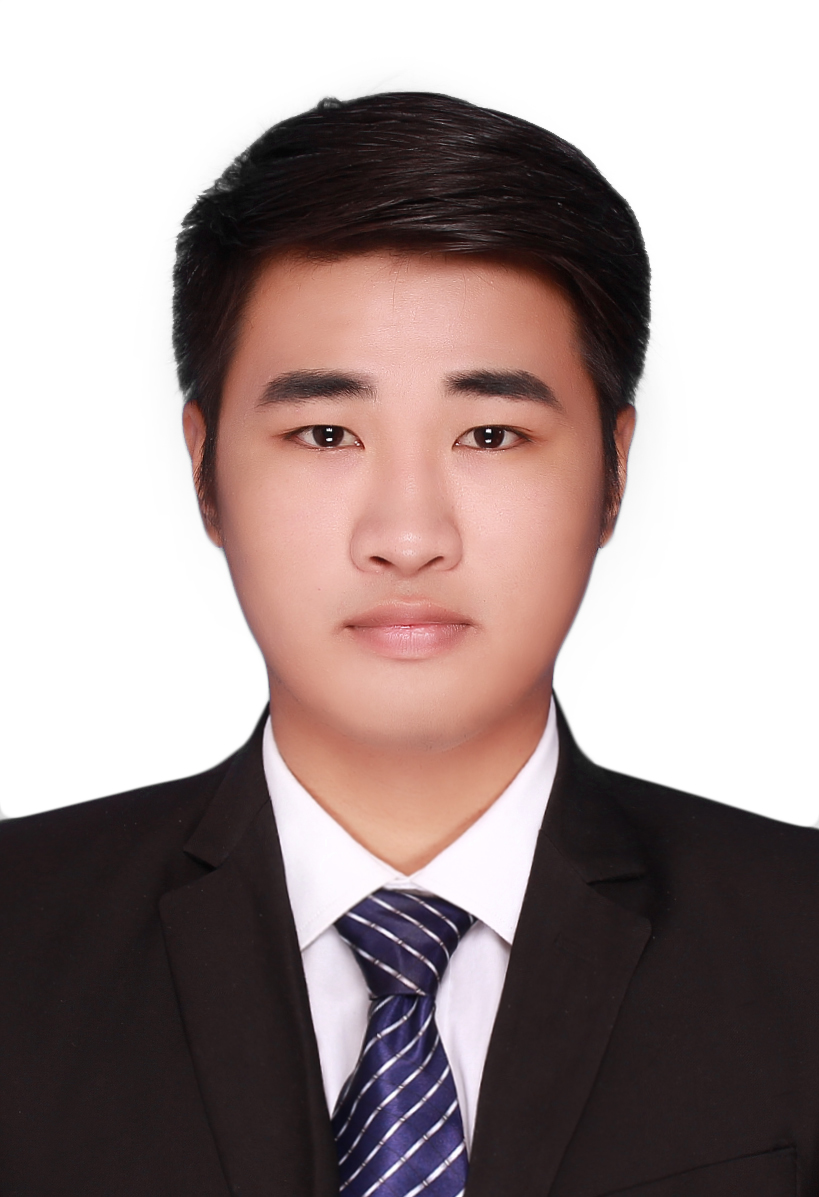}}]{Jingren Liu} received the B.S. degree in Computer Science and Technology from Nanjing University of Finance and Economy, Nanjing, China, in 2019, and is currently working toward the PhD degree in the School of Electrical and Information Engineering, Tianjin University, Tianjin, China. His current research interests include continual learning, few-shot learning, and prompt learning.
\end{IEEEbiography}

\vspace{-4ex}
\begin{IEEEbiography}[{\includegraphics[width=1in,height=1.25in,clip,keepaspectratio]{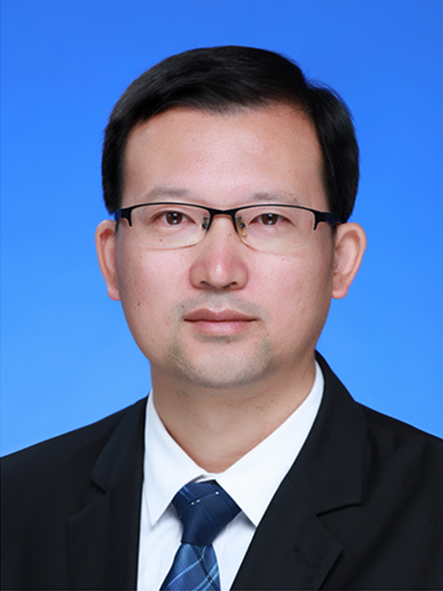}}]{Zhong Ji} received the Ph.D. degree in signal and information processing from Tianjin University, Tianjin, China, in 2008. He is currently a Professor with the School of Electrical and Information Engineering, Tianjin University. He has authored over 100 technical articles in refereed journals and proceedings. His current research interests include continual learning, few shot leanring, and cross-modal analysis.
\end{IEEEbiography}

\vspace{-4ex}
\begin{IEEEbiography}[{\includegraphics[width=1in,height=1.25in,clip,keepaspectratio]{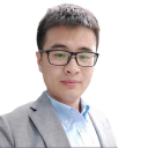}}]{YunLong Yu} received the Ph.D. degree in information and communication engineering from Tianjin University, Tianjin, China, in 2019. He is currently a Distinguished Researcher with the College of Information Science and Electronic Engineering, Zhejiang University, Hangzhou, China. His current research interests include machine learning and computer vision.
\end{IEEEbiography}

\vspace{-4ex}
\begin{IEEEbiography}[{\includegraphics[width=1in,height=1.25in,clip,keepaspectratio]{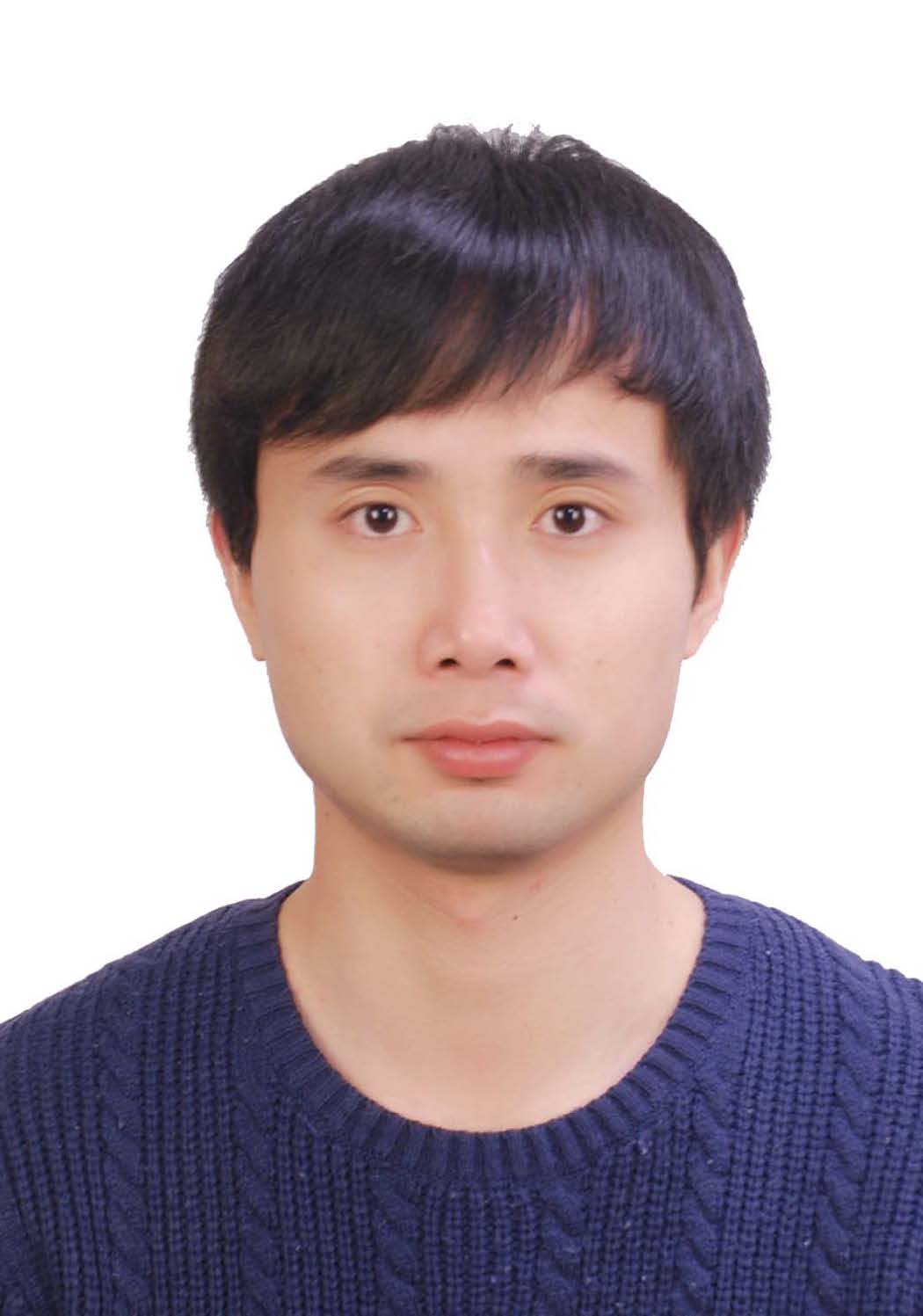}}]{Jiale Cao} received the Ph.D degree in information and communication engineering from Tianjin University, Tianjin, China, in 2018. He is currently an Associate Professor with Tianjin University. His research interests include image understanding and analysis, in which he has published 30+ IEEE Transactions and CVPR/ICCV/ECCV articles. He serves as a regular Program Committee Member for leading computer vision and artificial intelligence conferences, such as CVPR, ICCV, and ECCV.
\end{IEEEbiography}

\vspace{-4ex}
\begin{IEEEbiography}[{\includegraphics[width=1in,height=1.25in,clip,keepaspectratio]{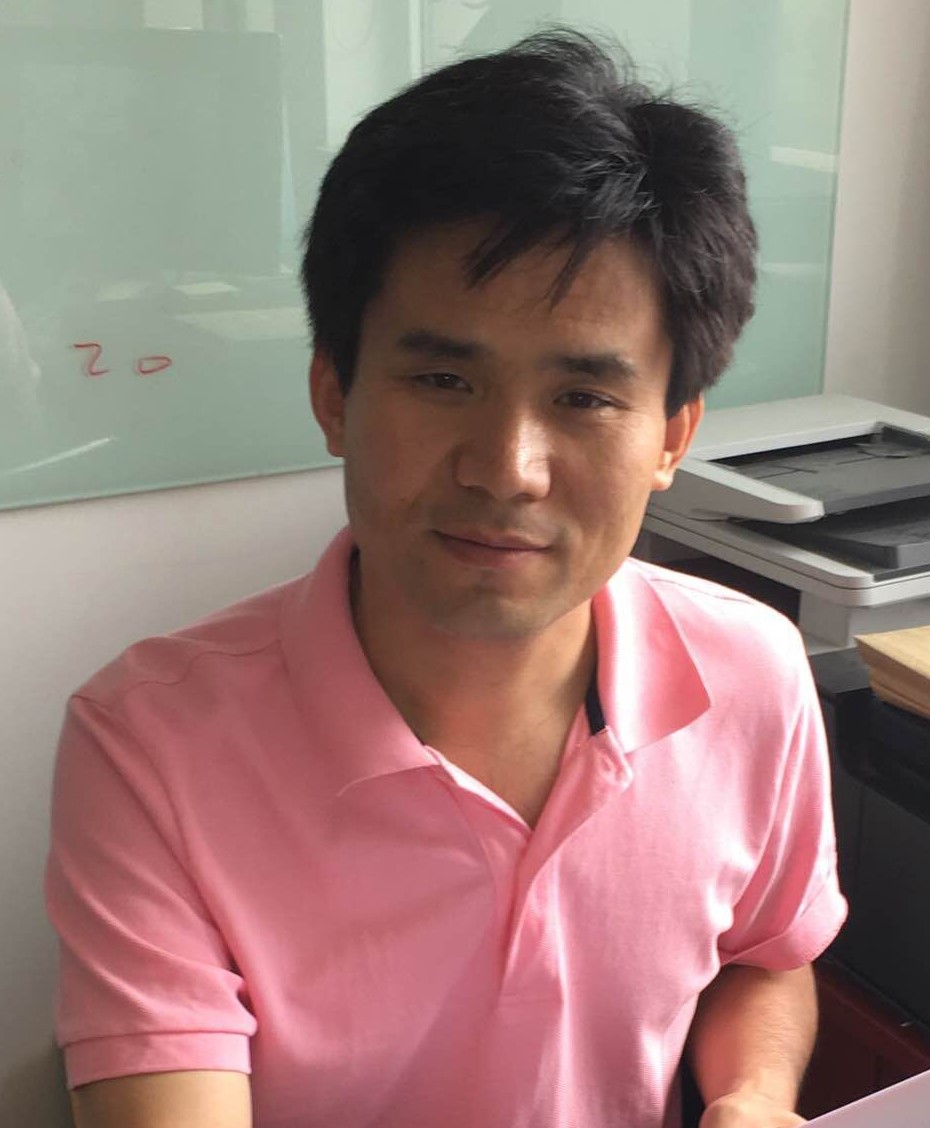}}]{YanWei Pang} received the Ph.D. degree in electronic engineering from the University of Science and Technology of China, Hefei, China, in 2004. He is currently a Professor with the School of Electrical and Information Engineering, Tianjin University, Tianjin, China. He has authored over 200 scientific papers. His current research interests include object detection and recognition, vision in bad weather, and computer vision.
\end{IEEEbiography}

\vspace{-4ex}
\begin{IEEEbiography}[{\includegraphics[width=1in,height=1.25in,clip,keepaspectratio]{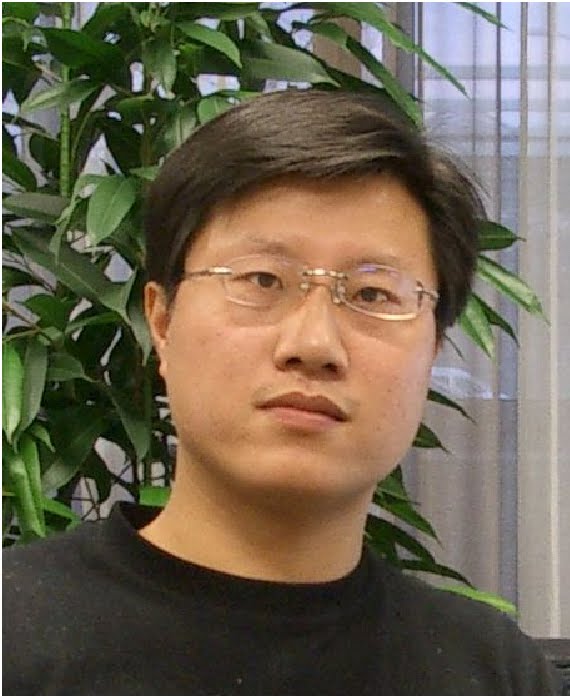}}]{Jungong Han} received the Ph.D. degree in telecommunication and information system from Xidian University, Xi’an, China, in 2004. He is a Chair Professor of Computer Vision with the Department of Computer Science, University of Sheffield, U.K. He has published over 200 articles, including more than 80 IEEE Transactions and more than 50 A* conference articles. His research interests span the fields of video analysis, computer vision, and applied machine learning. He is a Fellow of the International Association of Pattern Recognition.
\end{IEEEbiography}

\vspace{-4ex}
\begin{IEEEbiography}{Xuelong Li} is the CTO and Chief Scientist of China Telecom, where he founded the Institute of Artificial Intelligence (TeleAI) of China Telecom.
\end{IEEEbiography}

\clearpage

\appendices
\section{Prelude to the Study}
In the realm of PEFT-CL, we initiate with the foundational model \(f_0^*\), where the superscript \(^*\) signifies parameters optimized to their prime configuration, distinguishing them from those still under optimization. Our objective is to adeptly modify the feature space from \(f_0^*\) for each specific task \(i\) (represented as \(f_i^*\)), by finely adjusting the optimizable subnetwork parameters \(p_i\). This critical adaptation ensures that alterations in shared subnetwork components across different tasks do not lead to excessive catastrophic forgetting, thereby safeguarding the model’s generalizability.

To methodically investigate PEFT-CL, we conceptualize a sequence of \(T\) tasks, each optimizing subnetwork parameters \(p_\tau^*\) for \(1 \leq \tau \leq T\). This strategic adjustment of \(p_\tau^*\) enables the model to achieve an optimal state \(f_{\tau}^*\), thereby generating a specialized feature space for each task.

Expanding beyond traditional heuristic approaches \cite{wang2022learning,wang2022dualprompt,smith2023coda,jung2023generating,zhou2024expandable} prevalent in PEFT-CL for adjusting subnetwork components, our analysis delves into the training dynamics and explores the potential for reducing generalization gaps through the lens of NTK. This rigorous analysis helps pinpoint necessary adjustments to minimize generalization gaps and maximize model performance across diverse tasks. Grounded in seminal theories and contemporary studies in generalization dynamics \cite{doan2021theoretical, chai2009generalization, canatar2021spectral, Karakida2021LearningCF, bennani2020generalisation}, we introduce advanced tools for assessing task interplay and specific generalization gaps, as detailed in \cref{Task_Interplay_Generalization} and \cref{Task_Specific_Generalization}.

For comprehensive clarity in theoretical discourse, we segment our discussion into three distinct parts: analyzing NTK dynamics specific to PEFT-CL (Appendix~\ref{NTK_Dynamics}), evaluating inter-task generalization gap (Appendix~\ref{Interplay_Generalization_in_PEFT_CL}), and scrutinizing intra-task generalization gap (Appendix~\ref{Intrinsic_Generalization_in_PEFT_CL}). These segments collectively aim to provide a deep understanding of the underlying mechanisms influencing PEFT-CL performance, thereby informing better implementation practices in this field.

\section{NTK Dynamics in PEFT-CL}\label{NTK_Dynamics}
Initially, we concentrate on analyzing the least squares loss associated with the optimization of consecutive tasks \( \tau \) and \( \tau - 1 \). This involves quantifying the classification loss attributable to variations in the subnetwork components' parameters, which is expressed as follows:
\begin{equation}
    \begin{aligned}
        \mathcal{L}(p_{\tau}|X, Y \in \mathcal{D_{\tau}}) = & \mathop{argmin}\limits_{p_{\tau}} \Bigl\|f_{\tau-1}^*(X) + \nabla_{p_\tau} f_\tau(X) \\
        & \times (p_\tau - p_{\tau-1}^*) - Y\Bigr\|_2^2, \\
        = & \mathop{argmin}\limits_{p_{\tau}} \Bigl\|f_{\tau-1}^*(X) + \phi_\tau(X)\\
        & \times (p_\tau - p_{\tau-1}^*) - Y\Bigr\|_2^2.
    \end{aligned}
    \label{NTK_Dynamics_0}
\end{equation}
Here, \(D_{\tau}\) refers to the data subset associated with the \(\tau\)-th task, where \(X\) and \(Y\) are the input images and corresponding labels, respectively. The term \(\phi_\tau(\cdot)\) denotes the Jacobian matrix relevant to task \(\tau\) for the inputs \(X\).
At the onset of a task's optimization, the subnetwork component parameters inherit parameters from the preceding task, setting the initial states for optimizing \( f_\tau(\cdot) \) and \( p_\tau \) as \( f_{\tau-1}^*(\cdot) \) and \( p_{\tau-1}^* \), respectively.

To ensure a globally optimal solver \( p_\tau^* \) for \( p_\tau \) and a well-posed solution, we introduce an appropriate regularization term, transforming the initial loss defined in Eq. \ref{NTK_Dynamics_0} to:
\begin{equation}
    \begin{aligned}
        \mathcal{L}(p_{\tau}|X, Y & \in \mathcal{D_{\tau}}) = \mathop{argmin}\limits_{p_{\tau}} \left\|f_{\tau-1}^*(X) + \nabla_{p_\tau} f_\tau(X) \right. \\
        & \left. \times (p_\tau - p_{\tau-1}^*)-Y \right\|_2^2 + \lambda \left\|p_\tau - p_{\tau-1}^*\right\|_2^2.
    \end{aligned}
    \label{NTK_Dynamics_1}
\end{equation}

The saddle-point solution of Eq. \ref{NTK_Dynamics_1} is given by:
\begin{equation}
    \begin{aligned}
        p_\tau - p_{\tau-1}^* = \phi_\tau(X)^\top(\phi_\tau(X)^\top\phi_\tau(X)+\lambda I)^{-1}(Y-f_{\tau-1}^*(X)).
    \end{aligned}
    \label{NTK_Dynamics_2}
\end{equation}

Consequently, the optimal dynamic outputs for the \(\tau\)-th task during optimization can be expressed as:
\begin{equation}
\small
    \begin{aligned}
        f_{\tau}(x) =& f_{\tau-1}^*(x) + \nabla_{p_{\tau}}f_\tau(x) (p_\tau-p_{\tau-1}^*), \\
        =& f_{\tau-1}^*(x) + \nabla_{p_{\tau}}f_\tau(x) \phi_\tau(X)^\top \\
        &\quad \times(\phi_\tau(X)^\top\phi_\tau(X)+\lambda I)^{-1}(Y-f_{\tau-1}^*(X)), \\
        =& f_{\tau-1}^*(x) + \phi_\tau(x) \phi_\tau(X)^\top \\
        &\quad \times(\phi_\tau(X)^\top\phi_\tau(X)+\lambda I)^{-1}(Y-f_{\tau-1}^*(X)), \\
        =& f_{\tau-1}^*(x) + \Phi_\tau(x, X)^\top \\
        &\quad \times(\Phi_\tau(X, X)+\lambda I)^{-1}(Y-f_{\tau-1}^*(X)).
    \end{aligned}
    \label{NTK_Dynamics_3}
\end{equation}

Denoting \(\tilde{Y}_{\tau} = Y - f_{\tau-1}^*(X)\), Equations \ref{NTK_Dynamics_2} and \ref{NTK_Dynamics_3} can be articulated as:
\begin{equation}
    \begin{aligned}
        p_\tau - p_{\tau-1}^* = \phi_\tau(X)^\top(\Phi_\tau(X, X)+\lambda I)^{-1}\tilde{Y}_{\tau}.
    \end{aligned}
    \label{NTK_Dynamics_4}
\end{equation}
\begin{equation}
    \begin{aligned}
        f_{\tau}(x) - f_{\tau-1}^*(x) = \Phi_\tau(x, X)^\top(\Phi_\tau(X, X)+\lambda I)^{-1}\tilde{Y}_{\tau}.
    \end{aligned}
    \label{NTK_Dynamics_5}
\end{equation}

Summing over Eq. \ref{NTK_Dynamics_5}, we obtain:
\begin{equation}
    \begin{aligned}
        f_{\tau}(x) = f_0^*(x) + \sum\limits_{i=1}^{\tau} \Phi_i(x, X)(\Phi_i(X, X)+\lambda I)^{-1}\tilde{Y}_i.
    \end{aligned}
    \label{NTK_Dynamics_6}
\end{equation}

Ultimately, when \( f_{\tau}(\cdot) \) is optimized to the global optimum for task \(\tau\), the NTK, derived from the gradients of \( p_\tau \), is expected to converge and stabilize, preserving the forms of Equations \ref{NTK_Dynamics_4}, \ref{NTK_Dynamics_5}, and \ref{NTK_Dynamics_6}.
\begin{equation}
    \begin{aligned}
        p_\tau^* - p_{\tau-1}^* = \phi_\tau(X)^\top(\Phi_\tau(X, X)+\lambda I)^{-1}\tilde{Y}_{\tau}.
    \end{aligned}
    \label{NTK_Dynamics_7}
\end{equation}
\begin{equation}
\small
    \begin{aligned}
        f_{\tau}^*(x) - f_{\tau-1}^*(x) = \Phi_\tau(x, X)^\top(\Phi_\tau(X, X)+\lambda I)^{-1}\tilde{Y}_{\tau}.
    \end{aligned}
    \label{NTK_Dynamics_8}
\end{equation}
\begin{equation}
    \begin{aligned}
        f_{\tau}^*(x) = f_0^*(x) + \sum\limits_{i=1}^{\tau} \Phi_i(x, X)(\Phi_i(X, X)+\lambda I)^{-1}\tilde{Y}_i.
    \end{aligned}
    \label{NTK_Dynamics_9}
\end{equation}

As delineated in Eq. \ref{NTK_Dynamics_9}, the output for task \(\tau\) fundamentally hinges on the NTKs associated with the preceding \(\tau\) tasks, the corresponding data labels, and the initial pre-trained weight.

\section{Task-Interplay Generalization in PEFT-CL}\label{Interplay_Generalization_in_PEFT_CL}
In this section, we explore the dynamics of task-interplay generalization gap within the PEFT-CL scenario, utilizing the NTK theory. We begin by outlining relevant mathematical properties of the NTK, followed by detailed analyses and derivations to elucidate how these properties influence generalization across tasks. This rigorous approach aims to provide a robust theoretical foundation for understanding the interplay between task transitions in PEFT-CL scenarios.

Owing to the reproducing property of the NTK function in the Reproducing Kernel Hilbert Space (RKHS), we deduce that for any task and any model, it follows that:
\begin{equation}
    \begin{aligned}
         f(x) = \left< \Phi(\cdot,x), f \right>_{\mathcal{H}}.
    \end{aligned}
    \label{Interplay_Generalization_in_PEFT_CL_0}
\end{equation}

In accordance with Mercer's Theorem, within an ideal RKHS, the NTK can be expressed as an infinite sum of orthogonal basis functions and eigenvalues:
\begin{equation}
    \begin{aligned}
         \Phi(x, x') = \sum_\rho \lambda_\rho O_\rho(x) O_\rho(x') = \sum_\rho \varphi_\rho(x) \varphi_\rho(x'), |\rho| \rightarrow \infty,
    \end{aligned}
    \label{Interplay_Generalization_in_PEFT_CL_1}
\end{equation}
where \(\lambda\) and \(O(\cdot)\) denote the eigenvalues and eigenfunctions from the decomposition, and \(|\rho|\) signifies the count of eigenvalues and eigenfunctions realized post-decomposition. For clarity in subsequent derivations, we define \(\varphi(\cdot) = \sqrt{\lambda} O(\cdot)\).

In addition, by denoting \(\Phi_\tau(x, X)(\Phi_\tau(X, X)+\lambda I)^{-1}\tilde{Y}_\tau\) from Eq. \ref{NTK_Dynamics_8} as \(\alpha_\tau\), we deduce:
\begin{equation}
    \begin{aligned}
        \tilde{f}^*_\tau(x) = \Phi_\tau(x, X)^\top \alpha_\tau = \sum\limits_{i=1}^{n_t} \Phi_\tau(x, x^i)^\top \alpha_\tau^i,
    \end{aligned}
    \label{Interplay_Generalization_in_PEFT_CL_2}
\end{equation}
Here, \( \tilde{f} \) denotes the functional difference between the outcomes of two consecutive tasks.

From the aforementioned content, it is known that in the RKHS, the norm of the function \( \tilde{f} \) can be denoted as:
\begin{equation}
    \begin{aligned}
        || \tilde{f}^*_\tau ||^2_\mathcal{H} = \alpha_\tau^\top  \Phi_\tau(X, X) \alpha_\tau.
    \end{aligned}
    \label{Interplay_Generalization_in_PEFT_CL_3}
\end{equation}

Considering that $(\Phi_\tau(X, X)+\lambda I)^{-1} \leq (\Phi_\tau(X, X))^{-1}$ holds, we deduce the following inequality:
\begin{equation}
    \begin{aligned}
        || \tilde{f}_\tau^* ||^2_\mathcal{H} & = \tilde{Y}_\tau^\top (\Phi_\tau(X, X)+\lambda I)^{-1} \\
        & \quad \times \Phi_\tau(X, X) (\Phi_\tau(X, X)+\lambda I)^{-1}\tilde{Y}_\tau, \\
        & \leq \tilde{Y}_\tau^\top (\Phi_\tau(X, X)+\lambda I)^{-1} \\
        & \quad \times \Phi_\tau(X, X) (\Phi_\tau(X, X))^{-1}\tilde{Y}_\tau, \\
        & \leq \tilde{Y}_\tau^\top (\Phi_\tau(X, X)+\lambda I)^{-1} \tilde{Y}_\tau, \\
        & \leq G^2_\tau.
    \end{aligned}
    \label{Interplay_Generalization_in_PEFT_CL_4}
\end{equation}

In relation to Equations \ref{NTK_Dynamics_9}, \ref{Interplay_Generalization_in_PEFT_CL_2} and considering the symmetry properties of the inner product in high-dimensional Hilbert spaces, we can deconstruct it as:
\begin{equation}
    \begin{aligned}
        f_T(x) & = \sum\limits_{\tau=1}^{T} \sum\limits_{i=1}^{n_\tau} \alpha_\tau^i \left< \varphi_\tau(x), \varphi_\tau(x_\tau^i) \right>_\mathcal{H}, \\
        & = \sum\limits_{\tau=1}^{T} \sum\limits_{i=1}^{n_\tau} \alpha_\tau^i \left< \varphi_\tau(x_\tau^i), \varphi_\tau(x) \right>_\mathcal{H}, \\
        & = \sum\limits_{\tau=1}^{T} \left< \sum\limits_{i=1}^{n_\tau} \alpha_\tau^i \varphi_\tau(x_\tau^i), \varphi_\tau(x) \right>_\mathcal{H},
    \end{aligned}
    \label{Interplay_Generalization_in_PEFT_CL_5}
\end{equation}
Here, $\varphi_\tau(\cdot)$ denotes the matrix of orthogonal eigenfunctions and associated eigenvalues obtained from decomposing \(\Phi_\tau(\cdot)\) in the RKHS for each task \(\tau\). For this analysis, \( f_0^*(x) \) is omitted, acting as a baseline constant in the model's performance.

Considering the properties of Eq. \ref{Interplay_Generalization_in_PEFT_CL_3}, we infer:
\begin{equation}
    \begin{aligned}
        ||\sum\limits_{i=1}^{n_\tau} \alpha_\tau^i \varphi_\tau(x_\tau^i)||_{\mathcal{H}}^2 & = \sum\limits_{i,j} \alpha_\tau^{i\top} \Phi_\tau(x_\tau^i, x_\tau^j) \alpha_\tau^j \leq G^2_\tau,
    \end{aligned}
    \label{Interplay_Generalization_in_PEFT_CL_6}
\end{equation}
\begin{equation}
\small
    \begin{aligned}
        \mathcal{F}_T \subset \{x \rightarrow \sum\limits_{\tau=1}^T \left< w_\tau, \varphi_\tau(x) \right>_{\mathcal{H}}, ||w_\tau||_{\mathcal{H}}^2 \leq G^2_\tau \}_{\mathcal{D}} := \tilde{\mathcal{F}}_T,
    \end{aligned}
    \label{Interplay_Generalization_in_PEFT_CL_7}
\end{equation}
Initially, the set \(\tilde{\mathcal{F}}_T\) comprises functions characterized by the inner product between the feature mapping \(\varphi_\tau(x)\) and the weight vector \(w_\tau\) within the RKHS, in the form of \(\left< w_\tau, \varphi_\tau(x) \right>_{\mathcal{H}}\). Accordingly, any arbitrary function \(f_\tau^*(x)\) in \(\mathcal{F}_T\) can be decomposed into the sum of the output of the previous task and the current task output change \(f_\tau^*(x) = f_{\tau-1}^*(x) + \tilde{f}_\tau^*(x)\), and \(\tilde{f}_\tau^*(x)\) is reconstructed into \(\left< w_\tau, \varphi_\tau(x) \right>_{\mathcal{H}}\). Thus, as every function \(f_\tau^*(x)\) in \(\mathcal{F}_T\) can be reconstructed into the form found in \(\tilde{\mathcal{F}}_T\), it can be concluded that \(\mathcal{F}_T\) is a subset of \(\tilde{\mathcal{F}}_T\).

Combining the computation method of Rademacher Complexity, we obtain the upper bound of \(\hat{\mathcal{R}}(\tilde{\mathcal{F}})\),
\begin{align}
    \hat{\mathcal{R}}(\mathcal{F}) &= \mathbb{E}_{\epsilon}\left[\sup_{f \in \mathcal{F}}\frac{1}{n}\sum_{i=1}^{n}\epsilon_i f(x_i)\right], \label{Interplay_Generalization_in_PEFT_CL_8} \\
    \hat{\mathcal{R}}(\mathcal{F}_T) &\leq \hat{\mathcal{R}}(\tilde{\mathcal{F}}_T), \nonumber \\
    &= \sum\limits_{\tau=1}^T \mathbb{E}_{\epsilon}\left[\sup_{||w_\tau||_{\mathcal{H}}^2 \leq G_\tau^2} \left< w_\tau, \frac{1}{n_\tau}\sum_{i=1}^{n_\tau}\epsilon_i \varphi_\tau(x_\tau^i) \right>_{\mathcal{H}} \right], \label{Interplay_Generalization_in_PEFT_CL_9}
\end{align}
where \(\epsilon_i\) are independently and identically distributed random variables, taking values of \(\pm 1\). And since \(\mathcal{F}_T\) is a subset of \(\tilde{\mathcal{F}}_T\), its Rademacher Complexity is less than or equal to that of \(\tilde{\mathcal{F}}_T\).

\begin{lemma}
Consider a kernel $k : \mathcal{X} \times \mathcal{X} \to \mathbb{R}$, and let $X_1, \ldots, X_n$ be random elements of $\mathcal{X}$. Then for the class $\mathcal{F}$ defined above,
\begin{align}
    \hat{\mathcal{G}}_n(\mathcal{F}) &\leq \frac{2B}{n} \sqrt{\sum_{i=1}^n \mathbb{E}[k(X_i, X_i)]}, \\
    \hat{\mathcal{R}}_n(\mathcal{F}) &\leq \frac{2B}{n} \sqrt{\sum_{i=1}^n \mathbb{E}[k(X_i, X_i)]}.
\end{align}
\label{lemma_rc}
\end{lemma}
\begin{proof}
    Suppose that $\mathcal{H}$ is a Hilbert space with inner product $\langle \cdot, \cdot \rangle$ and induced norm $\| \cdot \|$, and the kernel $k$ has feature map $\phi : \mathcal{X} \to \mathcal{H}$. Let $g_1, \ldots, g_n$ be independent standard normal random variables. Then
    \begin{equation}
        \begin{split}
            \hat{\mathcal{G}}_n(\mathcal{F}) & \leq \mathbb{E} \left[ \sup_{\|w\| \leq B} \left< w, \frac{2}{n} \sum_{i=1}^n g_i \phi(X_i)\right>_{\mathcal{H}, \mathcal{D}} \right] \\
            & = \frac{2B}{n} \mathbb{E} \left[ \sqrt{\sum_{i=1}^n g_i^2 \phi(X_i)^T \phi(X_i)} \right] \\
            & = \frac{2B}{n} \mathbb{E} \left[ \sqrt{\sum_{i, j}^n g_i g_j k(X_i, X_j)} \right] \\
            & \leq \frac{2B}{n} \sqrt{\mathbb{E} \left[ \sum_{i, j}^n g_i g_j k(X_i, X_j) \right]} \\
            & = \frac{2B}{n} \sqrt{\sum_{i=1}^n \mathbb{E}[k(X_i, X_i)]}.
        \end{split}
    \end{equation}
    Clearly, the same argument applies with any independent, zero mean, unit variance random variables replacing the $g_i$, which gives the same bound for $\hat{\mathcal{R}}_n(\mathcal{F})$.
    
    From the definitions and Jensen's inequality, we can deduce that:
    \begin{align}
        R_n(\mathcal{F}) &= \mathbb{E}R_n(\mathcal{F}) \leq 2B\sqrt{\frac{\mathbb{E}k(X, X)}{n}}, \\
        G_n(\mathcal{F}) &= \mathbb{E}\hat{G}_n(\mathcal{F}) \leq 2B\sqrt{\frac{\mathbb{E}k(X, X)}{n}}.
    \end{align}
    It is noteworthy that $\mathbb{E}k(X, X)$ represents the trace (sum of the eigenvalues) of the integral operator $T_k$ defined on $L_2(\mu)$,
    \begin{equation}
        T_k(f) = \int k(x, y)f(y)d\mu(y),
    \end{equation}
    where $\mu$ is the induced probability measure on $\mathcal{X}$.
\end{proof}

Utilizing the \cref{lemma_rc} from \cite{bartlett2002rademacher}, we can derive:
\begin{equation}
    \begin{aligned}
        \hat{\mathcal{R}}(\mathcal{F}_T) & \leq \left[\sum\limits_{\tau=1}^T \frac{G_\tau}{n_\tau} \sqrt{\text{Tr}(\Phi_\tau(X, X))} \right]_{\mathcal{D}_\tau}, \\
        & = \left[\sum\limits_{\tau=1}^T \sqrt{\frac{[\tilde{Y}_\tau^\top (\Phi_\tau(X, X)+\lambda I)^{-1} \tilde{Y}_\tau] \text{Tr}(\Phi_\tau(X, X))}{n_\tau^2}} \right]_{\mathcal{D}_\tau}, \\
        & = \left[\sum\limits_{\tau=1}^T \mathcal{O}(\sqrt{\frac{[\tilde{Y}_\tau^\top (\Phi_\tau(X, X)+\lambda I)^{-1} \tilde{Y}_\tau]}{n_\tau}}) \right]_{\mathcal{D}_\tau}.
    \end{aligned}
    \label{Interplay_Generalization_in_PEFT_CL_10}
\end{equation}

Expanding upon Eq. \ref{NTK_Dynamics_8}, we express the generalization dynamics of PEFT-CL for the final task as follows:
\begin{equation}
    \begin{aligned}
        f_T^*(x) = f^*_\tau(x) + \sum\limits_{k=\tau+1}^{T} \tilde{f}^*_{k}(x),
    \end{aligned}
    \label{Interplay_Generalization_in_PEFT_CL_11}
\end{equation}
\begin{equation}
\small
    \begin{aligned}
        ||f_T^*(X_\tau)-Y_\tau||_2^2 & = ||f^*_\tau(X_\tau)+\sum\limits_{k=\tau+1}^{T} \tilde{f}^*_{k}(X_\tau)-Y_\tau||_2^2, \\
        & \leq ||f^*_\tau(X_\tau)-Y_\tau||_2^2 + \sum\limits_{k=\tau+1}^{T} ||\tilde{f}^*_{k}(X_\tau)||_2^2. 
    \end{aligned}
    \label{Interplay_Generalization_in_PEFT_CL_12}
\end{equation}

For the first term on the right-hand side of Eq. \ref{Interplay_Generalization_in_PEFT_CL_12}, we derive the following inequality:
\begin{equation}
\small
    \begin{aligned}
        ||f^*_\tau(X_\tau)-Y_\tau||_2^2 & = ||\tilde{f}^*_\tau(X_\tau)+f^*_{\tau-1}(X_\tau)-Y_\tau||_2^2, \\
        & = ||\tilde{f}^*_\tau(X_\tau)-\tilde{Y}_\tau||_2^2, \\
        & = ||\Phi_\tau(X_\tau,X_\tau)^\top(\Phi_\tau(X_\tau,X_\tau) +\lambda I)^{-1} \\
        & \quad \times \tilde{Y}_\tau-\tilde{Y}_\tau||_2^2, \\
        & = ||[\Phi_\tau(X_\tau,X_\tau) + \lambda I - \lambda I]^\top \\
        & \quad \times (\Phi_\tau(X_\tau,X_\tau) +\lambda I)^{-1}\tilde{Y}_\tau-\tilde{Y}_\tau||_2^2, \\
        & = ||\tilde{Y}_\tau - \lambda(\Phi_\tau(X_\tau,X_\tau) +\lambda I)^{-1} \\
        & \quad \times \tilde{Y}_\tau-\tilde{Y}_\tau||_2^2, \\
        & = \lambda^2||(\Phi_\tau(X_\tau,X_\tau) +\lambda I)^{-1} \\
        & \quad \times \tilde{Y}_\tau||_2^2, \\
        & \leq \lambda^2 \tilde{Y}_\tau^\top(\Phi_\tau(X_\tau,X_\tau) +\lambda I)^{-1}\tilde{Y}_\tau.
    \end{aligned}
    \label{Interplay_Generalization_in_PEFT_CL_13}
\end{equation}

Utilizing the formulation in Eq. \ref{NTK_Dynamics_8}, we further deduce:
\begin{equation}
\small
    \begin{aligned}
        ||\tilde{f}^*_k(X_\tau)||_2^2 = \ & \tilde{Y}_k^\top (\Phi_k(X_k, X_k) + \lambda I)^{-1} \Phi_k(X_\tau, X_k) \\
        & \times \Phi_k(X_\tau, X_k)^\top (\Phi_k(X_k, X_k)+\lambda I)^{-1}\tilde{Y}_k.
    \end{aligned}
    \label{Interplay_Generalization_in_PEFT_CL_14}
\end{equation}

Then, the inequality for \(||f_T^*(X_\tau)-Y_\tau||_2^2\) is given by:
\begin{equation}
    \begin{aligned}
        \mathcal{L}_S (f_T^*) & = ||f_T^*(X_\tau)-Y_\tau||_2^2 \\
        & \leq \frac{1}{n_\tau} \Big[ \lambda^2 \tilde{Y}_\tau^\top(\Phi_\tau(X_\tau,X_\tau) +\lambda I)^{-1}\tilde{Y}_\tau \\
        & + \sum\limits_{k=\tau+1}^{T} \tilde{Y}_k^\top (\Phi_k(X_k, X_k)+\lambda I)^{-1} \\
        & \quad \times \Phi_k(X_\tau, X_k) \Phi_k(X_\tau, X_k)^\top \\
        & \quad \times (\Phi_k(X_k, X_k)+\lambda I)^{-1}\tilde{Y}_k \Big].
    \end{aligned}
    \label{Interplay_Generalization_in_PEFT_CL_15}
\end{equation}

Building upon the insights of \cite{bartlett2002rademacher}, we can assert that, with probability at least \(1 - \delta\), the disparity between the population loss \(L_D(f)\) and the empirical loss \(L_S(f)\) for any function \(f\) within the function class \(\mathcal{F}_T\) is bounded as follows:
\begin{equation}
    \begin{aligned}
        \sup_{f \in \mathcal{F}_T} \{ L_D(f) - L_S(f) \} \leq 2 \rho \hat{\mathcal{R}}(\mathcal{F}_T) + 3 c \sqrt{\frac{\log(2/\delta)}{2 N}},
    \end{aligned}
    \label{Interplay_Generalization_in_PEFT_CL_16}
\end{equation}

Furthermore, applying this principle to our optimal function \(f^*_T\) from \(\mathcal{F}_T\), we obtain an upper bound for the population loss \(L_D(f^*_T)\) in terms of the empirical loss \(L_S(f^*_T)\), as delineated below:
\begin{equation}
    \begin{aligned}
        L_D(f^*_T) \leq L_S(f^*_T) + 2 \rho \hat{\mathcal{R}}(\mathcal{F}_T) + 3 c \sqrt{\frac{\log(2/\delta)}{2 N}},
    \end{aligned}
    \label{Interplay_Generalization_in_PEFT_CL_17}
\end{equation}
Here, \( \rho \) represents the Lipschitz constant. The term \( \hat{\mathcal{R}}(\mathcal{F}_T) \) refers to the empirical Rademacher complexity, as detailed in Eq. \ref{Interplay_Generalization_in_PEFT_CL_10}. \( L_S(f^*_T) \) represents the empirical loss in Eq. \ref{Interplay_Generalization_in_PEFT_CL_15} and \( \delta \) specifies the confidence level. While \( c \) is a constant and \( N \) denotes the total sample size.

\section{Task-Intrinsic Generalization in PEFT-CL}\label{Intrinsic_Generalization_in_PEFT_CL}
Utilizing Eq. \ref{NTK_Dynamics_9} and momentarily setting aside the initialization term \( f_0^*(x) \), we identify the NTK-related term for the entire task dataset as \( \alpha_i \). Incorporating its eigen-decomposition, we derive:
\begin{equation}
    \begin{aligned}
         f_{\tau}^*(x) &= \sum_{i=1}^{\tau} \alpha_i \sum_\rho \lambda_\rho O_\rho(x) O_\rho(X) \\
         &= \sum_\rho \left( \sum_{i=1}^{\tau} \alpha_i \varphi_\rho(X) \right) \varphi_\rho(x).
    \end{aligned}
    \label{Intrinsic_Generalization_in_PEFT_CL_0}
\end{equation}

Defining \( w_\rho = \sum\limits_{i=1}^{\tau} \alpha_i \varphi_\rho(X) \), the function \( f_{\tau}^*(x) \) is representable as \( f_{\tau}^*(x) = \sum\limits_\rho w_\rho \varphi_\rho(x) \). Consequently, under any task scenario, its output can be decomposed into a linear combination of eigenvalues and orthogonal eigenfunctions in the RKHS.

At this juncture, within the task, the generalization gap can be expressed as:
\begin{equation}
    \begin{aligned}
        \mathbb{E}_g(f_\tau, f_\tau^*) &= \left< (f_\tau(x) - y_\tau(x))^2 \right>_{x \in D_{\tau}} \\
        &= \sum_{\rho,\gamma} (w_\rho - w_\rho^*) (w_\gamma - w_\gamma^*) \left< \varphi_\rho(x), \varphi_\gamma(x) \right>_{x \in D_{\tau}}.
    \end{aligned}
    \label{Intrinsic_Generalization_in_PEFT_CL_1}
\end{equation}

Given that \(\varphi_\rho(x)\) and \(\varphi_\gamma(x)\) form the inner product of the Dirac function $\delta$ in RKHS, Eq. \ref{Intrinsic_Generalization_in_PEFT_CL_1} is transformed into:
\begin{equation}
    \begin{aligned}
        \mathbb{E}_g(f_\tau, f_\tau^*) &= \sum_{\rho} \lambda_\rho \left<(w_\rho - w_\rho^*)^2\right>_{x \in D_{\tau}}, \\
        & = (w - w^*)\Lambda(w - w^*),
    \end{aligned}
    \label{Intrinsic_Generalization_in_PEFT_CL_2}
\end{equation}
where $\Lambda = \lambda_\rho \delta_{\rho \gamma}, \rho = \gamma$. Here, \(w\) and \(w^*\) denote matrices composed of weights corresponding to the orthogonal eigenfunctions reconstituted in the RKHS for each output.

In an approach analogous to the solution process for NTK Dynamics discussed in Appendix~\ref{NTK_Dynamics}, we construct a kernel regression error for the weight matrix \(w\):
\begin{equation}
    \begin{aligned}
        \mathbb{E}_{w} = ||\varphi(x)^\top w - y||_2^2 + \lambda ||w||_2^2,
    \end{aligned}
    \label{Intrinsic_Generalization_in_PEFT_CL_3}
\end{equation}
where \(\varphi(x)\) represents the matrix composed of \(\varphi_\rho(x_i)\). For simplicity, we omit the subscript in a similar manner to the treatment of \(w\).

By obtaining the saddle-point solution that minimizes the kernel regression error, we arrive at:
\begin{equation}
\small
    \begin{aligned}
        w &= (\varphi(x) \varphi(x)^\top +\lambda I)^{-1} \varphi(x) y, \\
        & = (\varphi(x) \varphi(x)^\top +\lambda I)^{-1} \varphi(x) \varphi(x)^\top w^*, \\
        & = (\varphi(x) \varphi(x)^\top +\lambda I)^{-1} [(\varphi(x) \varphi(x)^\top +\lambda I) w^* - \lambda w^*], \\
        & = w^* - \lambda(\varphi(x) \varphi(x)^\top +\lambda I)^{-1} w^*.
    \end{aligned}
    \label{Intrinsic_Generalization_in_PEFT_CL_4}
\end{equation}

Substituting \(w - w^* = - \lambda(\varphi(x) \varphi(x)^\top +\lambda I)^{-1} w^*\) back into Eq. \ref{Intrinsic_Generalization_in_PEFT_CL_2}, we obtain:
\begin{equation}
    \begin{aligned}
        \mathbb{E}_g(f_\tau, f_\tau^*) = \lambda^2 \Big< w^* (\varphi(x) \varphi(x)^\top +\lambda I)^{-1} \\
        \times \Lambda (\varphi(x) \varphi(x)^\top +\lambda I)^{-1} w^*  \Big>_{x \in D_{\tau}}.
    \end{aligned}
    \label{Intrinsic_Generalization_in_PEFT_CL_5}
\end{equation}

As both \(\Lambda\) and \(w^*\) are diagonal matrices, we separate them from the non-diagonal matrix part for easier solving:
\begin{equation}
\small
    \begin{aligned}
        \mathbb{E}_g(f_\tau, f_\tau^*) &= \lambda^2 \left< w^* (\varphi(x) \varphi(x)^\top +\lambda I)^{-1} \right. \\
        &\quad \left. \times \Lambda (\varphi(x) \varphi(x)^\top +\lambda I)^{-1} w^*  \right>_{x \in D_{\tau}}, \\
        &= \left< \Lambda^{-\frac{1}{2}} w^* w^{*\top} \Lambda^{-\frac{1}{2}}\right>_{x \in D_{\tau}} \\
        &\quad \times \left< (\lambda \Lambda^{\frac{1}{2}}(\varphi(x) \varphi(x)^\top +\lambda I)^{-1} \Lambda^{\frac{1}{2}})^2 \right>_{x \in D_{\tau}}, \\
        &= \left< \Lambda^{-\frac{1}{2}} w^* w^{*\top} \Lambda^{-\frac{1}{2}}\right>_{x \in D_{\tau}} \\
        &\quad \times \left< ((\frac{1}{\lambda}O(x)O(x)^\top + \Lambda^{-1})^{-1})^2 \right>_{x \in D_{\tau}}, \\
        &= \sum_{\rho} \sum_{\gamma} \left< K_{\rho, \gamma} U^{2}_{\rho, \gamma} \right>_{x \in D_{\tau}}.
    \end{aligned}
    \label{Intrinsic_Generalization_in_PEFT_CL_6}
\end{equation}

Drawing from \cite{chai2009generalization}, we aim to determine the dynamic changes of \( U_{\rho, \gamma} \). Introducing auxiliary variable \( z \) and data quantity variable \( s \), \( U_{\rho, \gamma} \) can be represented as:
\begin{equation}
    \begin{aligned}
        U_{\rho,\gamma}(s,z) = \left(\frac{1}{\lambda}O(x)O(x)^\top + \Lambda^{-1} + zI \right)^{-1}.
    \end{aligned}
    \label{Intrinsic_Generalization_in_PEFT_CL_7}
\end{equation}

At this stage of the analysis, by applying the Woodbury Matrix Inversion Formula, we derive the following expression:
\begin{equation}
    \begin{aligned}
        & \left< U(s+1,z) \right>_{x \in D_{\tau}} = \left< \left( U(s,z)^{-1} + \frac{1}{\lambda} O(x)O(x)^\top \right)^{-1} \right>_{x \in D_{\tau}}, \\
        & = \left< U(s,z) \right>_{x \in D_{\tau}} - \left< U(s,z) O(x) \right>_{x \in D_{\tau}} \\
        & \quad + \left< (\lambda I + O(x)^\top U(s,z) O(x))^{-1}O(x)^\top U(s,z) \right>_{x \in D_{\tau}}, \\
        & = \left< U(s,z) \right>_{x \in D_{\tau}} - \left< \frac{U(s,z) O(x)O(x)^\top U(s,z)}{\lambda + O(x)^\top U(s,z) O(x)} \right>_{x \in D_{\tau}},
    \end{aligned}
    \label{Intrinsic_Generalization_in_PEFT_CL_8}
\end{equation}
For the sake of conciseness, we continue to omit the subscripts \(\rho\) and \(\gamma\) in this proof.

Confronted with the intricate condition of averaging the last term on the right-hand side, we employ an approximation method where the numerator and denominator are averaged separately. This leads to the ensuing approximation:
\begin{equation}
    \begin{aligned}
        \left< U(s+1,z) \right>_{x \in D_{\tau}} \approx \left< U(s,z) \right>_{x \in D_{\tau}} - \frac{\left< U(s,z)^2 \right>_{x \in D_{\tau}}}{\lambda + \text{Tr} \left< U(s,z) \right>_{x \in D_{\tau}}}.
    \end{aligned}
    \label{Intrinsic_Generalization_in_PEFT_CL_9}
\end{equation}

Considering \(s\) as a continuous variable, we derive the first-order dynamics of \(U\) with respect to \(s\):
\begin{equation}
    \begin{aligned}
        \nabla U(s,z)|_s = U(s+1,z) - U(s,z) \approx - \frac{\left< U(s,z)^2 \right>}{\lambda + \text{Tr} \left< U(s,z) \right>}.
    \end{aligned}
    \label{Intrinsic_Generalization_in_PEFT_CL_10}
\end{equation}

Next, revisiting Equations \ref{Intrinsic_Generalization_in_PEFT_CL_6} and \ref{Intrinsic_Generalization_in_PEFT_CL_7}, by taking the first-order derivative with respect to variable \(z\) and setting it to zero, we arrive at:
\begin{equation}
    \begin{aligned}
        \nabla U(s,z)|_{z=0} = - (\frac{1}{\lambda}O(x)O(x)^\top + \Lambda^{-1})^{-2} = -U^{2}_{\rho, \gamma}.
    \end{aligned}
    \label{Intrinsic_Generalization_in_PEFT_CL_11}
\end{equation}

Subsequently, by substituting Eq. \ref{Intrinsic_Generalization_in_PEFT_CL_11} into Eq. \ref{Intrinsic_Generalization_in_PEFT_CL_10}, we deduce:
\begin{equation}
    \begin{aligned}
        \nabla U(s,z)|_s \approx \frac{1}{\lambda + \text{Tr} \left< U(s,z) \right>} \nabla U(s,z)|_{z=0}.
    \end{aligned}
    \label{Intrinsic_Generalization_in_PEFT_CL_12}
\end{equation}

To simplify subsequent derivations, we omit variables \(s\) and \(z\) from \(U(s, z)\), yielding the following simplified expression:
\begin{equation}
    \begin{aligned}
        \frac{\partial U}{\partial s} \approx \frac{1}{\lambda + \text{Tr} \left< U \right>} \frac{\partial U}{\partial z}.
    \end{aligned}
    \label{Intrinsic_Generalization_in_PEFT_CL_13}
\end{equation}

For the given partial differential equation (PDE) in Eq. \ref{Intrinsic_Generalization_in_PEFT_CL_13}, we use the method of characteristics to solve it. This approach transforms the PDE into a set of ordinary differential equations (ODEs), describing the solution's behavior along characteristic curves. These curves are paths in the solution space along which the PDE simplifies to an ODE. For path construction, we identify the normal vector \((-1, \frac{\partial U}{\partial s}, \frac{\partial U}{\partial z})\), perpendicular to the vector \((0, 1, -\frac{1}{\lambda + \text{Tr} \left< U \right>})\) in the PDE. From PDE in Eq. \ref{Intrinsic_Generalization_in_PEFT_CL_13}, we obtain a set of ODEs:
\begin{equation}
    \begin{aligned}
         \frac{d U}{d v} = 0,\quad
         \frac{d s}{d v} = 1, \quad
         \frac{d z}{d v} = -\frac{1}{\lambda + \text{Tr} \left< U \right>},
    \end{aligned}
    \label{Intrinsic_Generalization_in_PEFT_CL_14}
\end{equation}
$v$ is an additional variable we introduce, related to the characteristic curves.

Consequently, it can be deduced that \( U \) is a constant term independent of \( v \), with \( s = v + s_0 \) and \( z = - \frac{v}{\lambda + \text{Tr} \left< U \right>} + z_0 \). Since \( s_0 = 0 \) and in conjunction with Eq. \ref{Intrinsic_Generalization_in_PEFT_CL_7}, we obtain \( U(s, z) = \left( \mathbf{\Lambda}^{-1} + z_0 \mathbf{I} \right)^{-1} = \left( \mathbf{\Lambda}^{-1} + (z + \frac{v}{\lambda + \text{Tr} \left< U \right>}) \mathbf{I} \right)^{-1} = \left( \mathbf{\Lambda}^{-1} + (z + \frac{s}{\lambda + \text{Tr} \left< U \right>}) \mathbf{I} \right)^{-1} \).

Consequently, taking into account the properties of the Dirac function, we deduce the following equations:
\begin{equation}
    \begin{aligned}
         U_{\rho, \gamma}(s, z) = \left(\frac{1}{\lambda_\rho} + z + \frac{s}{\lambda + \text{Tr} \left< U_{\rho, \gamma}(s, z) \right>}\right)^{-1},
    \end{aligned}
    \label{Intrinsic_Generalization_in_PEFT_CL_15}
\end{equation}

\begin{equation}
    \begin{aligned}
         TU(s, z) &= \text{Tr} \left< U_{\rho, \gamma}(s, z) \right> \\
         &= \text{Tr} \left( \frac{1}{\lambda_\rho} + z + \frac{s}{\lambda + TU(s, z)} \right)^{-1},
    \end{aligned}
    \label{Intrinsic_Generalization_in_PEFT_CL_16}
\end{equation}

\begin{equation}
    \begin{aligned}
         \frac{\partial U_{\rho, \gamma}(s, z)}{\partial z}\bigg|_{z=0} = & - \left(\frac{1}{\lambda_\rho} + \frac{s}{\lambda + TU(s, 0) }\right)^{-2} \\
         & \times \left(1-\frac{s}{(\lambda + TU(s,0))^2}\frac{\partial TU(s,0)}{\partial z}\right).
    \end{aligned}
    \label{Intrinsic_Generalization_in_PEFT_CL_17}
\end{equation}

Furthermore, since \( U(s, z) \) at initialization is \( U(0, z) = (\Lambda^{-1} + z \mathbf{I})^{-1} \), a diagonal matrix, and as the amount of data \( s \) increases, \( \frac{1}{\lambda}O(x)O(x)^\top \) will not change this diagonal property. Therefore, the derivative of its trace is equal to the sum of the derivatives of the original matrix.
\begin{equation}
    \begin{aligned}
         \frac{\partial TU(s, z)}{\partial z}\bigg|_{z=0} & = \sum\limits_{\rho} \frac{\partial U_{\rho, \gamma}(s, z)}{\partial z}\bigg|_{z=0}, \\
         & = - \sum\limits_{\rho}\left(\frac{1}{\lambda_\rho} + \frac{s}{\lambda + TU(s, 0) }\right)^{-2} \\
         & \quad \times \left(1-\frac{s}{(\lambda + TU(s,0))^2}\frac{\partial TU(s,0)}{\partial z}\right).
    \end{aligned}
    \label{Intrinsic_Generalization_in_PEFT_CL_18}
\end{equation}

From the above formula derivation, we can conclude:
\begin{equation}
    \begin{aligned}
         \frac{\partial TU(s, 0)}{\partial z} = \frac{m}{\frac{ms}{(\lambda + TU(s, 0))^2} -1}.
    \end{aligned}
    \label{Intrinsic_Generalization_in_PEFT_CL_19}
\end{equation}

\begin{equation}
    \begin{aligned}
         \frac{\partial U_{\rho, \gamma}(s, z)}{\partial z} = & - \left(\frac{1}{\lambda_\rho} + \frac{s}{\lambda + TU(s, 0) }\right)^{-2} \\
         & \times (1 - \frac{ms}{(\lambda + TU(s, 0))^2})^{-1}.
    \end{aligned}
    \label{Intrinsic_Generalization_in_PEFT_CL_20}
\end{equation}
where $m = \sum\limits_{\rho}\left(\frac{1}{\lambda_\rho} + \frac{s}{\lambda + TU(s, 0) }\right)^{-2}$.

Therefore, combining Eq. \ref{Intrinsic_Generalization_in_PEFT_CL_7}, the final generalization gap in this task can be represented as:
\begin{equation}
    \begin{aligned}
        \mathbb{E}_g & = \sum\limits_{\rho, \gamma} K_{\rho, \gamma} U^{2}_{\rho, \gamma} = - \sum\limits_{\rho}\frac{w_\rho^{*2}}{\lambda_\rho} \frac{\partial U_{\rho}(s, z)}{\partial z}\bigg|_{z=0}, \\
        & = \sum\limits_{\rho}\frac{w_\rho^{*2}}{\lambda_\rho} \left(\frac{1}{\lambda_\rho} + \frac{s}{\lambda + TU(s, 0) }\right)^{-2} \\
        & \quad \times (1 - \frac{ms}{(\lambda + TU(s, 0))^2})^{-1}, \\
        & = \sum\limits_{\rho}\frac{w_\rho^{*2}}{\lambda_\rho} \left(\frac{1}{\lambda_\rho} + \frac{s}{\lambda + TU(s) }\right)^{-2} \\
        & \quad \times (1 - \frac{ms}{(\lambda + TU(s))^2})^{-1}.
    \end{aligned}
    \label{Intrinsic_Generalization_in_PEFT_CL_21}
\end{equation}

Further, it finally can be transformed into
\begin{equation}
    \mathbb{E}_g = \sum\limits_{\rho, i}\frac{w_\rho^{*2}}{\lambda_\rho} \left(\frac{1}{\lambda_\rho} + \frac{s_i}{\lambda + tu_i}\right)^{-2} (1 - \frac{m_i s_i}{(\lambda + tu_i)^2})^{-1},
    \label{Intrinsic_Generalization_in_PEFT_CL_22}
\end{equation}
Here, the variable \(s_i\) indicates the sample size for \(i = 1, 2, \ldots, n_\tau\). The parameters \(m_i\) and \(tu_i\) are derived from the established relationships:
\begin{equation}
\small
    m_i = \sum_{\rho, i} (\frac{1}{\lambda_\rho} + \frac{s_i}{\lambda + m_i})^{-1}, \quad tu_i = \sum_{\rho, i} (\frac{1}{\lambda_\rho} + \frac{s_i}{\lambda + m_i})^{-2}.
\end{equation}

\section{Datasets and Experimental Configurations}
\label{appendix:datasets}
\noindent\textbf{Datasets:} Specifically, we utilize the CIFAR-100 dataset \cite{krizhevsky2009learning}, which consists of 60,000 32x32 color images distributed across 100 classes. To align with the input requirements of the pre-trained ViT model, the images are resized to 224x224 pixels and organized into 10 tasks, each comprising 10 classes. Additionally, the ImageNet-R dataset \cite{wang2022dualprompt} is employed, which extends the original ImageNet by incorporating artistic renditions, cartoons, and stylized interpretations for 200 classes, structured into 10 tasks with 20 classes each, featuring 24,000 training and 6,000 test images. The ImageNet-A dataset \cite{hendrycks2021natural} further evaluates the generalization of models against adversarial and out-of-distribution samples, consisting of 7,500 images from 200 classes, partitioned into 10 tasks. The DomainNet dataset \cite{peng2019moment}, a large-scale domain adaptation resource, is also utilized. It comprises six distinct domains—Clipart, Infograph, Painting, Quickdraw, Real, and Sketch—totaling 423,506 images across 345 categories. These are organized into 15 tasks, each containing 23 classes, to thoroughly test cross-domain generalization. Unlike prior studies, such as DAP \cite{jung2023generating}, which focuses on the Real domain, and CODA-Prompt \cite{smith2023coda}, which examines a limited five-task sequence within the Real domain, our study encompasses all six domains in a structured 15-task sequence. This approach establishes a more comprehensive benchmark for the continual domain adaptation.

Furthermore, we incorporate additional datasets, including Oxford Pets \cite{parkhi2012cats}, EuroSAT \cite{helber2018introducing}, PlantVillage \cite{hughes2015open}, VTAB \cite{zhai2019large}, and Kvasir \cite{pogorelov2017kvasir}, as detailed in Table~\ref{table:dataset_summary}. This extensive dataset selection underscores the robustness, generalization, and adaptability of our framework across a wide range of visual recognition tasks, thereby validating its efficacy in addressing domain-specific challenges.

\noindent\textbf{Training Details:} Experiments are conducted on NVIDIA RTX 4090 GPUs, with all methods implemented in PyTorch, consistent with the protocols in \cite{wang2022learning}. We utilize two configurations of the ViT: \textit{ViT-B/16-IN21K} and \textit{ViT-B/16-IN1K}, with the latter being fine-tuned on ImageNet-1K, as our foundational models. In our NTK-CL setup, the SGD optimizer is used for training across 20 epochs with a batch size of 16. The learning rate starts at 0.01, adjusting via cosine annealing to promote optimal convergence.

\noindent\textbf{Evaluation Metrics:} Following the established benchmark protocol in \cite{rebuffi2017icarl}, we evaluate the model's effectiveness using $A_\tau$, which signifies the accuracy post the $\tau$-th training stage. Notably, we employ $A_T$—the performance metric at the termination of the final stage—and $\bar{A} = \frac{1}{T} \sum_{\tau=1}^T A_\tau$, which calculates the average accuracy over all incremental stages. These metrics are selected as the principal measures of model performance, providing a holistic view of its efficacy and stability throughout the training process.

\section{Task Segmentation} \label{Task_Segmentation}
In Tables \ref{tab:cifar100_class_order} and \ref{tab:imagenet_class_order}, we outline the class order for CIFAR100, ImageNet-R, and ImageNet-A for each seed configuration. All subsequent task segmentations adhere to these class orders. The method to establish this class order involves setting the random seed and executing a random permutation of the class indices during task segmentation definition. The following code snippet illustrates this process:
\begin{center}
\begin{tcolorbox}[colback=gray!10!white, colframe=gray!80!black, title=Code Snippet]
\footnotesize
\begin{verbatim}
import numpy as np
np.random.seed(seed)
order = len(all_categories)
order = np.random.permutation(order).tolist()
\end{verbatim}
\end{tcolorbox}
\end{center}

All remaining datasets are divided in this manner to maintain consistency and replicability across experiments.

\begin{table*}[ht]
    \centering
    \caption{The class order for each seed on CIFAR100 determines all subsequent task segmentations.}
    \label{tab:cifar100_class_order}
    \begin{tabular}{@{}c|>{\centering\arraybackslash}m{14cm}@{}}
        \toprule
        \textbf{Seed} & \textbf{Class Order} \\ 
        \midrule
        seed0 & \scriptsize [26, 86, 2, 55, 75, 93, 16, 73, 54, 95, 53, 92, 78, 13, 7, 30, 22, 24, 33, 8, 43, 62, 3, 71, 45, 48, 6, 99, 82, 76, 60, 80, 90, 68, 51, 27, 18, 56, 63, 74, 1, 61, 42, 41, 4, 15, 17, 40, 38, 5, 91, 59, 0, 34, 28, 50, 11, 35, 23, 52, 10, 31, 66, 57, 79, 85, 32, 84, 14, 89, 19, 29, 49, 97, 98, 69, 20, 94, 72, 77, 25, 37, 81, 46, 39, 65, 58, 12, 88, 70, 87, 36, 21, 83, 9, 96, 67, 64, 47, 44] \\ 
        \midrule
        seed1 & \scriptsize [80, 84, 33, 81, 93, 17, 36, 82, 69, 65, 92, 39, 56, 52, 51, 32, 31, 44, 78, 10, 2, 73, 97, 62, 19, 35, 94, 27, 46, 38, 67, 99, 54, 95, 88, 40, 48, 59, 23, 34, 86, 53, 77, 15, 83, 41, 45, 91, 26, 98, 43, 55, 24, 4, 58, 49, 21, 87, 3, 74, 30, 66, 70, 42, 47, 89, 8, 60, 0, 90, 57, 22, 61, 63, 7, 96, 13, 68, 85, 14, 29, 28, 11, 18, 20, 50, 25, 6, 71, 76, 1, 16, 64, 79, 5, 75, 9, 72, 12, 37] \\ 
        \midrule
        seed2 & \scriptsize [83, 30, 56, 24, 16, 23, 2, 27, 28, 13, 99, 92, 76, 14, 0, 21, 3, 29, 61, 79, 35, 11, 84, 44, 73, 5, 25, 77, 74, 62, 65, 1, 18, 48, 36, 78, 6, 89, 91, 10, 12, 53, 87, 54, 95, 32, 19, 26, 60, 55, 9, 96, 17, 59, 57, 41, 64, 45, 97, 8, 71, 94, 90, 98, 86, 80, 50, 52, 66, 88, 70, 46, 68, 69, 81, 58, 33, 38, 51, 42, 4, 67, 39, 37, 20, 31, 63, 47, 85, 93, 49, 34, 7, 75, 82, 43, 22, 72, 15, 40] \\ 
        \midrule
        seed3 & \scriptsize [93, 67, 6, 64, 96, 83, 98, 42, 25, 15, 77, 9, 71, 97, 34, 75, 82, 23, 59, 45, 73, 12, 8, 4, 79, 86, 17, 65, 47, 50, 30, 5, 13, 31, 88, 11, 58, 85, 32, 40, 16, 27, 35, 36, 92, 90, 78, 76, 68, 46, 53, 70, 80, 61, 18, 91, 57, 95, 54, 55, 28, 52, 84, 89, 49, 87, 37, 48, 33, 43, 7, 62, 99, 29, 69, 51, 1, 60, 63, 2, 66, 22, 81, 26, 14, 39, 44, 20, 38, 94, 10, 41, 74, 19, 21, 0, 72, 56, 3, 24] \\ 
        \midrule
        seed4 & \scriptsize [20, 10, 96, 16, 63, 24, 53, 97, 41, 47, 43, 2, 95, 26, 13, 37, 14, 29, 35, 54, 80, 4, 81, 76, 85, 60, 5, 70, 71, 19, 65, 62, 27, 75, 61, 78, 18, 88, 7, 39, 6, 77, 11, 59, 22, 94, 23, 12, 92, 25, 83, 48, 17, 68, 31, 34, 15, 51, 86, 82, 28, 64, 67, 33, 45, 42, 40, 32, 91, 74, 49, 8, 30, 99, 66, 56, 84, 73, 79, 21, 89, 0, 3, 52, 38, 44, 93, 36, 57, 90, 98, 58, 9, 50, 72, 87, 1, 69, 55, 46] \\ 
        \bottomrule
    \end{tabular}
\end{table*}

\begin{table*}[ht]
    \centering
    \caption{The class order for each seed on ImageNet-A and ImageNet-R determines all subsequent task segmentations.}
    \label{tab:imagenet_class_order}
    \begin{tabular}{@{}c|>{\centering\arraybackslash}m{14cm}@{}}
        \toprule
        \textbf{Seed} & \textbf{Class Order} \\ 
        \midrule
        seed0 & \scriptsize [18, 170, 107, 98, 177, 182, 5, 146, 12, 152, 61, 125, 180, 154, 80, 7, 33, 130, 37, 74, 183, 145, 45, 159, 60, 123, 179, 185, 122, 44, 16, 55, 150, 111, 22, 189, 129, 4, 83, 106, 134, 66, 26, 113, 168, 63, 8, 75, 118, 143, 71, 124, 184, 97, 149, 24, 30, 160, 40, 56, 131, 96, 181, 19, 153, 92, 54, 163, 51, 86, 139, 90, 137, 101, 144, 89, 109, 14, 27, 141, 187, 46, 138, 195, 108, 62, 2, 59, 136, 197, 43, 10, 194, 73, 196, 178, 175, 126, 93, 112, 158, 191, 50, 0, 94, 110, 95, 64, 167, 41, 69, 49, 48, 85, 13, 161, 23, 186, 135, 20, 15, 78, 104, 52, 100, 76, 3, 116, 164, 198, 6, 68, 84, 121, 155, 171, 156, 91, 199, 11, 119, 102, 35, 57, 65, 1, 120, 162, 42, 105, 132, 173, 17, 38, 133, 53, 157, 128, 34, 28, 114, 151, 31, 166, 127, 176, 32, 142, 169, 147, 29, 99, 82, 79, 115, 148, 193, 72, 77, 25, 165, 81, 188, 174, 190, 39, 58, 140, 88, 70, 87, 36, 21, 9, 103, 67, 192, 117, 47, 172] \\ 
        \midrule
        seed1 & \scriptsize [58, 40, 34, 102, 184, 198, 95, 4, 29, 168, 171, 18, 11, 89, 110, 118, 159, 35, 136, 59, 51, 16, 44, 94, 31, 162, 38, 28, 193, 27, 47, 165, 194, 177, 176, 97, 174, 73, 69, 172, 108, 107, 189, 14, 56, 19, 114, 39, 185, 124, 98, 123, 119, 53, 33, 179, 181, 106, 199, 138, 116, 67, 78, 42, 17, 5, 127, 105, 48, 66, 54, 84, 183, 158, 166, 113, 12, 117, 93, 120, 154, 90, 81, 122, 191, 13, 82, 132, 187, 45, 99, 36, 161, 186, 153, 103, 195, 197, 148, 173, 75, 21, 91, 152, 2, 70, 85, 150, 6, 112, 0, 155, 77, 65, 55, 167, 88, 130, 46, 62, 74, 92, 147, 160, 143, 87, 180, 145, 164, 10, 32, 83, 182, 100, 125, 23, 126, 9, 170, 104, 151, 135, 111, 188, 64, 15, 41, 163, 109, 80, 52, 26, 76, 43, 24, 3, 169, 49, 149, 131, 190, 30, 121, 115, 175, 8, 60, 128, 1, 57, 22, 61, 63, 7, 196, 141, 86, 96, 68, 50, 142, 157, 156, 139, 146, 101, 20, 178, 25, 134, 71, 129, 144, 192, 79, 133, 137, 72, 140, 37] \\ 
        \midrule
        seed2 & \scriptsize [112, 29, 182, 199, 193, 85, 10, 54, 115, 35, 12, 92, 13, 126, 174, 2, 44, 3, 113, 14, 23, 25, 6, 134, 165, 173, 45, 65, 48, 122, 178, 64, 9, 57, 78, 71, 128, 176, 131, 53, 137, 163, 111, 123, 109, 141, 41, 130, 140, 5, 159, 100, 11, 187, 24, 89, 66, 8, 172, 175, 28, 133, 94, 42, 169, 82, 184, 106, 108, 143, 180, 166, 146, 79, 1, 119, 192, 149, 160, 188, 147, 36, 171, 179, 62, 0, 27, 157, 98, 118, 20, 158, 156, 142, 77, 30, 154, 17, 59, 181, 114, 127, 139, 191, 93, 151, 21, 55, 16, 152, 91, 99, 120, 197, 74, 190, 161, 144, 196, 87, 90, 84, 18, 97, 101, 125, 164, 135, 61, 81, 68, 129, 56, 19, 86, 70, 60, 34, 40, 138, 76, 153, 26, 32, 195, 96, 83, 110, 105, 73, 117, 150, 145, 155, 198, 136, 39, 49, 186, 132, 50, 52, 80, 185, 121, 189, 46, 88, 69, 67, 183, 58, 33, 38, 103, 51, 107, 170, 4, 102, 167, 37, 116, 124, 148, 31, 63, 47, 194, 95, 177, 162, 7, 104, 75, 43, 22, 72, 15, 168] \\ 
        \midrule
        seed3 & \scriptsize [40, 51, 139, 197, 170, 82, 183, 46, 70, 100, 179, 83, 25, 190, 159, 173, 95, 3, 41, 58, 14, 143, 12, 6, 182, 161, 128, 122, 101, 86, 64, 47, 158, 34, 38, 196, 4, 72, 67, 145, 156, 115, 155, 15, 61, 175, 120, 130, 23, 153, 31, 103, 89, 132, 109, 126, 17, 30, 178, 162, 77, 73, 71, 78, 42, 133, 192, 13, 146, 74, 5, 114, 102, 181, 121, 168, 171, 24, 144, 92, 8, 53, 27, 105, 118, 163, 43, 57, 165, 22, 180, 187, 160, 87, 134, 63, 140, 193, 135, 45, 35, 65, 50, 125, 98, 16, 19, 108, 44, 68, 76, 141, 112, 10, 84, 11, 55, 88, 176, 111, 136, 9, 137, 32, 29, 39, 185, 56, 186, 194, 91, 59, 174, 36, 177, 52, 191, 48, 96, 75, 151, 80, 99, 124, 154, 117, 85, 1, 113, 164, 116, 18, 195, 54, 188, 28, 127, 189, 49, 94, 20, 37, 79, 123, 33, 7, 62, 198, 199, 157, 97, 110, 104, 69, 90, 129, 60, 2, 66, 150, 81, 26, 142, 167, 93, 172, 148, 166, 119, 149, 138, 169, 107, 147, 21, 0, 184, 131, 152, 106] \\ 
        \midrule
        seed4 & \scriptsize [11, 99, 128, 175, 1, 111, 90, 177, 88, 187, 61, 199, 191, 123, 184, 188, 33, 171, 138, 84, 81, 102, 147, 34, 47, 124, 112, 6, 14, 190, 80, 18, 167, 45, 153, 119, 100, 83, 181, 71, 26, 134, 180, 158, 189, 89, 48, 116, 12, 69, 110, 154, 16, 19, 2, 143, 185, 29, 155, 24, 77, 127, 5, 118, 113, 25, 163, 37, 91, 28, 92, 186, 148, 82, 76, 101, 41, 157, 140, 105, 20, 74, 120, 65, 170, 35, 130, 168, 42, 46, 173, 64, 93, 182, 121, 144, 63, 7, 10, 176, 13, 15, 86, 43, 60, 97, 27, 17, 106, 108, 150, 162, 141, 67, 135, 196, 70, 133, 39, 4, 165, 142, 146, 62, 68, 53, 192, 9, 78, 40, 31, 139, 198, 169, 132, 96, 54, 125, 72, 8, 51, 107, 59, 36, 79, 85, 152, 172, 23, 75, 22, 159, 151, 73, 145, 193, 95, 98, 115, 114, 3, 156, 179, 32, 161, 160, 194, 66, 49, 136, 30, 117, 56, 166, 149, 21, 0, 131, 52, 126, 38, 44, 178, 164, 195, 57, 197, 55, 94, 109, 103, 58, 137, 50, 87, 104, 129, 183, 174, 122] \\ 
        \bottomrule
    \end{tabular}
\end{table*}

\section{Platonic Representation in PEFT-CL}
\label{PR_Explaination}
\begin{figure}[t]
\centering
\includegraphics[width=0.5\textwidth]{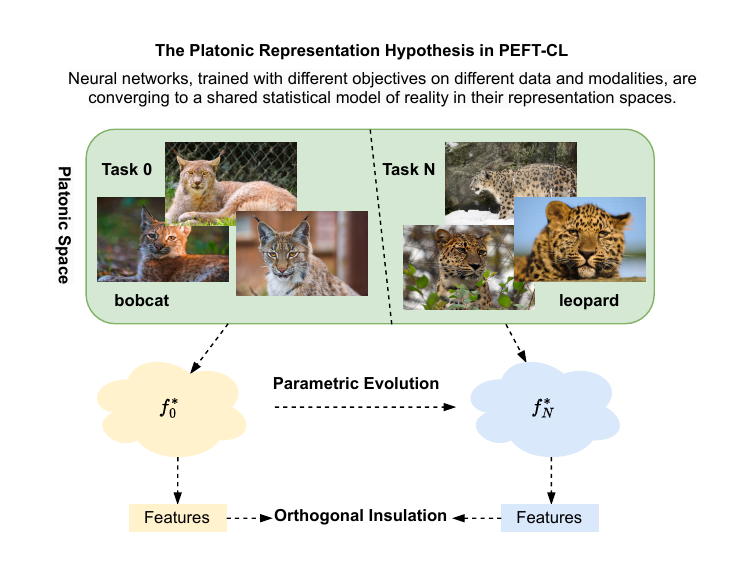}
\caption{An explanation of the contradiction between highly similar classes across different tasks and the insulation of task-level feature orthogonality.}
\label{Platonic_Representation}
\end{figure}

Researchers often question if ensuring orthogonality between features of different tasks might render the knowledge from previous tasks irrelevant, particularly when classes across tasks closely resemble each other. However, this perspective can be one-sided. Drawing on insights from \cite{huh2024platonic}, it is suggested that parameter spaces formed by different modalities and models tend to converge after extensive training—a concept we extend into the PEFT-CL context, illustrated in Fig. \ref{Platonic_Representation}. This aligns with the principles of the Neural Tangent Kernel Regime, where \(\Phi^*(X_\tau, X_k) = \Phi_0(X_\tau, X_k) = \Phi_1(X_\tau, X_k) = \cdots = \Phi_\infty(X_\tau, X_k)\). For similar classes, while they remain highly similar in Platonic Space, the mapping to a lower-dimensional space through varying subnetwork component parameters over different periods ensures their distinction without compromising the transfer and preservation of knowledge in the Platonic Space.

\section{Pre-trained Weight Matters} \label{pre_trained_model}
To rigorously assess the indispensability of pre-trained weight within our NTK-CL framework, we conduct systematic ablation studies on CIFAR100 dataset. As shown in Table~\ref{without_pretrain_weight}, the framework achieves anticipated performance enhancements only when initialized with pre-trained weight. Without this weight, adding subnetworks does not result in commensurate improvements. This evidence robustly supports the critical role of the pre-trained weight \( f_0^*(x) \) in  our NTK-CL framework, as described in Eq.~\ref{f_NTK}.


\begin{table*}[h!]
    \centering
    \caption{Evolution of incremental top-1 accuracy (\%) on CIFAR100 during full fine-tuning, comparing the original ViT-B/16 model with an enhanced variant incorporating three auxiliary subnetworks.}
    \resizebox{\textwidth}{!}{%
    \begin{tabular}{@{} c c *{10}{c} @{}}
        \toprule
        \multicolumn{2}{c}{\textbf{Network}} & \multicolumn{10}{c}{\textbf{Incremental Top-1 Accuracy (\%)}} \\
        \cmidrule(r){1-2} \cmidrule(l){3-12}
        \textbf{Name} & \textbf{Parameter} & \textbf{Task 1} & \textbf{Task 2} & \textbf{Task 3} & \textbf{Task 4} & \textbf{Task 5} & \textbf{Task 6} & \textbf{Task 7} & \textbf{Task 8} & \textbf{Task 9} & \textbf{Task 10} \\
        \midrule 
        ViT-B16 wo/ subnetworks & 85.80M & 57.16 $\pm$ 4.66 & 34.14 $\pm$ 2.43 & 27.91 $\pm$ 3.23 & 22.87 $\pm$ 1.42 & 17.94 $\pm$ 2.32 & 15.45 $\pm$ 1.05 & 13.91 $\pm$ 1.47 & 12.13 $\pm$ 1.27 & 10.53 $\pm$ 1.53 & 10.18 $\pm$ 1.57 \\
        ViT-B16 w/ subnetworks & 93.23M & 55.06 $\pm$ 5.66 & 31.03 $\pm$ 2.44 & 22.76 $\pm$ 3.46 & 19.94 $\pm$ 1.46 & 17.22 $\pm$ 0.43 & 14.85 $\pm$ 0.80 & 13.77 $\pm$ 0.98 & 11.22 $\pm$ 0.72 & 10.13 $\pm$ 0.86 & 10.09 $\pm$ 0.91 \\
        \bottomrule
    \end{tabular}}
    \label{without_pretrain_weight}
\end{table*}

\section{More Visualizations} \label{visualization}
In this section, we first present visual information generated using the Deep Image Prior (DIP) technique \cite{ulyanov2018deep} for a pre-trained ViT, alongside S1 and S2 modules. The specific results are shown in Fig.~\ref{fig:imagenetr_grid} and Fig.~\ref{fig:imageneta_grid}. Specifically, a random image from Task 0 is used to extract three-dimensional embeddings via parameters from S1 and S2 modules. As the Hybrid Adaptation Module, which employs two-dimensional CLS Token features, does not support DIP visualization, we focus on the embeddings from S1 and S2 modules. These embeddings, serving as inputs to the Hybrid Adaptation Module, effectively demonstrate the network’s learning and memory retention. This approach provides insight into how each module processes and retains task-relevant information, showcasing the dynamic learning and generalization capabilities within our NTK-CL framework.

Subsequently, to further investigate the performance discrepancies between self-supervised and supervised pre-trained weights, and to elucidate the pronounced advantage exhibited by the CLIP model on the ImageNet-R dataset, we conduct a series of additional visualization experiments. Leveraging our NTK-CL framework, we employ t-SNE to visualize the evolution of feature distributions for samples from Task-0 across both the CIFAR-100 and ImageNet-R datasets. The visualizations, presented in Fig.~\ref{all_t_sne}, offer a comprehensive comparison of feature representations derived from models initialized with Supervised ImageNet-21K, DINO, CLIP, and MAE-1K weights. Across both datasets, we observe that self-supervised pre-trained weights generally result in feature spaces with reduced inter-class separability, particularly as the continual learning process advances. On CIFAR-100, although DINO benefits from contrastive pretraining and maintains coherent class clusters in early tasks, its subsequent performance still lags behind models initialized with supervised pretraining. In addition, MAE-1K quickly exhibits significant overlap and dispersion as tasks increase. This suggests that the representations learned by MAE, which focus on reconstructing pixel-level content, are inherently less robust to the distributional shifts introduced in the PEFT-CL setting. In contrast, both Supervised ImageNet-21K and CLIP demonstrate well-separated clusters throughout the task sequence, indicating a higher degree of feature discrimination and resilience to forgetting.

The phenomenon becomes even more pronounced on the ImageNet-R dataset, where the visual complexity and semantic abstraction inherent in the data pose additional challenges for representation learning. In this context, the generative self-supervised paradigm of MAE performs particularly poorly, with feature representations exhibiting severe degradation in class separability from the initial task onward. By comparison, DINO’s contrastive learning objective enables it to preserve moderately structured feature spaces, although it still exhibits gradual degeneration. CLIP, leveraging its large-scale pre-training on aligned image-text pairs, consistently demonstrates superior feature clustering, particularly on ImageNet-R. We attribute this to CLIP's ability to capture semantically coherent and contextually enriched representations that align well with the subjective and stylistic diversity characteristic of ImageNet-R images, including artistic renderings, sketches, and abstract compositions.

These findings collectively underscore the critical limitations of current self-supervised pre-training strategies, such as Dino and MAE, in producing semantically discriminative and task-adaptive representations for PEFT-CL. Addressing these limitations represents a promising direction for future research. We will further consider prompt-conditioned encoding or task-adaptive masking as potential avenues to enhance class separability and mitigate catastrophic forgetting for self-supervised schemes in PEFT-CL.

\begin{figure*}[t]
    \centering
    \begin{tabular}{@{}c@{}c@{}c@{}c@{}c@{}c@{}c@{}c@{}c@{}c@{}c@{}c@{}}
        & \textbf{Original} & \textbf{Task-0} & \textbf{Task-1} & \textbf{Task-2} & \textbf{Task-3} & \textbf{Task-4} & \textbf{Task-5} & \textbf{Task-6} & \textbf{Task-7} & \textbf{Task-8} & \textbf{Task-9} \\
        \rotatebox{90}{\parbox{1.5cm}{\centering \textbf{ViT-21K}}} & 
        \includegraphics[width=0.085\textwidth]{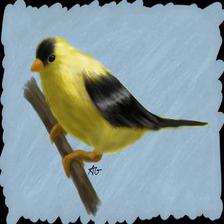} &
        \includegraphics[width=0.085\textwidth]{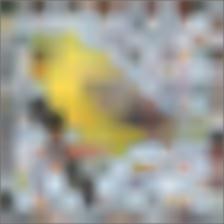} &
        \includegraphics[width=0.085\textwidth]{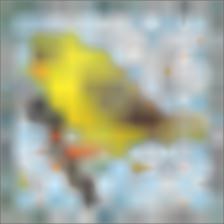} &
        \includegraphics[width=0.085\textwidth]{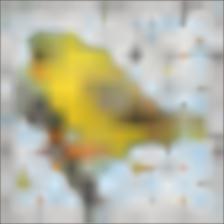} &
        \includegraphics[width=0.085\textwidth]{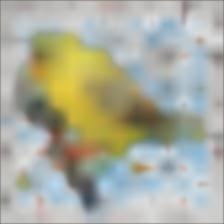} &
        \includegraphics[width=0.085\textwidth]{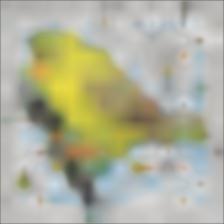} &
        \includegraphics[width=0.085\textwidth]{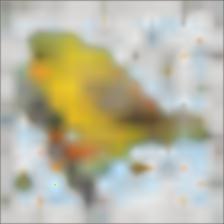} &
        \includegraphics[width=0.085\textwidth]{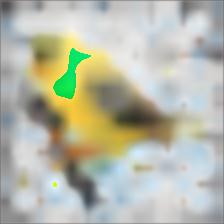} &
        \includegraphics[width=0.085\textwidth]{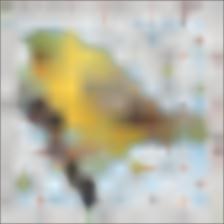} &
        \includegraphics[width=0.085\textwidth]{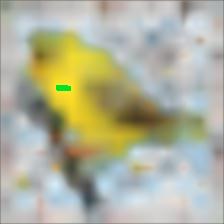} &
        \includegraphics[width=0.085\textwidth]{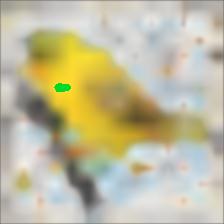} \\

        \rotatebox{90}{\parbox{1.5cm}{\centering \textbf{S1}}} & 
        \includegraphics[width=0.085\textwidth]{Figures/DIP/ImageNet-R/imagenet-r.jpg} &
        \includegraphics[width=0.085\textwidth]{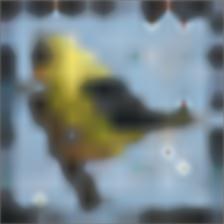} &
        \includegraphics[width=0.085\textwidth]{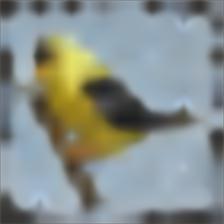} &
        \includegraphics[width=0.085\textwidth]{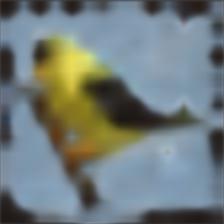} &
        \includegraphics[width=0.085\textwidth]{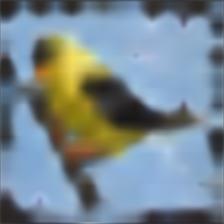} &
        \includegraphics[width=0.085\textwidth]{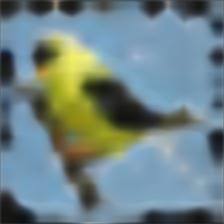} &
        \includegraphics[width=0.085\textwidth]{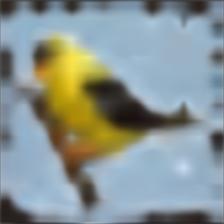} &
        \includegraphics[width=0.085\textwidth]{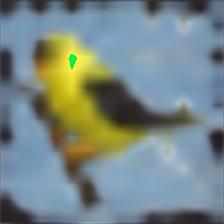} &
        \includegraphics[width=0.085\textwidth]{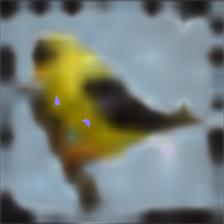} &
        \includegraphics[width=0.085\textwidth]{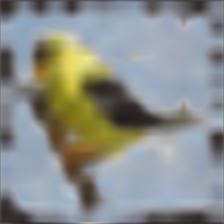} &
        \includegraphics[width=0.085\textwidth]{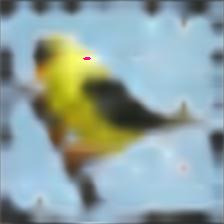} \\
        
        \rotatebox{90}{\parbox{1.5cm}{\centering \textbf{S2}}} &
        \includegraphics[width=0.085\textwidth]{Figures/DIP/ImageNet-R/imagenet-r.jpg} &
        \includegraphics[width=0.085\textwidth]{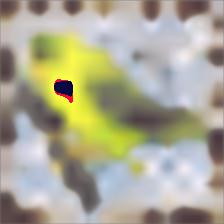} &
        \includegraphics[width=0.085\textwidth]{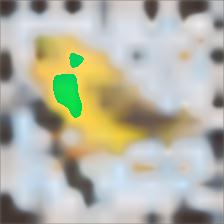} &
        \includegraphics[width=0.085\textwidth]{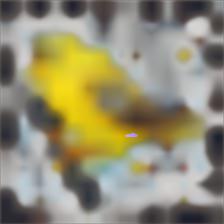} &
        \includegraphics[width=0.085\textwidth]{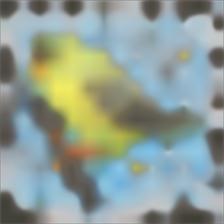} &
        \includegraphics[width=0.085\textwidth]{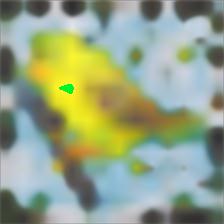} &
        \includegraphics[width=0.085\textwidth]{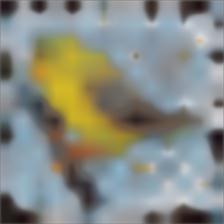} &
        \includegraphics[width=0.085\textwidth]{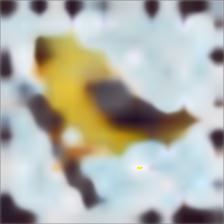} &
        \includegraphics[width=0.085\textwidth]{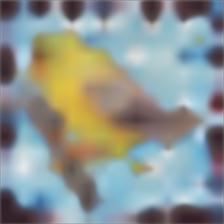} &
        \includegraphics[width=0.085\textwidth]{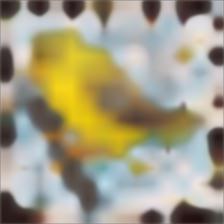} &
        \includegraphics[width=0.085\textwidth]{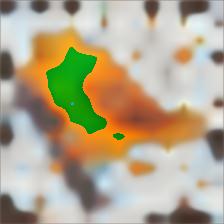} \\
    \end{tabular}
    \caption{The illustration showcases DIP visualizations for the painted serinus canaria in ImageNet-R. The first row features images generated at each task period using embeddings from the ImageNet-21K pre-trained model. The second and third rows display images produced by embeddings from the Subnetwork-1 (S1) Adaptation Module and the Subnetwork-2 (S2) Adaptation Module, respectively.}
    \label{fig:imagenetr_grid}
\end{figure*}

\begin{figure*}[t]
    \centering
    \begin{tabular}{@{}c@{}c@{}c@{}c@{}c@{}c@{}c@{}c@{}c@{}c@{}c@{}c@{}}
        & \textbf{Original} & \textbf{Task-0} & \textbf{Task-1} & \textbf{Task-2} & \textbf{Task-3} & \textbf{Task-4} & \textbf{Task-5} & \textbf{Task-6} & \textbf{Task-7} & \textbf{Task-8} & \textbf{Task-9} \\
        \rotatebox{90}{\parbox{1.5cm}{\centering \textbf{ViT-21K}}} & 
        \includegraphics[width=0.085\textwidth]{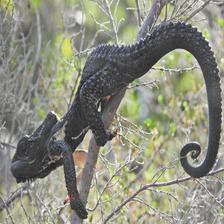} &
        \includegraphics[width=0.085\textwidth]{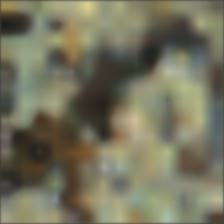} &
        \includegraphics[width=0.085\textwidth]{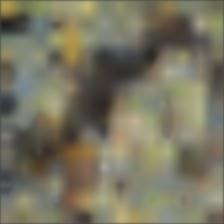} &
        \includegraphics[width=0.085\textwidth]{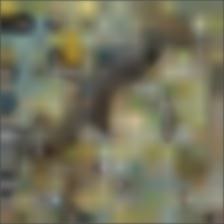} &
        \includegraphics[width=0.085\textwidth]{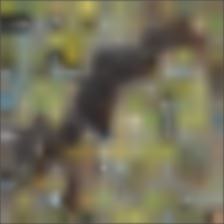} &
        \includegraphics[width=0.085\textwidth]{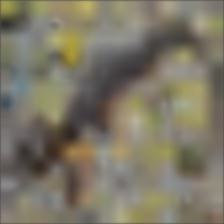} &
        \includegraphics[width=0.085\textwidth]{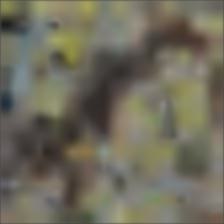} &
        \includegraphics[width=0.085\textwidth]{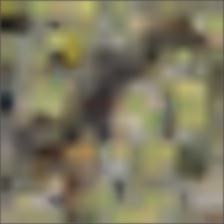} &
        \includegraphics[width=0.085\textwidth]{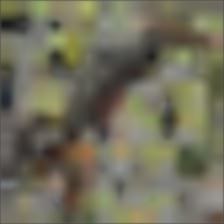} &
        \includegraphics[width=0.085\textwidth]{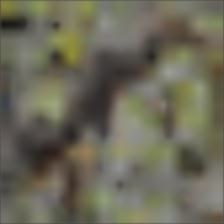} &
        \includegraphics[width=0.085\textwidth]{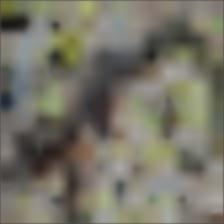} \\

        \rotatebox{90}{\parbox{1.5cm}{\centering \textbf{S1}}} & 
        \includegraphics[width=0.085\textwidth]{Figures/DIP/ImageNet-A/imagenet-a.jpg} &
        \includegraphics[width=0.085\textwidth]{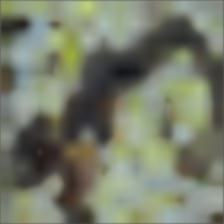} &
        \includegraphics[width=0.085\textwidth]{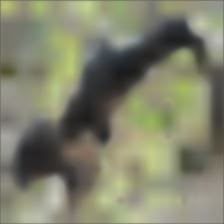} &
        \includegraphics[width=0.085\textwidth]{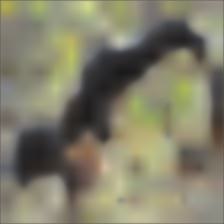} &
        \includegraphics[width=0.085\textwidth]{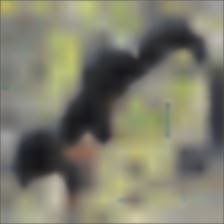} &
        \includegraphics[width=0.085\textwidth]{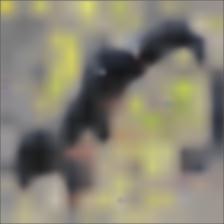} &
        \includegraphics[width=0.085\textwidth]{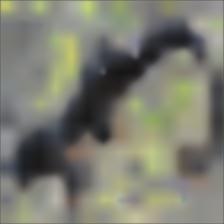} &
        \includegraphics[width=0.085\textwidth]{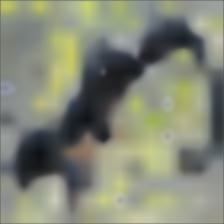} &
        \includegraphics[width=0.085\textwidth]{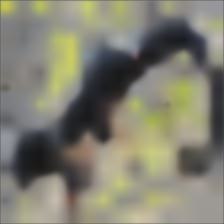} &
        \includegraphics[width=0.085\textwidth]{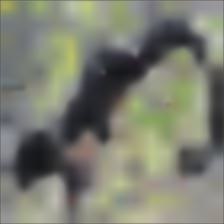} &
        \includegraphics[width=0.085\textwidth]{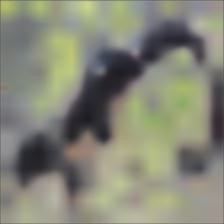} \\
        
        \rotatebox{90}{\parbox{1.5cm}{\centering \textbf{S2}}} & 
        \includegraphics[width=0.085\textwidth]{Figures/DIP/ImageNet-A/imagenet-a.jpg} &
        \includegraphics[width=0.085\textwidth]{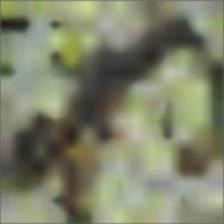} &
        \includegraphics[width=0.085\textwidth]{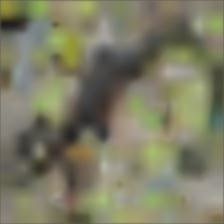} &
        \includegraphics[width=0.085\textwidth]{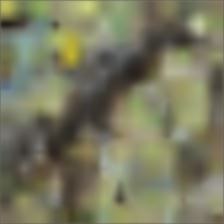} &
        \includegraphics[width=0.085\textwidth]{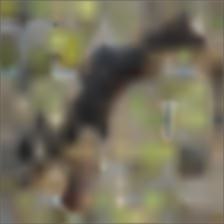} &
        \includegraphics[width=0.085\textwidth]{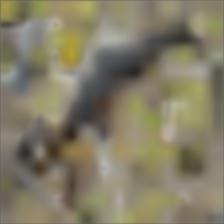} &
        \includegraphics[width=0.085\textwidth]{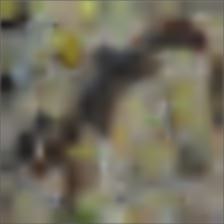} &
        \includegraphics[width=0.085\textwidth]{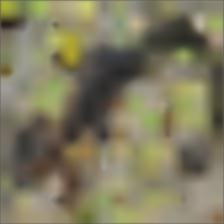} &
        \includegraphics[width=0.085\textwidth]{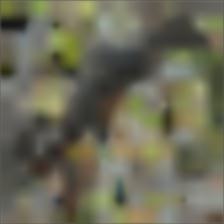} &
        \includegraphics[width=0.085\textwidth]{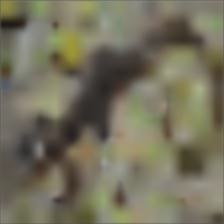} &
        \includegraphics[width=0.085\textwidth]{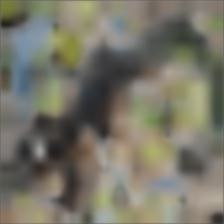} \\
    \end{tabular}
    \caption{The illustration showcases DIP visualizations for the lizard in ImageNet-A. The first row features images generated at each task period using embeddings from the ImageNet-21K pre-trained model. The second and third rows display images produced by embeddings from the Subnetwork-1 (S1) Adaptation Module and the Subnetwork-2 (S2) Adaptation Module.}
    \label{fig:imageneta_grid}
\end{figure*}

\begin{figure*}[t]
\centering
\includegraphics[width=1.0\textwidth]{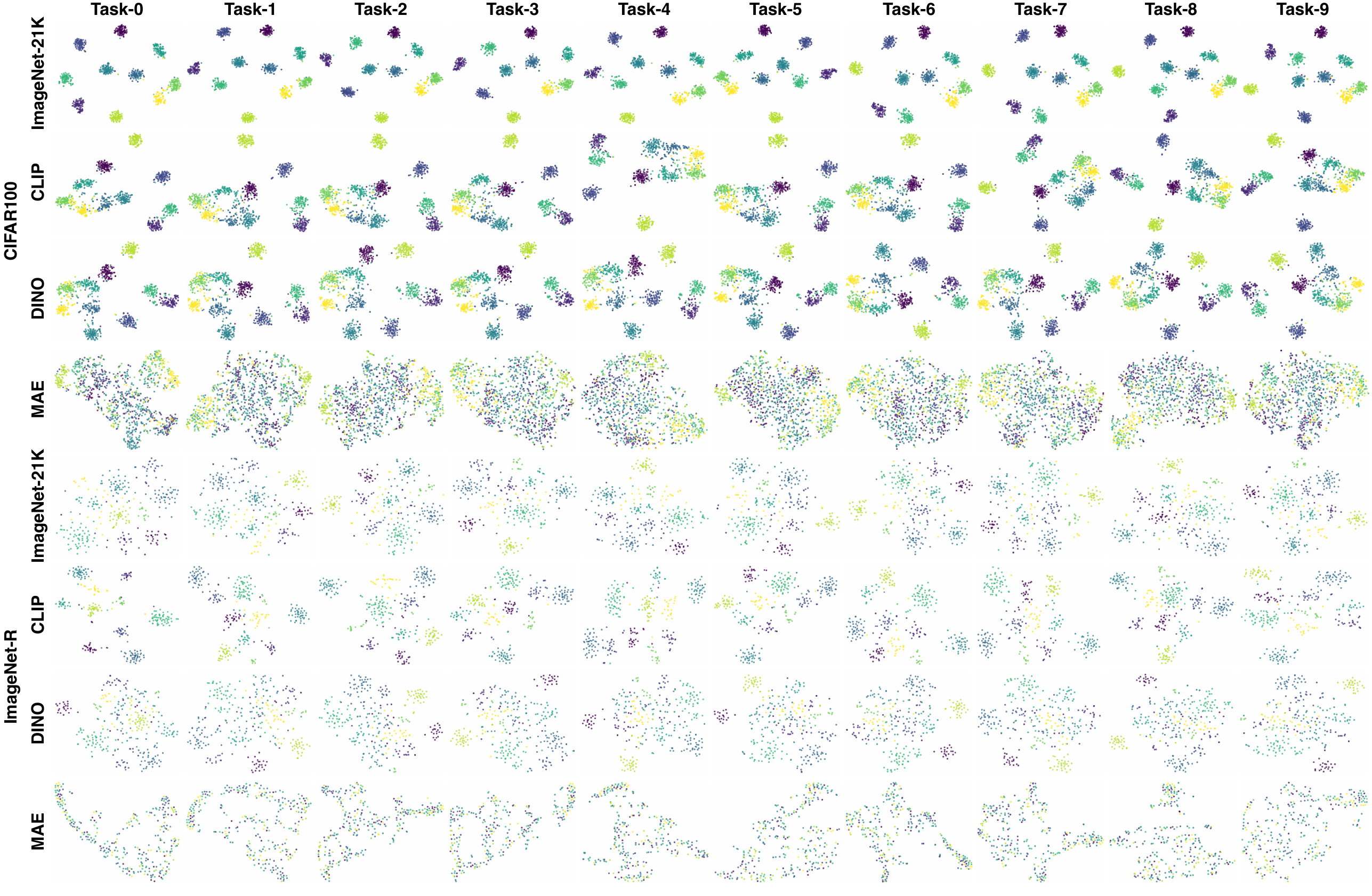}
\caption{The t-SNE visualization experiments conducted for Supervised ImageNet-21K, DINO, CLIP and MAE-1K weights on the CIFAR-100 and ImageNet-R datasets utilize images from Task-0 to investigate their performance fluctuations.}
\label{all_t_sne}
\end{figure*}

\begin{algorithm}[!t]
\caption{Bayesian optimization for AHPS.}
\label{by_code} 
\begin{tcolorbox}[colback=gray!10!white, colframe=gray!80!black, title=Code Snippet]
\footnotesize
\begin{verbatim}
    search_space = [
        Real(1e-5, 0.25, name='nce_temp'),
        Real(1e-5, 1e-2, name='dis_temp'),
        Real(1e-5, 1e-2, name='reg_temp')
    ]
    result = gp_minimize(
        lambda params: train(params, taskid),
        search_space,
        n_calls=10,
        random_state=seed
    )
    best_params = result.x
    best_acc = 1 - result.fun
    print(f"Best nce_temp: {best_params[0]}")
    print(f"Best dis_temp: {best_params[1]}")
    print(f"Best reg_temp: {best_params[2]}")
    print(f"Best accuracy: {best_acc}")
\end{verbatim}
\end{tcolorbox}
\end{algorithm}

\section{Network Architectures for Feature Fusion} \label{architectures}
In this section, we present a comprehensive overview of the network architectures associated with the diverse feature fusion methodologies detailed in Table~\ref{fusion_method_comparison}. Each method's structural intricacies are meticulously depicted in Fig.~\ref{net_structures}, providing a visual elucidation.

\begin{figure*}[t]
\centering
\includegraphics[width=1.0\textwidth]{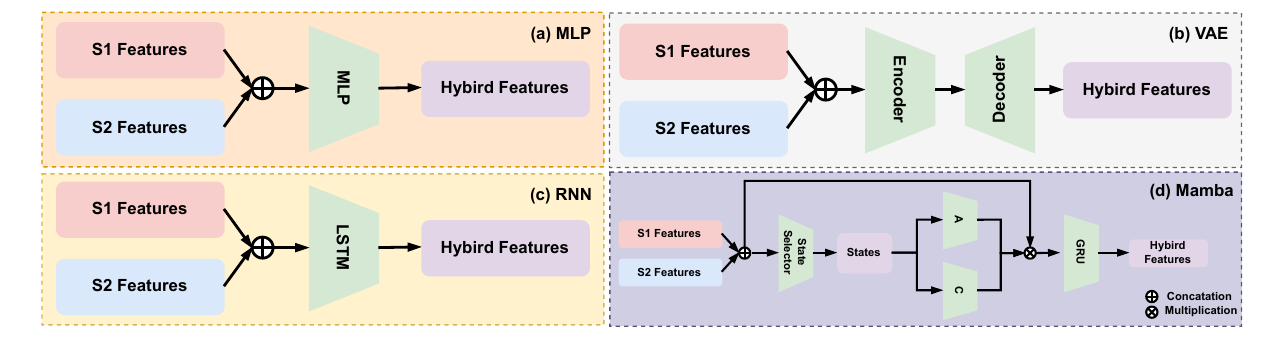}
\caption{Detailed network architectures of various feature fusion methods used in Table~\ref{fusion_method_comparison}.}
\label{net_structures}
\end{figure*}

\begin{algorithm*}[!t]
\caption{Dynamic loss scaling strategy for AHPS.}
\label{adaptive_weights}
\centering
\begin{tcolorbox}[colback=gray!10!white, colframe=gray!80!black, title=Code Snippet]
\footnotesize
\begin{tabbing}
1.\ \= \textbf{Inputs:} \= $\mathcal{L}_{dis}, \mathcal{L}_{orth}, \mathcal{L}_{reg}; \ \eta \in [\eta_{\min}, \eta_{\max}] ([0.1, 0.5]), \upsilon \in [\upsilon_{\min}, \upsilon_{\max}] ([1e-5, 1e-3]), \lambda \in [\lambda_{\min}, \lambda_{\max}] ([1e-5, 1e-3]); \ \beta = 0.95 $. \\[0.5em]

2.\ \= \textbf{Initialization:} \quad \= $\mu_{dis}, \mu_{orth}, \mu_{reg} \leftarrow 0; \ \nu_{dis}, \nu_{orth}, \nu_{reg} \leftarrow 0; \ \eta \leftarrow \eta_{\min},\ \upsilon \leftarrow \upsilon_{\min},\ \lambda \leftarrow \lambda_{\min}$. \\[0.5em]

3.\ \= \textbf{For each training iteration} $t$ \textbf{do:} \\
    \> \= Update moving averages: $\mu_{dis} \leftarrow \beta \mu_{dis} + (1-\beta) l_{dis}^{(t)};\ \nu_{dis} \leftarrow \beta \nu_{dis} + (1-\beta) \left( l_{dis}^{(t)} \right)^2$; (repeat for $\mu_{orth}, \nu_{orth}, \mu_{reg}, \nu_{reg}$). \\
    \> \= Compute standard deviations: $\sigma_{dis} \leftarrow \sqrt{ \max( \nu_{dis} - \mu_{dis}^2,\ 0 ) }$; (repeat for $\sigma_{orth}, \sigma_{reg}$). \\
    \> \= Normalize deviations: $\delta_{dis} \leftarrow \dfrac{ l_{dis}^{(t)} - \mu_{dis} }{ \sigma_{dis} }$; (repeat for $\delta_{orth}, \delta_{reg}$). \\
    \> \= Non-linear squashing: $\epsilon_{dis} \leftarrow \tanh( \delta_{dis} )$; (repeat for $\epsilon_{orth}, \epsilon_{reg}$). \\
    \> \= Compute target weights: $\eta^{target} \leftarrow \epsilon_{dis} \cdot (\eta_{\max} - \eta_{\min})$; (repeat for $\upsilon^{target}, \lambda^{target}$). \\
    \> \= EMA smoothing of weights: $\eta \leftarrow \beta \eta + (1 - \beta) \eta^{target}$; (repeat for $\upsilon, \lambda$). \\
    \> \= Clip to valid range: $\eta \leftarrow \mathrm{clip}(\eta, \eta_{\min}, \eta_{\max})$; (repeat for $\upsilon, \lambda$). \\[0.5em]

4.\ \= \textbf{Return:} $\eta, \upsilon, \lambda$.
\end{tabbing}
\end{tcolorbox}
\end{algorithm*}

\section{Hyper-parameter Search} \label{hyp_search}
In this section, we introduce two advanced methods designed to automatic hyper-parameter search (AHPS), thereby obviating the need for repetitive manual tuning and enabling dynamic self-optimization within the NTK-CL framework. The first proposed approach leverages the skopt library to facilitate an efficient and systematic exploration of the hyper-parameter space. The overall algorithmic workflow is illustrated in Algorithm~\ref{by_code}, which adheres to a meta-learning paradigm comprising a nested loop architecture. Specifically, the inner loop performs NTK-CL model training, while the outer loop employs Bayesian optimization \cite{snoek2012practical} to iteratively refine hyper-parameters based on performance feedback.

To balance computational efficiency with optimization quality, we impose a maximum of ten iterations for the outer-loop Bayesian optimization on each task. Within these iterations, the framework identifies and records the optimal incremental accuracy along with its corresponding hyper-parameter configuration. This optimal configuration is subsequently propagated and serves as the initialization for hyper-parameter selection in subsequent tasks. Although this process incurs additional computational overhead, it maintains consistency in the NTK-CL training protocol across tasks and eliminates the need for task-specific manual adjustments. Such a design ensures a principled and fully automatic hyper-parameter search that adapts to evolving task dynamics without human intervention.

Beyond global hyper-parameter search, we further propose a dynamic loss scaling strategy that enables dynamic adjustment of specific loss contributions during training. Unlike conventional approaches that rely on static, heuristically determined weighting factors, our method autonomously regulates the balance among multiple loss terms in response to the training dynamics. As depicted in Algorithm~\ref{adaptive_weights}, the proposed strategy employs an Exponential Moving Average (EMA) mechanism to continuously track the first- and second-order statistics of each loss component, including the dissimilarity loss $\mathcal{L}_{dis}$, the orthogonality loss $\mathcal{L}_{orth}$, and the regularization loss $\mathcal{L}_{reg}$. These statistics are utilized to compute normalized deviations, which are subsequently transformed via a non-linear squashing function to generate adaptive weight updates.

Specifically, the algorithm maintains exponentially smoothed estimates of the first and second moments of each loss term, denoted as \( \mu \) and \( \nu \), respectively. These statistics are used to compute the standard deviation \( \sigma \), capturing the magnitude of fluctuations in each loss component. The deviation \( \delta \) measures the normalized difference between the current loss value and its expected value, thereby quantifying its relative significance at each iteration. To mitigate the influence of outliers and ensure stability, the deviations are passed through a bounded non-linear squashing function, \( \tanh(\cdot) \). The resulting signals are linearly mapped to the predefined ranges of the balancing coefficients \( \eta, \upsilon, \lambda \), which are then updated via EMA smoothing to ensure gradual and stable transitions. The final coefficients are strictly constrained within their respective ranges to maintain interpretability and prevent oscillations. By dynamically modulating the contribution of each loss component in accordance with its statistical behavior, the proposed strategy eliminates the need for labor-intensive, dataset-specific hyper-parameter search. Extensive empirical evaluations demonstrate that our method consistently achieves stable performance and effectively balances multiple objectives across diverse datasets and tasks, thereby validating its efficacy in practical applications.

\begin{table}[!t]
    \centering
    \setlength{\tabcolsep}{1.6mm}
    \renewcommand{\arraystretch}{0.95}
    \footnotesize
    \caption{Statistics of benchmark datasets. $\mathcal{C}^{base}$: number of classes in base session. $\mathcal{C}^{inc}$: total number of classes in incremental sessions. \#Inc.: number of incremental sessions. Shots: training shots for incremental sessions. $\mathcal{N}_{base}$: number of samples in base session.}
    \resizebox{1.0\linewidth}{!}{
        \begin{tabular}{ccccccc}
            \toprule
            Dataset & $\mathcal{C}^{base}$ & $\mathcal{N}_{base}$ & $\mathcal{C}^{inc}$ & \#Inc. & Shots & Resolution \\
            \midrule
            CIFAR100 & 60 & 30000 & 40 & 8 & 5 & 224×224 \\
            \emph{mini}ImageNet & 60 & 30000 & 40 & 8 & 5 & 224×224 \\
            CUB200 & 100 & 3000 & 100 & 10 & 5 & 224×224 \\
            \bottomrule
        \end{tabular}
    }
    \label{supp:datasets}
\end{table}

\begin{table*}[h!]
    \centering
    \caption{Evolution of incremental top-1 accuracy (\%) for different datasets under the FSCIL setting, using ImageNet-21K pre-trained weights. Bold segments indicate the best results, while underlined segments denote suboptimal outcomes.}
    \resizebox{\textwidth}{!}{%
    \begin{tabular}{@{} cc *{11}{c} @{}}
        \toprule
        \multicolumn{2}{c}{\textbf{FSCIL Methods}} & \multicolumn{11}{c}{\textbf{Incremental Top-1 Accuracy (\%)}} \\
        \cmidrule(r){1-2} \cmidrule(l){3-13}
        \textbf{Name} & \textbf{Dataset} & \textbf{Task 1} & \textbf{Task 2} & \textbf{Task 3} & \textbf{Task 4} & \textbf{Task 5} & \textbf{Task 6} & \textbf{Task 7} & \textbf{Task 8} & \textbf{Task 9} & \textbf{Task 10} & \textbf{Task 11} \\
        \midrule
        CEC \cite{zhang2021few} & CIFAR100 &
        80.71 $\pm$ 0.51 & \underline{76.34 $\pm$ 0.27} & \underline{73.46 $\pm$ 0.19} & \underline{70.03 $\pm$ 0.40} & \underline{67.39 $\pm$ 0.31} & \underline{65.35 $\pm$ 0.23} & \underline{64.04 $\pm$ 0.20} & \underline{62.12 $\pm$ 0.39} & \underline{59.69 $\pm$ 0.36} & - & - \\
        ALICE \cite{peng2022few} & CIFAR100 &
        \underline{82.23 $\pm$ 0.35} & 73.34 $\pm$ 0.31 & 70.81 $\pm$ 0.35 & 67.33 $\pm$ 0.26 &
        65.92 $\pm$ 0.26 & 63.79 $\pm$ 0.21 & 62.37 $\pm$ 0.28 & 60.61 $\pm$ 0.28 & 58.60 $\pm$ 0.24 & - & - \\
        \rowcolor{blue!14}
        NTK-CL (Ours) & CIFAR100 &
        \textbf{94.08 $\pm$ 0.47} & \textbf{91.48 $\pm$ 0.12} & \textbf{90.54 $\pm$ 0.15} & \textbf{89.16 $\pm$ 0.25} &
        \textbf{89.14 $\pm$ 0.29} & \textbf{88.34 $\pm$ 0.23} & \textbf{88.22 $\pm$ 0.28} & \textbf{87.88 $\pm$ 0.16} & \textbf{86.30 $\pm$ 0.32} & - & - \\
        \midrule
        CEC \cite{zhang2021few} & \emph{mini}ImageNet &
        \underline{94.92 $\pm$ 0.25} & \underline{92.36 $\pm$ 0.17} & \underline{89.52 $\pm$ 0.29} & \underline{87.98 $\pm$ 0.17} & \underline{86.75 $\pm$ 0.32} & \underline{84.62 $\pm$ 0.15} & \underline{82.23 $\pm$ 0.26} & \underline{81.50 $\pm$ 0.10} & \underline{80.91 $\pm$ 0.16} & - & - \\
        ALICE \cite{peng2022few} & \emph{mini}ImageNet &
        92.81 $\pm$ 0.37 & 90.65 $\pm$ 0.42 & 88.57 $\pm$ 0.20 & 86.88 $\pm$ 0.31 &
        85.80 $\pm$ 0.12 & 83.43 $\pm$ 0.22 & 81.81 $\pm$ 0.30 & 80.77 $\pm$ 0.39 & 80.09 $\pm$ 0.25 & - & - \\
        \rowcolor{blue!14}
        NTK-CL (Ours) & \emph{mini}ImageNet &
        \textbf{97.75 $\pm$ 0.15} & \textbf{97.04 $\pm$ 0.32} & \textbf{95.30 $\pm$ 0.25} & \textbf{95.01 $\pm$ 0.26} &
        \textbf{94.92 $\pm$ 0.20} & \textbf{94.29 $\pm$ 0.18} & \textbf{93.09 $\pm$ 0.10} & \textbf{92.84 $\pm$ 0.21} & \textbf{92.81 $\pm$ 0.30} & - & - \\
        \midrule
        CEC \cite{zhang2021few} & CUB200 &
        \underline{84.67 $\pm$ 0.30} & \underline{82.68 $\pm$ 0.20} & \underline{80.52 $\pm$ 0.27} & \underline{76.55 $\pm$ 0.25} & \underline{76.47 $\pm$ 0.19} & \underline{74.74 $\pm$ 0.11} & \underline{74.69 $\pm$ 0.25} & \underline{74.05 $\pm$ 0.25} & \underline{72.58 $\pm$ 0.23} & \underline{72.26 $\pm$ 0.26} & \underline{71.56 $\pm$ 0.15} \\
        ALICE \cite{peng2022few} & CUB200 &
        77.47 $\pm$ 0.25 & 69.73 $\pm$ 0.26 & 68.70 $\pm$ 0.22 & 68.64 $\pm$ 0.21 &
        67.89 $\pm$ 0.12 & 66.25 $\pm$ 0.38 & 66.00 $\pm$ 0.17 & 64.77 $\pm$ 0.16 & 64.62 $\pm$ 0.17 & 64.22 $\pm$ 0.24 & 63.73 $\pm$ 0.32 \\
        \rowcolor{blue!14}
        NTK-CL (Ours) & CUB200 &
        \textbf{89.81 $\pm$ 0.26} & \textbf{88.15 $\pm$ 0.29} & \textbf{87.67 $\pm$ 0.24} & \textbf{86.24 $\pm$ 0.13} &
        \textbf{84.94 $\pm$ 0.32} & \textbf{84.78 $\pm$ 0.25} & \textbf{84.57 $\pm$ 0.34} & \textbf{84.49 $\pm$ 0.24} & \textbf{84.40 $\pm$ 0.19} & \textbf{84.38 $\pm$ 0.20} & \textbf{84.22 $\pm$ 0.22} \\
        \bottomrule
    \end{tabular}}
    \label{FSCIL_Results_21K}
\end{table*}

\begin{table*}[h!]
    \centering
    \caption{Evolution of incremental top-1 accuracy (\%) for different datasets under the Imbalanced CIL setting, utilizing the pre-trained weight derived from the ImageNet-21K. The suffix '-LFS' denotes uniform partitioning of all classes into \(N\) tasks for incremental training from scratch, while the suffix '-LFH' involves initial training on the first half of classes followed by incremental learning of the remaining classes divided into \(N\) tasks. Bold segments indicate optimal results, while underlined segments denote suboptimal results.}
    \setlength{\tabcolsep}{1.6mm}
    \renewcommand{\arraystretch}{0.95}
    \footnotesize
    \resizebox{0.98\linewidth}{!}{%
    \begin{tabular}{@{} cc *{11}{c} @{} }
        \toprule
        \multicolumn{2}{c}{\textbf{Imbalanced CIL Methods}} & \multicolumn{10}{c}{\textbf{Incremental Top-1 Accuracy (\%)}} \\
        \cmidrule(r){1-2} \cmidrule(l){3-13}
        \textbf{Name} & \textbf{Dataset} & \textbf{Task 1} & \textbf{Task 2} & \textbf{Task 3} & \textbf{Task 4} & \textbf{Task 5} & \textbf{Task 6} & \textbf{Task 7} & \textbf{Task 8} & \textbf{Task 9} & \textbf{Task 10} & \textbf{Task 11} \\
        \midrule
        LT-CIL-LFS \cite{liu2022long} & CIFAR100-LT & \underline{83.21 $\pm$ 0.32} & 72.42 $\pm$ 0.35 & 69.86 $\pm$ 0.18 & \underline{65.57 $\pm$ 0.27} & \underline{63.86 $\pm$ 0.33} & \underline{59.91 $\pm$ 0.09} & \underline{59.15 $\pm$ 0.18} & \underline{57.40 $\pm$ 0.30} & \underline{56.31 $\pm$ 0.13} & \underline{55.30 $\pm$ 0.40} & - \\
        GR-LFS \cite{he2024gradient} & CIFAR100-LT & \textbf{84.59 $\pm$ 0.18} & \underline{75.56 $\pm$ 0.20} & \underline{70.15 $\pm$ 0.21} & 65.29 $\pm$ 0.30 & 62.32 $\pm$ 0.31 & 58.82 $\pm$ 0.21 & 58.16 $\pm$ 0.28 & 56.15 $\pm$ 0.18 & 55.38 $\pm$ 0.22 & 54.62 $\pm$ 0.18 & - \\
        \rowcolor{blue!14}
        NTK-CL-LFS (Ours) & CIFAR100-LT & 82.67 $\pm$ 0.23 & \textbf{78.71 $\pm$ 0.14} & \textbf{77.24 $\pm$ 0.22} & \textbf{74.05 $\pm$ 0.40} & \textbf{72.63 $\pm$ 0.21} & \textbf{71.11 $\pm$ 0.33} & \textbf{71.07 $\pm$ 0.20} & \textbf{69.95 $\pm$ 0.20} & \textbf{69.90 $\pm$ 0.20} & \textbf{69.57 $\pm$ 0.23} & - \\
        \midrule
        LT-CIL-LFS \cite{liu2022long} & ImageNetSubset-LT & 94.29 $\pm$ 0.12 & 90.85 $\pm$ 0.27 & 89.75 $\pm$ 0.26 & 89.47 $\pm$ 0.28 & 89.25 $\pm$ 0.16 & 87.45 $\pm$ 0.16 & 87.01 $\pm$ 0.34 & 84.22 $\pm$ 0.17 & 83.49 $\pm$ 0.38 & 83.39 $\pm$ 0.21 & - \\
        GR-LFS \cite{he2024gradient} & ImageNetSubset-LT & \underline{96.04 $\pm$ 0.29} & \underline{93.83 $\pm$ 0.15} & \underline{92.52 $\pm$ 0.29} & \underline{92.24 $\pm$ 0.21} & \underline{92.16 $\pm$ 0.20} & \underline{90.38 $\pm$ 0.23} & \underline{90.35 $\pm$ 0.20} & \underline{86.67 $\pm$ 0.20} & \underline{86.45 $\pm$ 0.22} & \underline{86.24 $\pm$ 0.29} & - \\
        \rowcolor{blue!14}
        NTK-CL-LFS (Ours) & ImageNetSubset-LT & \textbf{96.30 $\pm$ 0.31} & \textbf{93.90 $\pm$ 0.18} & \textbf{92.94 $\pm$ 0.18} & \textbf{92.87 $\pm$ 0.23} & \textbf{92.60 $\pm$ 0.10} & \textbf{91.00 $\pm$ 0.15} & \textbf{90.51 $\pm$ 0.23} & \textbf{88.57 $\pm$ 0.18} & \textbf{88.46 $\pm$ 0.22} & \textbf{88.06 $\pm$ 0.38} & - \\
        \midrule
        LT-CIL-LFH \cite{liu2022long} & CIFAR100-LT & 62.74 $\pm$ 0.14 & 57.40 $\pm$ 0.14 & 52.08 $\pm$ 0.26 & 54.87 $\pm$ 0.20 & 56.58 $\pm$ 0.26 & 55.22 $\pm$ 0.21 & 55.37 $\pm$ 0.15 & 54.30 $\pm$ 0.30 & 54.21 $\pm$ 0.23 & 54.15 $\pm$ 0.10 & 54.16 $\pm$ 0.12 \\
        GR-LFH \cite{he2024gradient} & CIFAR100-LT & \underline{65.06 $\pm$ 0.14} & \underline{62.41 $\pm$ 0.22} & \underline{59.08 $\pm$ 0.28} & \underline{59.76 $\pm$ 0.27} & \underline{61.09 $\pm$ 0.11} & \underline{59.75 $\pm$ 0.18} & \underline{58.52 $\pm$ 0.12} & \underline{59.04 $\pm$ 0.31} & \underline{57.81 $\pm$ 0.14} & \underline{57.30 $\pm$ 0.25} & \underline{56.66 $\pm$ 0.20} \\
        \rowcolor{blue!14}
        NTK-CL-LFH (Ours) & CIFAR100-LT & \textbf{83.13 $\pm$ 0.12} & \textbf{77.97 $\pm$ 0.12} & \textbf{76.55 $\pm$ 0.10} & \textbf{76.35 $\pm$ 0.33} & \textbf{78.32 $\pm$ 0.18} & \textbf{77.42 $\pm$ 0.33} & \textbf{76.43 $\pm$ 0.22} & \textbf{76.89 $\pm$ 0.09} & \textbf{76.87 $\pm$ 0.14} & \textbf{76.78 $\pm$ 0.22} & \textbf{76.42 $\pm$ 0.10} \\
        \midrule
        LT-CIL-LFH \cite{liu2022long} & ImageNetSubset-LT & 90.68 $\pm$ 0.24 & 90.25 $\pm$ 0.33 & 87.97 $\pm$ 0.14 & 86.97 $\pm$ 0.16 & 87.87 $\pm$ 0.20 & 86.78 $\pm$ 0.18 & 82.56 $\pm$ 0.25 & 83.54 $\pm$ 0.20 & 84.24 $\pm$ 0.17 & 83.17 $\pm$ 0.28 & 82.48 $\pm$ 0.19 \\
        GR-LFH \cite{he2024gradient} & ImageNetSubset-LT & \underline{93.08 $\pm$ 0.31} & \underline{90.55 $\pm$ 0.30} & \underline{90.79 $\pm$ 0.19} & \underline{92.12 $\pm$ 0.12} & \underline{91.99 $\pm$ 0.16} & \underline{91.09 $\pm$ 0.10} & \underline{86.94 $\pm$ 0.29} & \underline{87.24 $\pm$ 0.16} & \underline{86.96 $\pm$ 0.14} & \underline{86.93 $\pm$ 0.19} & \underline{86.77 $\pm$ 0.22} \\
        \rowcolor{blue!14}
        NTK-CL-LFH (Ours) & ImageNetSubset-LT & \textbf{94.16 $\pm$ 0.24} & \textbf{93.37 $\pm$ 0.14} & \textbf{93.72 $\pm$ 0.20} & \textbf{94.37 $\pm$ 0.16} & \textbf{94.19 $\pm$ 0.19} & \textbf{94.30 $\pm$ 0.31} & \textbf{90.73 $\pm$ 0.19} & \textbf{90.58 $\pm$ 0.11} & \textbf{90.75 $\pm$ 0.24} & \textbf{90.04 $\pm$ 0.16} & \textbf{90.18 $\pm$ 0.14} \\
        \bottomrule
    \end{tabular}}
    \label{Imbalanced_CIL_Results_21K}
\end{table*}

\section{Few-shot and Imbalanced CIL} \label{other_settings}
To systematically investigate the model generalization and performance of our NTK-CL framework across diverse CIL settings, we have extended its application to encompass Few-Shot Class-Incremental Learning (FSCIL) and Imbalanced Class-Incremental Learning (Imbalanced CIL) scenarios.

In the context of FSCIL, our NTK-CL framework stands as a competitor to two prominent methodologies: CEC \cite{zhang2021few} and ALICE \cite{peng2022few}, both of which are prominently featured in the literature. Notably, FSCIL fundamentally differs from PEFT-CL, which frequently relies on pre-trained models. In contrast, FSCIL adheres to a strict protocol that avoids leveraging pre-trained models to maintain the integrity and purity of the few-shot learning process. The training phase is confined exclusively to an initial base session. Following this, the model remains unaltered through subsequent incremental sessions. This paradigm underscores the critical importance of the generalization capacity developed from the initial training on base session. The model, once trained during the base session, serves to extract features from data encountered in later incremental sessions, thereby enabling few-shot classification task while effectively addressing the challenge of catastrophic forgetting. First, to align the FSCIL methodologies with PEFT-CL setting, the initial model in CEC and ALICE is replaced with one pre-trained on the ImageNet-21K dataset, followed by the linear probe technique to fine-tune the feature layers. \textit{Empirical evidence demonstrates that a full fine-tuning approach results in a significant decline in model performance. For example, on the miniImageNet dataset, incremental top-1 accuracies drop from 53.3\% to 34.57\%. In contrast, the linear probe approach avoids this performance degradation and sustains a high level of accuracy.} Second, adhering to the established FSCIL paradigm, the Knowledge Retention, Task-Feature Dissimilarity, and Regularization Adjustment components are omitted from our NTK-CL framework. Our comparisons are conducted on the three most widely used datasets in FSCIL methods: CIFAR100, \emph{mini}ImageNet, and CUB200. The data splits strictly adhere to the divisions outlined in Table~\ref{supp:datasets}.

Despite these modifications, the experimental results, as detailed in Table~\ref{FSCIL_Results_21K}, highlight the superior effectiveness of the proposed method. Specifically, our method achieves an average improvement of 10\% to 20\% in incremental top-1 accuracies compared to current FSCIL methodologies, representing a substantial advancement in the field. This further demonstrates that expanding the sample (feature) size is an effective way to enhance the model generalization, even in few-shot scenarios. The findings suggest that future iterations of the FSCIL paradigm should reconsider their methodologies to incorporate the PEFT-FSCIL configuration.

For the Imbalanced CIL scenario, we have evaluated two prominent methodologies: LT-CIL \cite{liu2022long} and GR \cite{he2024gradient}. These evaluations are conducted within the refined settings of Learning From Scratch (LFS) and Learning From Half (LFH). The LFS setting is characterized by the equitable distribution of all classes into \(N\) sequential tasks, each of which is introduced incrementally. Conversely, the LFH setting initiates with the comprehensive training on the initial half of the class set, succeeded by the incremental acquisition of the residual classes, equally apportioned across \(N\) subsequent tasks. This systematic approach facilitates a nuanced comparison, elucidating the relative efficacy and adaptability of the selected methodologies under varying conditions of class imbalance and incremental learning challenges. Following the setup in \cite{liu2022long}, we have developed long-tailed variants of the CIFAR-100 and ImageNet Subset datasets, denoted as CIFAR100-LT and ImageNetSubset-LT, respectively. These adaptations are constructed from their originally balanced counterparts through the systematic removal of training instances to introduce a controlled level of class imbalance. Specifically, this process is guided by an imbalance factor \(\rho = \frac{n_{\text{max}}}{n_{\text{min}}} = 100\), wherein \(n_{\text{max}}\) represents the highest number of training samples associated with any single class, and \(n_{\text{min}}\) signifies the lowest such count across all classes. In these methods, we modify the initialization model to use a ViT model pre-trained on the ImageNet-21K dataset and adopt the linear probe approach for fine-tuning. Unlike in the FSCIL scenario, our NTK-CL framework utilizes all its components, thereby leveraging its full potential.

The empirical results presented in Table~\ref{Imbalanced_CIL_Results_21K} unequivocally demonstrate that our NTK-CL framework markedly surpasses peer methodologies when initialized with identical pre-trained weight. Specifically, on the ImageNetSubset-LT dataset—a close approximation to the pre-training ImageNet-21K dataset—the observed performance enhancement is substantial relative to the initial benchmarks reported in extant literature. Notably, this superior performance is maintained even under conditions of long-tailed distribution, underscoring the robustness of our proposed framework. For the CIFAR100-LT dataset, which serves as a more stringent test of our framework's capabilities, the initial performance in Task 1 under the LFS setting is observed to be slightly inferior. However, in the context of subsequent incremental tasks, our NTK-CL framework exhibits a pronounced superiority over contemporary methodologies. This outcome highlights the pivotal role of our task-level orthogonality constraints and the knowledge retention mechanism. Under the LFH setting, the introduction of a pre-trained model and the long-tailed distribution of the training data can lead to unusual fluctuations in incremental top-1 accuracy. This is expected, as performance may be poorer on tasks with more extreme long-tailed distributions but improve on subsequent tasks, resulting in a trend of initial decline followed by recovery. Despite these fluctuations, the performance of our framework on the CIFAR100-LT dataset is particularly noteworthy, achieving a near 20\% improvement in incremental top-1 accuracy across all incremental tasks. This significant improvement further corroborates the effectiveness of our NTK-CL framework, which innovatively reinterprets and decomposes PEFT-CL through the theoretical frameworks of generalization and NTK theory. 

In conclusion, these findings not only validate the theoretical underpinnings of our framework but also attest to its robustness and efficiency across a spectrum of CIL scenarios.

\section{Discussion for LLMs and Omni-Models}
\label{appendix:llms}
While the present study primarily concentrates on mainstream research trajectories within the domain of CL, with a particular emphasis on visual tasks, it is imperative to recognize the accelerating advancements in natural language processing (NLP). From an industrial and practical standpoint, these developments warrant heightened scholarly attention. The advent of pre-trained large language models (LLMs), trained on extensive and diverse corpora, has conferred a distinctive advantage upon NLP relative to computer vision (CV). In parallel, CL for NLP has emerged as an increasingly prominent field, yielding a series of notable contributions, including but not limited to \cite{chen2023parameterizing, zhao2024sapt, qin2022lfpt, razdaibiedina2023progressive, wang2023orthogonal, yang2024parameter}. These works merit rigorous examination, as they exhibit methodological innovations and conceptual frameworks that bear significant resemblance to state-of-the-art advancements in vision-centric CL research.

For instance, several representative studies, namely \cite{chen2023parameterizing, qin2022lfpt, razdaibiedina2023progressive}, adopt paradigms and architectural strategies closely aligned with approaches introduced in the visual domain \cite{wang2022learning, wang2022dualprompt, smith2023coda, kurniawan2024evolving, huang2024ovor, qiao2023prompt, gao2024consistent}. Similarly, the concept of sample replay as a means to optimize sequential task learning, as articulated in \cite{zhao2024sapt}, demonstrates notable conceptual congruence with the hierarchical replay mechanisms advanced in \cite{wang2024hierarchical}. In addition, the structural insights and parameter isolation techniques explored in \cite{wang2023orthogonal, yang2024parameter} reveal methodological parallels with frameworks such as \cite{zhou2024expandable, liang2024inflora}. These convergences underscore a fundamental insight: the underlying principles governing the design of CL algorithms exhibit a remarkable degree of consistency across different modalities and model architectures. This observation not only reinforces the universality of core CL paradigms but also provides a coherent basis for cross-domain methodological transfer and future research directions.

Among these NLP-oriented CL frameworks, the work presented by \cite{wang2023orthogonal} exhibits particularly strong conceptual alignment with our proposed NTK-CL framework introduced in this study. Both methodologies emphasize the pivotal role of orthogonalization constraints and regularization mechanisms in mitigating catastrophic forgetting. \cite{wang2023orthogonal} reports compelling empirical gains across 15 sequential text classification benchmarks, thereby attesting to the efficacy of their approach in sustaining knowledge retention over extended task sequences. Therefore, extending our work to the field of NLP is entirely feasible. However, the inherent complexity associated with re-implementing our techniques in the Transformer and HuggingFace python libraries renders a comprehensive empirical comparison beyond the scope of the current work. We defer an in-depth investigation of these methods to future research, with the objective of maintaining the clarity and focus of the present study.

In addition, the observed methodological convergence between CV-CL and NLP-CL paradigms invites a broader inquiry into the feasibility and effectiveness of CL within emerging multi-modal foundation models, often referred to as MLLMs or Omni-Models, which are pre-trained across multiple modalities (e.g., vision, language, audio) and are designed for versatile task generalization. A critical open question concerns whether the inherent modality diversity in such models improves resilience against catastrophic forgetting or introduces new forms of interference during PEFT-CL. Addressing this question constitutes a critical avenue for advancing CL techniques.


In conclusion, our NTK-CL framework presents a promising direction for ensuring the long-term adaptability and sustainability of both LLMs and Omni-Models. Future research should prioritize the development of more efficient sample size extension module, past knowledge retention module, inter-task orthogonalization constraints, innovative regularization constraints, and rigorous theoretical analyses to deepen our understanding of forgetting mechanisms. Such advancements will be instrumental in achieving robust and scalable CL across diverse domains and modalities.
\end{document}